\documentclass[11pt]{article}
\usepackage{amssymb, amsmath, amsthm}

\usepackage[margin=1in]{geometry}
\usepackage{amsfonts,mathtools,mathrsfs,mathtools,color}
\usepackage[colorlinks]{hyperref}
\usepackage{graphicx}
\usepackage{xcolor}
\usepackage{tocloft}
\usepackage{bigints}
\usepackage{makecell}
\usepackage{algorithm}
\usepackage{algorithmic}
\usepackage{caption}
\usepackage{tikz}
\usepackage{pgfplots}
\pgfplotsset{compat=1.18}
\usepgfplotslibrary{groupplots}
\usepackage{centernot}
\usepackage{bm}
\usepackage{nicefrac}
\usepackage{dsfont}
\usepackage{enumitem}
\usepackage{booktabs}
\usepackage{xspace}

\usepackage{caption}
\usetikzlibrary{math}
\usetikzlibrary{decorations.markings, calc, arrows} 
\usetikzlibrary{arrows.meta}

\definecolor{coolblack}{rgb}{0.0, 0.0, 0.230}
\hypersetup{
  colorlinks,
  citecolor=coolblack,
  linkcolor=coolblack,
  urlcolor=coolblack}

\newtheorem{definition}{Definition}
\newtheorem{remark}{Remark}
\newtheorem{proposition}{Proposition}
\newtheorem{theorem}{Theorem}
\newtheorem*{theorem*}{Theorem}
\newtheorem{lemma}{Lemma}
\newtheorem{corollary}{Corollary}

\newtheorem{assumption}{Assumption}

\usepackage{romanbar}

\newcommand{\rn}[1]{\Romanbar{#1}}
\newcommand{\lb}[1]{{\lfloor #1 \rfloor_h}}
\newcommand{\ub}[1]{{\lceil #1 \rceil_h}}

\newcommand{\nll}{\centernot{\ll}}
\newcommand{\op}{\mathsf{op}}
\newcommand{\dvert}{\mathbin{\|}}

\newcommand{\bP}{\bm{\mathrm{P}}}
\newcommand{\bQ}{\bm{\mathrm{Q}}}

\renewcommand{\P}{\mathcal{P}}
\newcommand{\R}{\mathbb{R}} 
\newcommand{\N}{\mathcal{N}}

\newcommand{\E}{\mathop{\mathbb{E}}}

\newcommand{\mfm}{\mathfrak{m}}

\renewcommand{\part}[2]{\frac{\partial #1}{\partial #2}}

\newcommand{\bbN}{\mathbb{N}}

\newcommand{\Tr}{\mathsf{Tr}}

\newcommand{\KL}{\mathsf{KL}}

\newcommand{\TV}{\mathsf{TV}}

\newcommand{\law}{\mathcal{L}}

\newcommand{\Lip}{\mathsf{Lip}}

\newcommand{\sfD}{\mathsf{D}}

\newcommand{\sfR}{\mathsf{R}}
\newcommand{\sfP}{\mathsf{P}}
\newcommand{\sfc}{\mathsf{c}}
\newcommand{\sfj}{\mathsf{j}}
\newcommand{\sfp}{\mathsf{p}}
\newcommand{\sR}{\mathsf{R}}

\newcommand{\MI}{\mathsf{MI}}

\newcommand{\cC}{\mathcal{C}}
\newcommand{\HF}{\mathsf{HF}}
\newcommand{\uHMC}{\textup{uHMC}\xspace}
\newcommand{\uhmc}{\textup{uHMC}\xspace}
\newcommand{\uHMCv}{\textup{uHMC-v}\xspace}
\newcommand{\uhmcv}{\textup{uHMC-v}\xspace}

\newcommand{\uhmcs}{\textup{uHMC-s}\xspace}
\newcommand{\eHMC}{\textup{eHMC}\xspace}
\newcommand{\ehmc}{\textup{eHMC}\xspace}

\newcommand{\gammad}{\boldsymbol{\gamma}_d}

\renewcommand{\d}{\,d}
\newcommand{\dt}{\,dt}

\newcommand{\dx}{\,dx}

\newcommand{\ds}{\,ds}
\newcommand{\dz}{\,dz}

\title{Tail-Sensitive KL and R\'enyi Convergence of\\ Unadjusted Hamiltonian Monte Carlo via One-Shot Couplings}

\author{Nawaf Bou-Rabee\thanks{
   Department of Mathematical Sciences, Rutgers University, and Center for Computational Mathematics, Flatiron Institute.
    \texttt{nawaf.bourabee@rutgers.edu}
}
\and Siddharth Mitra\thanks{Department of Computer Science, Yale University. \texttt{siddharth.mitra@yale.edu}}
\and Andre Wibisono\thanks{Department of Computer Science, Yale University. \texttt{andre.wibisono@yale.edu}}
}

\date{January 2026}

\begin{document}

\maketitle

\begin{abstract}
Hamiltonian Monte Carlo (HMC) algorithms are among the most widely used sampling methods in high dimensional settings, yet their convergence properties are poorly understood in divergences that quantify relative density mismatch, such as Kullback–Leibler (KL) and~R\'enyi divergences. These divergences naturally govern acceptance probabilities and warm-start requirements for Metropolis-adjusted Markov chains. In this work, we develop a  framework for upgrading Wasserstein convergence guarantees for unadjusted Hamiltonian Monte Carlo (\uhmc) to guarantees in tail-sensitive KL and R\'enyi divergences. Our approach is based on \emph{one-shot~couplings}, which we use to establish a regularization property of the \uhmc transition kernel. This regularization allows Wasserstein--$2$ mixing-time and asymptotic bias bounds to be lifted to KL divergence, and analogous Orlicz–Wasserstein bounds to be lifted to R\'enyi divergence, paralleling earlier work of Bou-Rabee and Eberle (2023) that upgrade Wasserstein--$1$ bounds to total variation distance via kernel smoothing.  As a consequence, our results provide quantitative control of relative density mismatch, clarify the role of discretization bias in strong divergences, and yield principled guarantees relevant both for unadjusted sampling and for generating warm starts for Metropolis-adjusted Markov chains.
\end{abstract}

\tableofcontents

\section{Introduction}\label{sec:Introduction}

Let $\nu$ be an absolutely continuous probability measure on $\R^d$ with density proportional to $e^{-f}$. Sampling from $\nu$ is a fundamental algorithmic task arising in many applications, including computational statistics and machine learning.
When exact sampling from $\nu$ is infeasible, a standard approach is Markov chain Monte Carlo (MCMC), which constructs a suitable Markov chain $(X_k)_{k \ge 0}$ whose iterates are used to generate approximate samples from $\nu$ \cite{RobertCassella,asmussen2007,Liu,DoucMoulinesPriouretSoulier}.

A common paradigm in the design of modern MCMC algorithms is to begin with a continuous-time stochastic process that is ergodic with respect to $\nu$ and then discretize its dynamics; see, for example,
\cite{roberts1996exponential,mattingly2002ergodicity,durmus2024asymptotic}.
Such discretizations typically introduce a \emph{bias}:  when the resulting discrete-time Markov chain has an invariant distribution $\tilde \nu$, it generally differs from $\nu$.

A classical approach to eliminating this discretization bias is to incorporate a Metropolis–Hastings correction \cite{metropolis1953equation, hastings1970monte}, yielding an \emph{adjusted} Markov chain with invariant distribution $\nu$. In this work, however, we focus on \emph{unadjusted} schemes. 
The motivation for this choice is multifaceted.

First, in certain applications, the target distribution $\nu$ is itself defined only up to an approximation, so that Metropolis correction targets the invariant distribution of an approximate model rather than the idealized target.
This occurs, for example, in interacting particle approximations of nonlinear measures \cite{Sz91,delmoral2004feynman}; in finite-dimensional approximations of measures on infinite-dimensional spaces (as in Bayesian inverse problems and conditioned diffusions) \cite{stuart2010inverse,hairer2005analysis}; and in statistical models in which the likelihood is defined implicitly through a numerically discretized ordinary or partial differential equation \cite{matsuda2021estimation}.
In such settings, Metropolis adjustment does not remove --- and may not meaningfully address --- the dominant sources of bias, and unadjusted MCMC methods are therefore natural objects of study.

Second, even when exact Metropolis correction is available in principle, both theoretical analyses and empirical experience show that the performance of adjusted Markov chains can depend sensitively on the choice of initial distribution.
If the chain is initialized in the tail of the target distribution, proposals are typically local and acceptance probabilities can be extremely small, particularly in high dimensions.
Consequently, the chain may require a large number of iterations to reach the typical set of $\nu$, resulting in slow initial convergence and degraded practical performance.

To obtain meaningful convergence guarantees for adjusted Markov chains, and to avoid degeneracy of acceptance probabilities in practice, it is therefore crucial to control how much probability mass the chain assigns to regions that are rare under the target distribution $\nu$.
This is commonly achieved through a \emph{warm-start} condition, which requires the initial distribution to be close to $\nu$ in a tail-sensitive divergence such as Kullback--Leibler (KL) or R\'enyi divergence.
Such divergences control relative density mismatch and prevent the initialization from placing excessive mass in regions where $\nu$ is extremely small.  Such warm-start assumptions are standard in the theoretical analysis of adjusted MCMC methods
\cite{dwivedi2019log,chen2023does,altschuler2024faster,srinivasan2024fast,srinivasan25high,kook2025faster,kook2025sampling}.

By contrast, warm-start conditions formulated only in weaker metrics such as total variation distance do not control relative density in the tails. Total variation distance controls discrepancies additively (through absolute differences in probability mass) whereas KL and R\'enyi divergences control discrepancies multiplicatively, through ratios of densities.
 As a result, a distribution may be close to $\nu$ in total variation while still assigning nonnegligible mass to regions that are exponentially unlikely under $\nu$, a mismatch that can lead to vanishing acceptance probabilities and bottlenecks in adjusted chains (see Figure~\ref{fig:DivComparison}).

Obtaining a warm start directly for an adjusted Markov chain is often difficult, as it typically requires prior knowledge of the typical set of $\nu$.
Unadjusted Markov chains therefore play a natural auxiliary role: when they admit quantitative convergence guarantees in tail-sensitive divergences, they can be used to efficiently and rigorously generate warm initializations for adjusted chains.

From this perspective, unadjusted Markov chains play a dual role.
They provide natural sampling schemes in settings where the target is itself approximate, and they serve as effective mechanisms for generating warm starts for subsequent adjusted algorithms.
This two-phase strategy --- first running an unadjusted chain from a cold start to enter the typical set, and then switching to an adjusted chain for exact sampling --- has proved effective in both theory and practice
\cite{durmus2024asymptotic,altschuler2024faster,apers2024hamiltonian,sherlock2021efficiency,golightly2015delayed,stuart2019two}.

These considerations lead naturally to two quantitative questions for unadjusted Markov chains, which we study in tail-sensitive KL and R\'enyi divergences.
First, how rapidly does the law of $X_k$ converge to its invariant distribution $\tilde \nu$?
Second, how close is $\tilde \nu$ to the intended target $\nu$?
The former concerns mixing-time guarantees, while the latter concerns asymptotic bias.
Together, they determine both the computational cost required to reach stationarity and the accuracy with which the chain approximates the target distribution.

In this work, we assume that $f$ is differentiable and that its gradient can be evaluated.
In this context, we establish mixing-time  and asymptotic bias bounds for unadjusted Hamiltonian Monte Carlo (\uhmc) in KL and R\'enyi divergences.  Guarantees formulated in these tail-sensitive divergences are particularly well suited for warm-starting Metropolis-adjusted Markov chains, as they provide quantitative control of relative density mismatch prior to the introduction of a Metropolis correction.

\begin{figure}[t]
    \centering
    \pgfplotsset{
        every axis/.append style={
            width=\linewidth,       
            height=5.5cm,
            axis lines=left,        
            xmin=-5, xmax=15,
            xtick={0, 10},          
            xlabel={},            
            label style={font=\small},
            tick label style={font=\small}
        }
    }
    \hspace{-1.5em}
    \begin{minipage}{0.48\textwidth} 
        \centering                   
        \begin{tikzpicture}
            \begin{axis}[
                ymin=0, ymax=0.45,
                ylabel={density},
                title={}
            ]
                \addplot[domain=-5:14.7, samples=200, color=black, very thick] 
                    {exp(-0.5*x^2)/sqrt(2*pi)};
                
                \addplot[domain=-5:14.7, samples=200, color=red, dashed, very thick] 
                    {0.99*exp(-0.5*x^2)/sqrt(2*pi) + 0.01*exp(-0.5*(x-10)^2)/sqrt(2*pi)};
            \end{axis}
        \end{tikzpicture}
    \end{minipage}
    \hspace{0.5em}
    \begin{minipage}{0.48\textwidth} 
        \centering
        \begin{tikzpicture}
            \begin{axis}[
                ymode=log,       
                ymin=1e-30, ymax=1,
                ylabel={density (log-scale)},
                ytick={1e-30, 1e-15, 0.1},
                yminorticks=false,
                title={}
            ]
                \addplot[domain=-5:14, samples=500, color=black, very thick] 
                    {exp(-0.5*x^2)/sqrt(2*pi)};
                
                \addplot[domain=-5:14, samples=500, color=red, dashed, very thick] 
                    {0.99*exp(-0.5*x^2)/sqrt(2*pi) + 0.01*exp(-0.5*(x-10)^2)/sqrt(2*pi)};
            \end{axis}
        \end{tikzpicture}
    \end{minipage}
    \hfill

    \vspace{0.5em}
    
    \begin{minipage}[t]{0.48\textwidth}
        \centering
        {\small (a) linear--$y$ plot}
    \end{minipage}
    \hfill
    \begin{minipage}[t]{0.44\textwidth}
        \centering
        {\small (b) semilog--$y$ plot}
    \end{minipage}

    \caption{Distributions $\pi = \mathcal{N}(0,1)$ plotted as [\protect\tikz[baseline=-0.6ex]{\protect\draw[black, very thick] (0,0)--(0.5,0);}] and $\mu = 0.99\mathcal{N}(0,1) + 0.01\mathcal{N}(10,1)$ plotted as [\protect\tikz[baseline=-0.6ex]{\protect\draw[red, dashed, very thick] (0,0)--(0.5,0);}] depicted on a (a) linear--$y$ axis, and (b) semilog--$y$ axis.  
    For these distributions, $\TV(\mu, \pi) \leq 0.01$, $\KL(\mu \dvert \pi) \geq 0.4$, and $\sfR_2(\mu \dvert \pi) \geq 90$. The linear--$y$ scale shows the large overlap (small TV distance) and the semilog--$y$ scale highlights the discrepancy in the tails (large KL divergence and Rényi--$2$ divergence).}
    \label{fig:DivComparison}
\end{figure}
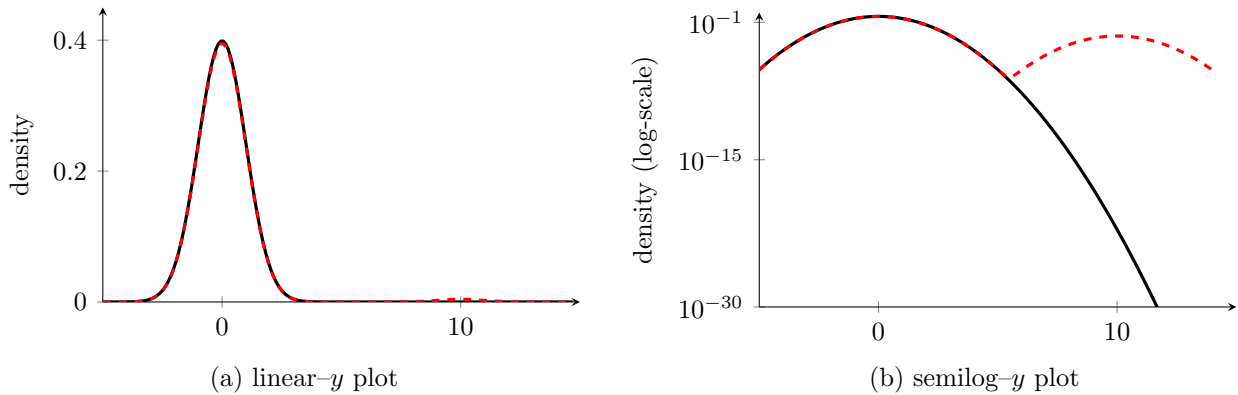

\subsection{Main Results}\label{subsec:MainResults}

Hamiltonian Monte Carlo algorithms generate proposals by approximately simulating Hamiltonian dynamics associated with the target density $\nu \propto e^{-f}$. In continuous time, these dynamics preserve $\nu$ exactly and give rise to \emph{exact} HMC (\ehmc). In practice, however, the Hamiltonian flow must be discretized, leading to \emph{unadjusted} HMC schemes whose invariant distribution generally differs from $\nu$. Throughout this paper, we study unadjusted HMC implemented using the \emph{velocity Verlet} (or leapfrog) integrator~\cite{bou2018geometric,hairer2006geometric}. This integrator is the standard discretization used in HMC due to its simplicity, time-reversibility, and symplectic structure. We denote the resulting Markov chain by \uhmcv. Each transition of \uhmcv\ consists of (i) a Gaussian velocity refreshment followed by (ii) several deterministic Verlet steps approximating the Hamiltonian flow.

Our main results establish quantitative convergence guarantees for \uhmcv\ in Kullback--Leibler (KL) divergence and R\'enyi divergence. These guarantees take two complementary forms: mixing-time bounds, which control convergence to the invariant distribution of \uhmcv, and asymptotic bias bounds, which quantify the discrepancy between this invariant distribution and the target distribution $\nu$. The results are presented in Sections~\ref{subsec:Main-KL} and~\ref{subsec:Main-Renyi}.

\medskip
\noindent\textbf{Regularization via one-shot couplings.}
The key technical ingredient underlying our analysis is a \emph{regularization} (or smoothing) property of the \uhmcv\ transition kernel in tail-sensitive divergences. This property is established using \emph{one-shot couplings}, which allow us to relate Wasserstein-type distances between initial distributions to stronger divergences after a single Markov transition.

To formalize this notion, let $\bP$ be a Markov kernel on $\R^d$, let $W$ denote a Wasserstein distance associated with a given cost function, and let $\sfD$ be a probability divergence or distance. We say that $\bP$ is \emph{regularizing} in $\sfD$ if there exists a constant $C>0$ such that
\begin{equation}\label{eq:Regularization}
    \sfD(\mu \bP \dvert \pi \bP) \;\le\; C\, W(\mu,\pi),
\end{equation}
for all probability measures $\mu,\pi \in \P(\R^d)$.

Inequality~\eqref{eq:Regularization} captures the idea that the stochastic update $\bP$ smooths singular or poorly aligned initial distributions. For example, if~\eqref{eq:Regularization} holds in KL or R\'enyi divergence, then applying $\bP$ to two distinct Dirac measures $\delta_x$ and $\delta_y$ yields a finite divergence, whereas the divergence is infinite before the update. Additional consequences follow from the variational representations of $\sfD$. For instance, when~\eqref{eq:Regularization} holds in total variation distance, the kernel $\bP$ maps bounded (possibly discontinuous) test functions to Lipschitz functions; see Section~\ref{app:Regularization}.

\medskip
\noindent\textbf{Why regularization is nontrivial for \uhmcv?}
The regularization effect in \uhmcv\ arises from the Gaussian velocity refreshment at the beginning of each transition. However, establishing~\eqref{eq:Regularization} is technically delicate because the injected randomness is subsequently propagated through multiple deterministic Verlet steps. In contrast to schemes where Gaussian noise is injected explicitly at each discretization step --- such as splitting schemes for kinetic Langevin dynamics~\cite[Propositions~3 and~22]{monmarche2021high} or exponential Euler integrators~\cite[Lemma~4.2]{altschuler2025shifted} --- the randomness in \uhmcv\ must be analyzed after being transported through a composition of nonlinear maps.

\medskip
\noindent\textbf{Cross-regularization and asymptotic bias.}
To obtain asymptotic bias bounds, it is not sufficient to compare two trajectories of \uhmcv. Instead, we must compare a single step of \uhmcv\ to a single step of \ehmc. This leads to a \emph{cross-regularization} problem, in which we bound divergences of the form
\[
    \sfD(\mu \bP \dvert \pi \bQ),
\]
where $\bP$ denotes the \uhmcv\ kernel and $\bQ$ denotes the \ehmc\ kernel.

This comparison is unavoidable because KL and R\'enyi divergences neither satisfy a triangle inequality nor are symmetric.  The required bounds are established in Lemmas~\ref{lem:KL_Target_OneStepGeneral} and~\ref{lem:Renyi_Target_OneStepGeneral} using a novel one-shot coupling construction (Figure~\ref{fig:coupling}(b)). To our knowledge, this type of cross-regularization analysis has not previously appeared in the study of \uhmcv.

By contrast, when working in total variation distance --- which is symmetric and satisfies a triangle inequality --- the cross comparison task reduces to a simpler same-initial-point comparison; see~\cite[Figure~1(b) and Theorem~7]{bou2023mixing}. Related regularization bounds comparing discretized and exact dynamics have been obtained for kinetic Langevin dynamics in KL divergence~\cite[Lemma~4.2]{altschuler2025shifted} and for the unadjusted Langevin algorithm in KL and R\'enyi divergence~\cite[Lemmas~4.7 and~C.2]{altschuler2024shifted}, typically via Girsanov’s theorem. Our approach instead relies on coupling arguments and accommodates the more intricate structure of the discrete flow of the Verlet integrator.

\medskip
\noindent\textbf{End-to-end guarantees and modularity.}
By combining the regularization and cross-regularization bounds with existing Wasserstein--$2$ and Orlicz--Wasserstein contractivity estimates for \uhmcv, we obtain end-to-end mixing-time and asymptotic bias bounds in KL and R\'enyi divergence (Section~\ref{subsubsec:Examples-Verlet}).

An important feature of our framework is its modularity. The Wasserstein contraction need not be driven by \uhmcv\ itself: any discrete-time Markov kernel satisfying suitable Wasserstein contractivity assumptions may be used, provided that the final step of the chain is an \uhmcv\ update. We illustrate this flexibility by applying our results to the stratified integrator for \uhmc\ introduced in~\cite{bou2025unadjusted} (Section~\ref{subsubsection:Examples-stratified}).

The resulting gradient complexity depends on the choice of kernel driving the Wasserstein contraction; see Corollaries~\ref{cor:complexity-KL-verlet} and~\ref{cor:complexity-KL-stratified}. Our framework also allows different kernels to be used for Wasserstein contraction and for regularization. In Section~\ref{sec:Discussion}, we outline how unadjusted Langevin dynamics may be used as a regularizing kernel and discuss the associated trade-offs.

\medskip
\noindent\textbf{Summary of contributions.}
Our main contributions are as follows:
\begin{itemize}[noitemsep]
    \item We establish mixing-time guarantees for \uhmcv\ in KL divergence (Theorem~\ref{thm:KLMixing}), verifying the required Wasserstein--$2$ assumptions for strongly log-concave targets.
    \item We derive asymptotic bias bounds for \uhmc\ in KL divergence (Theorem~\ref{thm:KLTarget}), applicable whenever the final step of the chain is \uhmcv, and instantiate these results for \uhmcv\ and for the stratified integrator \uhmcs~\cite{bou2025unadjusted}. 
    \item We establish mixing-time guarantees for \uhmcv\ in R\'enyi divergence (Theorem~\ref{thm:RenyiMixing}) based on Orlicz--Wasserstein contractivity.
    \item We derive asymptotic bias bounds for \uhmc\ in R\'enyi divergence (Theorem~\ref{thm:RenyiTarget}), again assuming a final \uhmcv\ update.
    \item We apply the KL mixing-time bounds to study information contraction along \uhmcv\ for strongly log-concave targets (Section~\ref{subsec:InformationContraction}).
\end{itemize}

\subsection{Consequences of Tail-Sensitive Convergence Guarantees}

Kullback--Leibler (KL) divergence (Definition~\ref{def:KLDivergence}) and R\'enyi divergence (Definition~\ref{def:RenyiDivergence}) are natural metrics in which to formulate convergence guarantees, as they control a range of other commonly used statistical distances. In particular, they imply bounds on total variation distance via Pinsker’s inequality, on Wasserstein distance under suitable transportation inequalities, and on $\chi^2$ divergence. These relationships are summarized by the chain of inequalities
\[
2\TV^2(\mu,\pi) \le \KL(\mu \dvert \pi) \le \log \mkern-0.5mu \bigl(1+\chi^2(\mu \dvert \pi)\bigr) \le \sfR_q(\mu \dvert \pi),
\]
valid for all probability measures $\mu,\pi$ and all $2 \le q < \infty$. Beyond these implications, convergence guarantees in KL and R\'enyi divergence have several further consequences:
\begin{itemize}[noitemsep]
\item \emph{Warm-start generation for adjusted samplers.} Control in tail-sensitive divergences provides quantitative bounds on relative density mismatch and ensures that the output of an unadjusted chain yields an initial distribution suitable for Metropolis-adjusted sampling; see, for example, 
\cite{dwivedi2019log,chen2023does,altschuler2024faster,srinivasan2024fast,srinivasan25high,kook2025faster,kook2025sampling}.
\item \emph{Dependence along Markov chains.} KL convergence controls information contraction along the chain and yields quantitative bounds on temporal dependence, thereby characterizing the rate at which approximately independent samples are obtained; see~\cite{liang25PhiMI} and Section~\ref{subsec:InformationContraction}.

\item \emph{Differential privacy.} R\'enyi divergence plays a central role in the analysis of differentially private iterative algorithms, where privacy guarantees are obtained via compositional and regularization properties of Markov kernels \cite{mironov2017renyi,feldman2018privacy,ganesh2020faster,chourasia2021differential,altschuler2024privacy,bravo2025mixing}.
\end{itemize}

\subsection{Theoretical Context and Related Work on Hamiltonian Monte Carlo}

Establishing rigorous convergence guarantees for algorithms based on Hamiltonian dynamics --- such as exact Hamiltonian Monte Carlo (\ehmc)~\cite{duane1987hybrid,neal2011mcmc}, unadjusted Hamiltonian Monte Carlo (\uhmc), Metropolis-adjusted HMC, and the No-U-Turn Sampler (NUTS)~\cite{hoffman2014no} --- has been an active area of research. Nevertheless, the theoretical understanding of Hamiltonian-based sampling methods remains significantly less developed than that of Langevin dynamics and their discrete-time counterparts, particularly when convergence is measured in strong, tail-sensitive divergences.

For discretizations of overdamped Langevin dynamics, including the unadjusted Langevin algorithm (uLA)~\cite{roberts1996exponential,roberts1998optimal} and the proximal sampler~\cite{lee2021structured}, convergence guarantees are available under relatively mild assumptions in Kullback--Leibler divergence, R\'enyi divergence, and even broader families of divergences such as $\Phi$-divergences~\cite{chafai2004entropies}; see, for example,~\cite{VW23,chewi2024analysis,mitra25fast,chen2022improved} and the survey~\cite{chewi2023log}. These results provide detailed control of relative density mismatch and have played a central role in understanding warm-start requirements and transient behavior for Langevin-based samplers.

In contrast, existing guarantees for Hamiltonian-based methods are more limited. Convergence results for \ehmc\ have been established primarily in Wasserstein distance~\cite{chen2019optimal,mangoubi2021mixing}. For \uhmc, convergence guarantees are available in Wasserstein distance and total variation distance~\cite{bou2020coupling,bou2023mixing,bou2023convergence}. For Metropolis-adjusted HMC, convergence has been analyzed in total variation distance~\cite{chen2023does} and in $\chi^2$ divergence~\cite{chen2020fast}. Results for NUTS are currently restricted to total variation distance for Gaussian targets~\cite{bou2024mixing}.

Asymptotic bias bounds for \uhmc\ in KL divergence were obtained in~\cite{camrud2023second} by analyzing the decay of a modified entropy functional. In the setting of strongly log-concave targets, that work yields an asymptotic KL bias of order $O(h^4 d^4)$, where $h$ denotes the Verlet step size and $d$ the dimension. By contrast, our results yield a bias of order $O(h^4 d^3)$ under comparable assumptions (see Section~\ref{subsubsec:Examples-Verlet} and Table~\ref{table:complexity}). To the best of our knowledge, no prior work establishes convergence guarantees for \uhmc --- or even \ehmc ---  in R\'enyi divergence.

This gap in the theory is particularly striking given the widespread practical use of Hamiltonian-based samplers, especially in high-dimensional Bayesian inference. HMC and its variants form the computational backbone of modern probabilistic programming systems such as Stan~\cite{carpenter2017stan} and PyMC~\cite{patil2010pymc}, where algorithmic performance is known to be highly sensitive to initialization and tail behavior.

\emph{The present work addresses this gap by developing a systematic framework for tail-sensitive convergence analysis of unadjusted Hamiltonian Monte Carlo.} By upgrading existing Wasserstein guarantees to KL and R\'enyi divergence via one-shot couplings and kernel regularization, our results provide quantitative control of relative density mismatch, clarify the role of discretization bias in strong divergences, and supply rigorous guarantees relevant both for unadjusted sampling and for warm-starting Metropolis-adjusted HMC.

\section{Preliminaries}

Let $\P(\R^d)$ denote the set of probability distributions on $\R^d$.
When a probability distribution $\rho \in \P(\R^d)$ is absolutely continuous with respect to Lebesgue measure on $\R^d$, we identify $\rho$ with its density (with respect to Lebesgue measure), which we also denote by $\rho : \R^d \to [0,\infty)$.
Throughout the paper, we assume that all densities satisfy the regularity conditions required for the statements and results to be well defined.

We let $\delta_x \in \P(\R^d)$ denote the Dirac distribution at a point $x \in \R^d$, and we denote by $\Gamma(\mu,\pi)$ the set of couplings between two probability distributions $\mu,\pi \in \P(\R^d)$.
The standard Gaussian distribution $\mathcal{N}(0,I)$ on $\R^d$ is denoted by $\gammad$.
The law of a random variable $X$ is denoted by $\law(X)$.
For any $\mu \in \P(\R^d)$ and integer $k \ge 1$, we define the $k$th moment
\[
m_k(\mu) \coloneqq \E_{X \sim \mu}\bigl[\|X\|^k\bigr].
\]

\subsection{Probability distances and divergences}

We begin by recalling the definition of Wasserstein--$p$ distance between two probability distributions. Our results in KL divergence (Theorems~\ref{thm:KLMixing} and~\ref{thm:KLTarget}) upgrade Wasserstein--$2$ guarantees to KL divergence.

\begin{definition}[Wasserstein--$p$ distance]\label{def:WassersteinDistance}
    Let $\mu, \pi \in \P(\R^d)$ and let $1 \leq p < \infty$. The Wasserstein--$p$ distance between $\mu$ and $\pi$ is defined by
    \begin{equation*}
        W_p(\mu, \pi) \coloneqq \left( \inf_{\gamma \in \Gamma(\mu, \pi)} \E_{(x,y) \sim \gamma} \Big[ \|x-y\|^p \Big] \right)^{\frac{1}{p}}.
    \end{equation*}
    
\end{definition}

The Rényi divergence results (Theorems~\ref{thm:RenyiMixing} and~\ref{thm:RenyiTarget}) strengthen Orlicz-Wasserstein guarantees by upgrading them to bounds in Rényi divergence.  To formalize these notions, we first recall the definition of the Orlicz norm of a random variable~\cite[Chapter~2.7.1]{vershynin2018high}, and then use it to define the Orlicz-Wasserstein distance between probability measures (Definition~\ref{def:OrliczWassersteinDistance}).

\begin{definition}[Orlicz norm]\label{def:OrliczNorm}
    Let $\psi: \R_{\geq 0} \to \R_{\geq 0}$ be an increasing, convex function satisfying $\psi(0) = 0$ and $\psi(x) \to \infty$ as $x \to \infty$.
    The Orlicz norm of a random variable $X \sim \mu$ with respect to the Orlicz function $\psi$ is defined as
    \[
    \|X\|_{\psi} \coloneqq \inf \left\{ \lambda > 0 : \E_{X \sim \mu} \left[ \psi \left(\frac{\|X\|}{\lambda}\right) \right] \leq 1   \right\} \,.
    \]
\end{definition}

\begin{remark}
Different choices of the Orlicz function $\psi$ encode different tail behaviors
of the random variable $X$. For example, if $\psi(x)=e^{x}-1$, finiteness of
$\|X\|_{\psi}$ is equivalent to $X$ having sub-exponential tails, whereas if
$\psi(x)=e^{x^{2}}-1$, finiteness of $\|X\|_{\psi}$ corresponds to sub-Gaussian
tail behavior. See \cite[Chapters~2.5 and~2.7]{vershynin2018high} for a detailed
discussion of these equivalences.
\end{remark}

\begin{remark}
Throughout this paper, we fix the Orlicz function
\[
\psi(x) \coloneqq e^{x^{2}} - 1,
\]
so that the associated Orlicz norm controls sub-Gaussian tails.
\end{remark}

The moment generating function (m.g.f.) of a real-valued random variable $X \sim \mu$ is the function $M_X: \R \to \R_{>0}$ defined by $M_X(t) \coloneqq \E_{X \sim \mu}[e^{tX}]$, whenever the expectation is finite. The following lemma, which bounds the m.g.f. of $\|X\|^2$ under a finite Orlicz norm assumption, will be useful in our analysis.

\begin{lemma}\label{lem:Helper_Orlicz}
    Suppose that $\|X\|_{\psi} \leq K$, where $\psi(x) = e^{x^2}-1$. Then for $c \in [0, K^{-2}]\,,$
    \[
    \E \left[ e^{c \|X\|^2} \right] \leq 2^{cK^2}.
    \]
\end{lemma}

\begin{proof}
By definition of the Orlicz norm and the assumption $\|X\|_{\psi}\le K$,
\[
\E\!\left[e^{\|X\|^2/K^2}\right] \le 2.
\]
For $c\in[0,K^{-2}]$, we write
\[
\E\!\left[e^{c\|X\|^2}\right]
=
\E\!\left[\left(e^{\|X\|^2/K^2}\right)^{cK^2}\right].
\]
Since $cK^2\in[0,1]$ and the function $t\mapsto t^{cK^2}$ is concave on
$(0,\infty)$, Jensen’s inequality yields
\[
\E\!\left[e^{c\|X\|^2}\right]
\le
\left(\E\!\left[e^{\|X\|^2/K^2}\right]\right)^{cK^2}
\le
2^{cK^2}.
\]
\end{proof}

\begin{definition}[Orlicz-Wasserstein distance]\label{def:OrliczWassersteinDistance}
Let $\mu, \pi \in \P(\R^d)$. The Orlicz--Wasserstein distance between $\mu$ and $\pi$ is defined by
\[
W_{\psi}(\mu, \pi) \coloneqq \inf_{\substack{\gamma \in \Gamma(\mu, \pi)\\(X,Y) \sim \gamma}
} \|X-Y\|_{\psi} =  \inf_{\gamma \in \Gamma(\mu, \pi)} \inf \left\{ \lambda > 0 : \E_{(x,y)\sim\gamma}\left[ e^{\frac{\|x-y\|^2}{\lambda^2}}  \right] \leq 2 \right\}.
\]
\end{definition}

If $W_{\psi} (\mu, \pi) \leq c$ (for some constant $c > 0$), this means there exists a coupling $\gamma^* \in \Gamma(\mu, \pi)$ and some $\lambda^* \leq c$ such that 
\[
\E_{(x,y) \sim \gamma^*} \left[ e^{\frac{\|x-y\|^2}{{\lambda^{*}}^2}} \right] \leq 2\,.
\]
Thus, in the same way that Wasserstein--$2$ distance controls the expected squared difference under an optimal coupling, the Orlicz-Wasserstein distance controls an exponential moment of the squared difference.

Orlicz-Wasserstein distances have recently been used in sampling, in particular for analyses in Rényi divergence~\cite{altschuler2024faster, altschuler2023shifted, altschuler2024shifted}, and we note that Orlicz-Wasserstein upper bounds Wasserstein--$p$ distance for any finite $p$~\cite[Remark~3.4]{altschuler2024faster}, i.e., for any $\mu, \pi \in \P(\R^d)$ and $p \geq 1$
\[
\frac{1}{\sqrt{2p}} W_p(\mu, \pi) \leq W_{\psi}(\mu, \pi)\,.
\]

We now define the KL divergence and Rényi divergence between probability distributions.

\begin{definition}[KL divergence]\label{def:KLDivergence}
    The KL divergence between probability distributions $\mu, \pi \in \P(\R^d)$ with $\mu \ll \pi$ is defined by
    \begin{equation*}
        \KL(\mu \dvert \pi) \coloneqq  \E_{x \sim\mu} \left[\log \frac{\mu(x)}{\pi(x)}\right].
    \end{equation*}
    If $\mu \nll \pi$, then $\KL(\mu \dvert \pi) \coloneqq \infty$.
\end{definition}

\begin{definition}[Rényi--$q$ divergence]\label{def:RenyiDivergence}
    Let $q$ be an integer $1<q<\infty$. The Rényi divergence of order $q$ between probability distributions $\mu, \pi \in \P(\R^d)$ with $\mu \ll \pi$ is defined by
    \begin{equation*}
        \sR_q(\mu\dvert\pi) = \frac{1}{q-1} \log \E_{x \sim \pi}\left[\left(\frac{\mu(x)}{\pi(x)}\right)^q\right].
    \end{equation*}
    If $\mu \nll \pi$, then $\sfR_q(\mu \dvert \pi) \coloneqq \infty$.
\end{definition}

The Rényi divergence of order $q=1$ is defined in the limiting sense and coincides with the KL divergence (Definition~\ref{def:KLDivergence}), i.e., $\sR_1(\mu\dvert\pi) \coloneqq \lim_{q \to 1} \sR_q(\mu \dvert\pi) = \KL(\mu \dvert \pi)$. 
Further properties and background on KL divergence and Rényi divergence can be found in~\cite{van2014renyi, polyanskiywu_book}.
Rényi divergences are monotone in their order~\cite[Theorem~3]{van2014renyi}: for any distributions $\mu, \pi \in \P(\R^d)$ and $1\leq q \leq q' < \infty$\,,
\begin{equation}\label{eq:RenyiMonotonicity}
\sfR_q(\mu \dvert \pi) \leq \sfR_{q'}(\mu \dvert \pi)\,,
\end{equation}
and also satisfy the following weak triangle inequality~\cite[Proposition~11]{mironov2017renyi}\footnote{We present a simplified upper bound; a more precise statement can be found in~\cite[Proposition~11]{mironov2017renyi}.}, i.e., for any distributions $\mu, \pi, \rho \in \P(\R^d)$ and $q \geq 2$\,,
\begin{equation}\label{eq:RenyiWTI}
\sfR_q(\mu \dvert \pi) \leq \frac{3}{2}\sfR_{2q}(\mu \dvert \rho) + \sfR_{2q-1}(\rho \dvert \pi)\,.
\end{equation}

Rényi divergences for large enough order can be infinite even for absolutely continuous measures, e.g.,~\cite[Example~2]{VW23} shows that
\[
\sfR_q(\N(0, \sigma^2 I) \dvert \N(0, \lambda^2 I)) = \infty
\]
if $\sigma^2 > \lambda^2$ and $q \geq \frac{\sigma^2}{\sigma^2 - \lambda^2}$, reflecting a mismatch in Gaussian tail decay rates that causes divergence despite absolute continuity.

\subsection{Hamiltonian Monte Carlo}\label{subsec:HMC}

We briefly recall the Hamiltonian Monte Carlo (HMC) framework and fix notation.
As recounted below, HMC targets a probability distribution $\nu$ on $\R^d$
with density proportional to $e^{-f}$, where
$f \colon \R^d \to \R$ is a continuously differentiable, potential function.
The method augments the state space with auxiliary velocity variables and uses Hamiltonian dynamics in the resulting phase space to define a Markov transition kernel on $\mathbb{R}^d$, as originally proposed in \cite{duane1987hybrid} and described pedagogically in \cite{neal2011mcmc}.

We begin by introducing the joint position--velocity (phase) space $\R^{2d}$.
Throughout, we restrict attention to the simplest setting of a \emph{unit mass
matrix}.  This restriction entails no loss of generality: any constant
positive-definite mass matrix can be absorbed by a linear change of variables in
velocity (or momentum), also known as passing to \emph{mass-weighted
coordinates}, yielding an equivalent Hamiltonian system with unit mass.  Since
our arguments rely only on structural properties of the Hamiltonian flow and not
on the specific parametrization of phase space, we adopt the unit-mass setting
for clarity.

Accordingly, we define the Hamiltonian $H \colon \R^{2d} \to \R$ by
\[
H(x,v) \coloneqq f(x) + \tfrac12 \|v\|^2 .
\]
The associated Hamiltonian dynamics $(x_t,v_t)\in\R^{2d}$ are given by the
system of ordinary differential equations
\begin{equation}\label{eq:HamiltonianDynamics}
\frac{\d x_t}{\d t} = \nabla_v H(x_t,v_t) = v_t,
\qquad
\frac{\d v_t}{\d t} = -\nabla_x H(x_t,v_t) = -\nabla f(x_t).
\end{equation}
For an initial condition $(x_0,v_0)=(x,v)$, we denote by
$(x_t(x,v),v_t(x,v))$ the solution to \eqref{eq:HamiltonianDynamics} at time
$t\ge0$.
The corresponding Hamiltonian flow map is
\[
\HF_t \colon \R^{2d} \to \R^{2d}, \qquad
\HF_t(x,v) \coloneqq (x_t(x,v),v_t(x,v)).
\]
We further define the position projection
\[
q_t \coloneqq \Pi_1 \circ \HF_t ,
\]
so that $q_t(x,v)=x_t(x,v)$ for all $(x,v)\in\R^{2d}$.

The \emph{exact Hamiltonian Monte Carlo} (\eHMC) algorithm
\cite{bou2020coupling,bou2023convergence}
constructs a Markov chain $(X_k)_{k\ge0}$ on $\R^d$ as follows.
Given an initial distribution $X_0\sim\rho$, the update rule is
\begin{equation}\label{eq:eHMC_update}
X_{k+1} = q_T(X_k,\xi_k),
\end{equation}
where $(\xi_k)_{k\ge0}$ are independent draws from the standard Gaussian
distribution $\gammad$, and $T>0$ is a fixed \emph{integration time}.
We denote the associated transition kernel by
\[
\bP_T \colon \P(\R^d) \to \P(\R^d),
\]
so that if $X_k\sim\rho_k$, then $X_{k+1}\sim\rho_{k+1}=\rho_k \bP_T$, and hence
$\rho_k = \rho \bP_T^k$.
Although not immediate from the above construction, the target distribution
$\nu$ is invariant for \eHMC, i.e., $\nu \bP_T = \nu$.  This follows from the fact that the Hamiltonian flow preserves the Hamiltonian --- and hence the density $e^{-H}$ --- and that the velocity refreshment step preserves the auxiliary momentum distribution.  Moreover, \eHMC is reversible with respect to $\nu$; this additional property relies on the time-reversibility of the Hamiltonian flow together with the fact that the velocity refreshment step resamples the velocity independently from the auxiliary momentum distribution.  See Appendix~\ref{app:HMC-background} for a concise review.

The \eHMC update~\eqref{eq:eHMC_update} requires solving the Hamiltonian dynamics~\eqref{eq:HamiltonianDynamics} \textit{exactly}, which is generally infeasible.
Replacing the exact flow by a numerical integrator introduces discretization
error, and in the absence of a Metropolis correction the resulting method falls
under the umbrella of \emph{unadjusted Hamiltonian Monte Carlo} (\uHMC)
\cite{bou2023convergence,bou2023mixing}. The term \uHMC encompasses any discretization-based
approximation of the Hamiltonian flow. In this paper, we focus on the velocity Verlet (leapfrog) integrator
\cite{bou2018geometric,hairer2006geometric}, leading to the \uHMCv algorithm
studied here. Extensions to other discretization schemes, such as stratified
Monte Carlo integrators~\cite{bou2025unadjusted}, are discussed in
Section~\ref{subsubsection:Examples-stratified}.

\begin{algorithm}[t]
\caption{Velocity Verlet integrator ($\tilde \HF_{t_1,t_2}$)}
\label{alg:VerletFlow}
\begin{algorithmic}[1]
\STATE \textbf{Input:} Integration time $t_1>0$, step size $t_2>0$ with
$t_2 \mid t_1$, and initial condition $(x,v)\in\R^{2d}$.
\STATE Set $x_0 \coloneqq x$, $v_0 \coloneqq v$, and $N \coloneqq t_1/t_2$.
\FOR{$j = 0$ to $N-1$}
\STATE $x_{j+1} = x_j + t_2 v_j - \tfrac{t_2^2}{2}\nabla f(x_j)$
\STATE $v_{j+1} = v_j - \tfrac{t_2}{2}\bigl(\nabla f(x_j) + \nabla f(x_{j+1})\bigr)$
\ENDFOR
\STATE \textbf{Output:} $(x_N, v_N)$
\end{algorithmic}
\end{algorithm}

To define \uHMCv, we introduce a discrete-time approximation of the Hamiltonian
flow based on the velocity Verlet integrator. Let $\tilde \HF_{t_1,t_2} \colon \R^{2d} \to \R^{2d}$ denote the numerical flow
map associated with the velocity Verlet scheme with total integration time
$t_1>0$ and step size $t_2>0$, where $t_2$ divides $t_1$.
For any initial condition $(x,v)\in\R^{2d}$, the value
$\tilde \HF_{t_1,t_2}(x,v)$ is defined by Algorithm~\ref{alg:VerletFlow}.

Under standard regularity assumptions on $f$, the velocity Verlet integrator is
a second-order accurate, symplectic approximation of the Hamiltonian dynamics,
and $\tilde \HF_{t_1,t_2}(x,v)$ converges to the exact flow
$\HF_{t_1}(x,v)$ as $t_2 \to 0$ for fixed $t_1$; see, e.g.,
\cite{hairer2006geometric,bou2018geometric}.

We define the associated position map
\[
\tilde q_{t_1,t_2} \coloneqq \Pi_1 \circ \tilde \HF_{t_1,t_2},
\]
so that $\tilde q_{t_1,t_2}(x,v)$ equals the position component of the numerical
trajectory after $t_1/t_2$ integration steps.

\smallskip
Starting from an initial distribution $X_0 \sim \rho$, the \uHMCv algorithm
generates a Markov chain $(X_k)_{k\ge0}$ on $\R^d$ via the update
\begin{equation}\label{eq:uHMCv_update}
X_{k+1} = \tilde q_{T,h}(X_k,\xi_k),
\end{equation}
where $(\xi_k)_{k\ge0}$ are independent draws from the standard Gaussian
distribution $\gammad$, $T>0$ is the integration time, and $h>0$ is the step
size.
This update defines a Markov transition kernel
\[
\tilde \bP_{T,h} \colon \P(\R^d) \to \P(\R^d),
\]
so that the law of $X_k$ is given by $\rho \tilde \bP_{T,h}^k$.

Under suitable regularity and stability assumptions on $f$, $T$, and $h$, the kernel $\tilde \bP_{T,h}$ admits an invariant distribution, which we denote by $\tilde \nu_h$.  In general, $\tilde \nu_h$ depends on $h$ and differs from the target
distribution $\nu$.
Moreover, under additional assumptions, $\tilde \nu_h$ converges to $\nu$
as $h \to 0$; see, for example, \cite{bou2023convergence,bou2023mixing}.

Throughout this manuscript, we assume that whenever the velocity Verlet step
size $h>0$, it divides the integration time $T>0$, i.e., $T/h \in \mathbb{N}$.
We also note that in the formal limit $h=0$, the numerical flow
$\tilde \HF_{T,h}$ coincides with the exact Hamiltonian flow $\HF_T$, and
\uHMCv reduces to \eHMC.

\begin{remark}[Reversibility]
The \eHMC kernel $\bP_T$ is reversible with respect to
the target distribution $\nu$.
By contrast, the \uHMC kernel $\tilde \bP_{T,h}$ is generally not
reversible with respect to its invariant distribution $\tilde \nu_h$, when
such an invariant distribution exists.
This lack of reversibility is a consequence of discretization error in the
numerical integrator and persists even for symplectic schemes such as velocity
Verlet in the absence of a Metropolis correction.
\end{remark}

\section{One-shot Couplings}\label{sec:OneShotCouplings}

A \emph{one-shot coupling} is a coupling of the one-step transition kernels of two Markov chains \cite{roberts2002one,madras2010quantitative,bou2023mixing}.  Such
couplings are typically constructed to maximize the probability that the two
chains coalesce after a single transition. Traditionally, such couplings are applied to two copies of the same Markov chain and are used to control its mixing behavior.

In this work, we employ one-shot couplings in two conceptually distinct ways.
First, we use one-shot couplings between two copies of \uHMCv to establish regularization bounds, which quantify how a single transition smooths or contracts discrepancies between different initial conditions and lead directly to mixing-time guarantees.
Second, we construct one-shot couplings between \uHMCv and \eHMC to obtain cross-regularization bounds, which compare the evolution of the discretized chain to that of its continuous-time counterpart and allow us to control the asymptotic bias introduced by discretization.

Concretely, the KL regularization and cross-regularization bounds are established in Lemmas~\ref{lem:KL_Mixing_OneStepGeneral} and~\ref{lem:KL_Target_OneStepGeneral} respectively, while the corresponding R\'enyi regularization and cross-regularization bounds are established in Lemmas~\ref{lem:Renyi_Mixing_OneStepGeneral} and~\ref{lem:Renyi_Target_OneStepGeneral}.
All of these results are proved via appropriately constructed one-shot couplings.

One-shot couplings for \uHMCv were studied in~\cite{bou2023mixing} to upgrade
Wasserstein--$1$ convergence guarantees to total variation distance.
There, the velocity Verlet scheme is characterized as the solution of a
discrete-time variational problem.
Under appropriate conditions, this variational problem is shown to be strongly
convex, implying that the velocity Verlet update is the \emph{unique} minimizer
of the variational objective~\cite[Corollary~14]{bou2023mixing}.
This uniqueness property has important consequences for the analysis of
one-shot couplings for \uHMCv: it implies the existence and uniqueness of
one-shot coupling maps that couple the velocities of two copies of the Markov
chain so that the chains coalesce after a single step; see
\cite[Figure~1]{bou2023mixing}.

We adopt a closely related strategy here to obtain guarantees for \uHMCv in
Kullback--Leibler and R\'enyi divergences.
The one-shot coupling map used to derive mixing-time bounds
(Figure~\ref{fig:coupling}(a)) coincides with the map introduced
in~\cite{bou2023mixing}.
For the purpose of controlling asymptotic bias, we consider a variant of the
one-shot coupling construction (Figure~\ref{fig:coupling}(b)) that couples the
discretized kernel \uHMCv with the continuous-time kernel \eHMC.
While~\cite{bou2023mixing} focuses on couplings starting from the same initial
state, our analysis extends this perspective to couplings between discretized
and exact dynamics initialized at distinct points.

Related coupling arguments appear in Chapter~4 of the doctoral dissertation of Marsden \cite{Marsden2025Thesis}, where a one-shot coupling map is used to lower bound the restricted conductance profile for adjusted HMC implemented with randomized integrators. This analysis builds on the restricted conductance profile framework of \cite{chen2020fast}, which establishes mixing-time bounds for adjusted HMC by relating one-step overlap to conductance lower bounds.

\begin{remark}[Relation between kernel regularization and Harnack inequalities]
\label{rem:HarnackRenyiSide}
There are well-studied equivalences
between regularization properties of a Markov kernel in KL and Rényi divergences, 
Wang-type Harnack inequalities, and reverse functional inequalities~\cite{wang1997logarithmic,bobkov2001hypercontractivity, arnaudon2006harnack, bakry2015harnack, baudoin2021transportation, altschuler2023shifted}; see also Appendix~\ref{app:Regularization}. 
These equivalences provide a functional inequality perspective on regularization bounds, which we establish via one-shot couplings.
\end{remark}

We now introduce the regularity assumptions on the potential function $f$
(recall $\nu \propto e^{-f}$) under which estimates for the one-shot coupling
maps can be obtained.

\begin{assumption}\label{assumption:potential}
The potential function $f : \R^d \to \R$ satisfies the following conditions
\begin{itemize}[noitemsep]
    \item[(a)] The function $f$ has global minimum at $0$ and $f(0) = 0$.
    \item[(b)] The function $f \in \cC^2(\R^d)$ satisfies $\|\nabla^2f(x)\|_\op \leq L$ for all $x \in \R^d$.
    \item[(c)] The function $f \in \cC^3(\R^d)$ satisfies $\|\nabla^3f(x)\|_\op \leq M$ for all $x \in \R^d$.
    \item[(d)] The function $f \in \cC^4(\R^d)$ satisfies $\|\nabla^4f(x)\|_\op \leq N$ for all $x \in \R^d$.
    
\end{itemize}
Here, for $k \ge 2$, $\nabla^k f(x)$ is viewed as a symmetric $k$-linear form on
$(\R^d)^k$, and $\|\cdot\|_{\op}$ denotes the corresponding operator norm induced
by the Euclidean norm on $\R^d$, namely
\[
\|T\|_{\op}
\coloneqq
\sup\bigl\{ |T(u_1,\ldots,u_k)| : \|u_1\|=\cdots=\|u_k\|=1 \bigr\}.
\]
\end{assumption}

The one-shot coupling map introduced for the analysis of mixing times
(Section~\ref{sec:MixingTimeOneShotMap}) allows us to control the divergence
after a single transition by reducing the problem to bounding the divergence
between a standard and perturbed Gaussian distribution.
The perturbation is induced by the one-shot coupling map, as made precise in the
following lemma.
The proof of Lemma~\ref{lem:DiracToPerturbed_Mixing} is given in
Section~\ref{app:RegularityProofs}.

\begin{lemma}\label{lem:DiracToPerturbed_Mixing}
    Recall that $\tilde \bP_{T,h}$ denotes the \uhmcv transition kernel with update given by~\eqref{eq:uHMCv_update}.
    Let $x,y \in \R^d$ be any two points, $\sfD \in \{ \KL, \sfR_q\}$ be the KL divergence (Definition~\ref{def:KLDivergence}) or Rényi divergence (Definition~\ref{def:RenyiDivergence}), and $\varphi_{x,y} : \R^d \to \R^d$ be a map satisfying~\eqref{eq:MixingMap}. Then for potential function satisfying Assumption~\ref{assumption:potential}(b), any integration time $T>0$ such that $LT^2 \leq \frac{2}{5}\pi^2$, and step-size $h \geq 0$\,,
    \[
    \sfD(\delta_x \tilde \bP_{T,h} \dvert \delta_y \tilde \bP_{T,h}) \leq \sfD((\varphi_{x,y})_\# \gammad \dvert \gammad)\,.
    \]
\end{lemma}

\begin{remark}
We assume that the step size $h>0$ divides the integration time $T$, i.e.,
$T = m h$ for some integer $m\ge 1$.  Combined with the constraint
$LT^2 \le \frac{2}{5}\pi^2$, this implies
\[
L h^2 = L\Bigl(\frac{T}{m}\Bigr)^2 \le L T^2 \le \frac{2}{5}\pi^2 < 4,
\]
and hence $h\sqrt{L}<2$, which is the stability condition for
the velocity Verlet integrator \cite{bou2018geometric}.
\end{remark}

To obtain bounds on the asymptotic bias, we prove cross-regularization bounds for \uhmcv.
The one-shot coupling map introduced for the analysis of asymptotic bias (Section~\ref{sec:TargetMap}) enables us to bound the divergence after one step in terms of the divergence between a standard and a perturbed Gaussian distribution.  As in the mixing-time analysis, the perturbation is induced by the one-shot coupling map itself, as formalized in the following lemma. The proof of Lemma~\ref{lem:DiracToPerturbed_Target} is given in Section~\ref{app:RegularityProofs}.

\begin{lemma}\label{lem:DiracToPerturbed_Target}
    Recall that $\tilde \bP_{T,h}$ and $\bP_T$ denote the \uhmcv and \ehmc transition kernels with updates given by~\eqref{eq:uHMCv_update} and~\eqref{eq:eHMC_update} respectively.
    Let $x,y \in \R^d$ be any two points, $\sfD \in \{ \KL, \sfR_q\}$ be the KL divergence (Definition~\ref{def:KLDivergence}) or Rényi divergence (Definition~\ref{def:RenyiDivergence}), and $\Phi_{x,y} : \R^d \to \R^d$ be a map satisfying~\eqref{eq:TargetMap}. Then for potential function satisfying Assumption~\ref{assumption:potential}(b), any integration time $T>0$ such that $LT^2 \leq \frac{2}{5}\pi^2$, and step-size $h \geq 0$\,,
    \[
    \sfD(\delta_x \tilde \bP_{T,h} \dvert \delta_y  \bP_{T}) \leq \sfD((\Phi_{x,y})_\# \gammad \dvert \gammad)\,.
    \]
\end{lemma}

Lemmas~\ref{lem:DiracToPerturbed_Mixing} and~\ref{lem:DiracToPerturbed_Target} show that both regularization and cross-regularization guarantees starting from Dirac initializations can be reduced to bounding the divergence between a standard Gaussian distribution and a perturbed Gaussian distribution. In each case, the perturbation is induced by the corresponding one-shot coupling map. The remainder of this section introduces these coupling maps and establishes regularity bounds for them, which are then used to control the divergence between the perturbed and standard Gaussian distributions.

\subsection{One-shot map for regularization}\label{sec:MixingTimeOneShotMap}

For any $x,y\in\R^d$, let $\varphi_{x,y}:\R^d\to\R^d$ denote a map satisfying
\begin{equation}\label{eq:MixingMap}
    \tilde q_{T,h}(x,v) = \tilde q_{T,h}\bigl(y,\varphi_{x,y}(v)\bigr),
    \qquad \forall\, v\in\R^d.
\end{equation}
That is, $\varphi_{x,y}$ maps an initial velocity $v$ at position $x$ to a
velocity at position $y$ such that the two discretized trajectories coincide in
position after time $T$.

This one-shot coupling map is illustrated in
Figure~\ref{fig:coupling}(a) and was introduced in~\cite{bou2023mixing} to obtain
mixing-time guarantees for \uhmcv in total variation distance.  For any
$x,y\in\R^d$, a map $\varphi_{x,y}$ satisfying~\eqref{eq:MixingMap} exists
whenever $f$ is $L$--smooth and $LT^2\le \frac{2}{5}\pi^2$
\cite[Corollary~14]{bou2023mixing}.

We have the following second-order regularity estimates for the map and its Jacobian. Lemmas~\ref{lem:MixingMapPtWise}
and~\ref{lem:MixingMapJacobian} are proved in
Section~\ref{app:MixingOneShotMap}.

\begin{lemma}\label{lem:MixingMapPtWise}
    Suppose Assumption~\ref{assumption:potential}(a)-(b) holds. Further suppose $h \geq 0$ and that $L(T^2+Th) \leq \frac{1}{12}$. Then for any $x,y,v \in \R^d$
    \begin{equation}\label{eq:MixingMapPtWise}
    \|\varphi_{x,y}(v)-v\| \leq \frac{3}{2T}\|x-y\|.
    \end{equation}
\end{lemma}

\begin{lemma}\label{lem:MixingMapJacobian}
    Suppose Assumption~\ref{assumption:potential}(a)-(c) holds. Further suppose $h \geq 0$ and that $L(T^2 +Th) \leq \frac{1}{12}$. Then for any $x,y,v \in \R^d$
    \begin{equation}\label{eq:MixingMapJacobian}
    \|\nabla \varphi_{x,y}(v)-I\|_\op \leq \min \left\{\frac{2}{9}, \frac{11}{2}MT^2\|x-y\| \right\}.
    \end{equation}
\end{lemma}

\begin{remark}
Lemmas~\ref{lem:MixingMapPtWise} and~\ref{lem:MixingMapJacobian} are restatements,
in our notation and under our assumptions, of Lemmas~25 and~26
of~\cite{bou2023mixing}.  We include them here for completeness.
\end{remark}

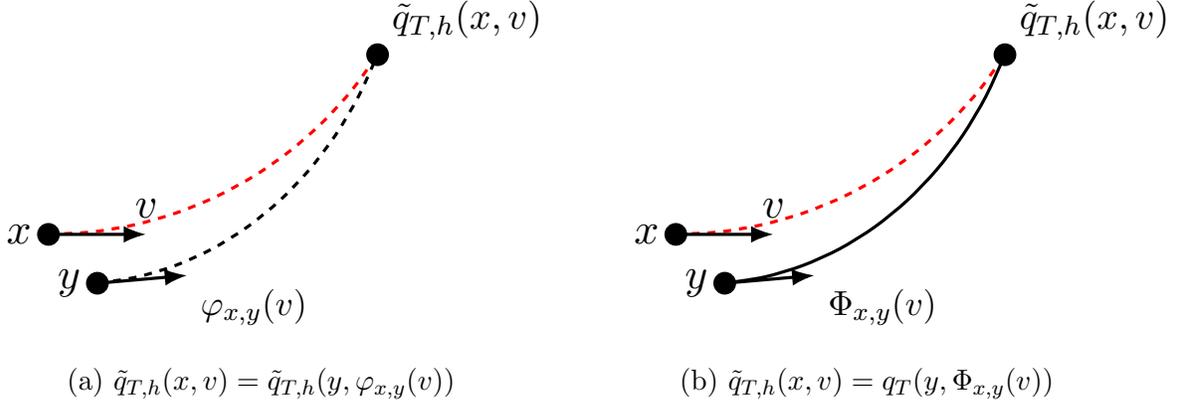
\begin{figure}[t]
\centering
\begin{minipage}{\textwidth} 
      \begin{minipage}[b]{0.50\textwidth}
      \begin{center}
\begin{tikzpicture}[scale=1.3]
\begin{scope}[very thick] 
       \tikzmath{
       \t=1.0; \x1 = 0.0; \x2 =-4.0; \y1=0.5; \y2=-4.5;  \u1=4.; \u2=0.0; \om=1.;
                      \v1=\u1+\om*cos(deg(\om*\t))/sin(deg(\om*\t))*(\x1-\y1); 
                      \v2=\u2+\om*cos(deg(\om*\t))/sin(deg(\om*\t))*(\x2-\y2); 
                      \z1=cos(deg(\om*\t))*\x1+sin(deg(\om*\t))*\u1/\om;
                      \z2=cos(deg(\om*\t))*\x2+sin(deg(\om*\t))*\u2/\om;
                      } 
                       \draw  [color=red,dashed,smooth,domain=0:\t,samples=10,variable=\xx]  plot ({cos(deg(\om*\xx))*\x1+sin(deg(\om*\xx))*\u1/\om},{cos(deg(\om*\xx))*\x2+sin(deg(\om*\xx))*\u2/\om});          
                       \draw  [color=black,dashed,smooth,domain=0:\t,samples=10,variable=\x]  plot ({cos(deg(\om*\x))*\y1+sin(deg(\om*\x))*\v1/\om},{cos(deg(\om*\x))*\y2+sin(deg(\om*\x))*\v2/\om});
\filldraw[color=black,fill=black] (\x1,\x2) circle (0.1) node [left,black, scale=1.5]  {$x$};
\filldraw[color=black,fill=black] (\y1,\y2) circle (0.1) node [left,black, scale=1.5]  {$y$};
\filldraw[color=black,fill=black] (\z1,\z2) circle (0.1) node [above right,black, scale=1.3] {$\tilde q_{T,h}(x,v)$};
\draw[->,-{Latex[length=3mm]}](\x1,\x2)--(\x1+0.25*\u1,\x2+0.25*\u2) node [above,black, scale=1.5] {$v$};
\draw[->,-{Latex[length=3mm]}](\y1,\y2)--(\y1+0.25*\v1,\y2+0.25*\v2) node [below right,black, scale=1.2] {$\varphi_{x,y}(v)$};
 \end{scope}
\end{tikzpicture}
\end{center}
\end{minipage}
      \begin{minipage}[b]{0.50\textwidth}
      \begin{center}
\begin{tikzpicture}[scale=1.3]
\begin{scope}[very thick] 
       \tikzmath{
       \t=1.0; \x1 = 0.0; \x2 =-4.0; \y1=0.5; \y2=-4.5;  \u1=4.; \u2=0.0; \om=1.;
                      \v1=\u1+\om*cos(deg(\om*\t))/sin(deg(\om*\t))*(\x1-\y1); 
                      \v2=\u2+\om*cos(deg(\om*\t))/sin(deg(\om*\t))*(\x2-\y2); 
                      \z1=cos(deg(\om*\t))*\x1+sin(deg(\om*\t))*\u1/\om;
                      \z2=cos(deg(\om*\t))*\x2+sin(deg(\om*\t))*\u2/\om;
                      } 
                       \draw  [color=red,dashed,smooth,domain=0:\t,samples=10,variable=\xx]  plot ({cos(deg(\om*\xx))*\x1+sin(deg(\om*\xx))*\u1/\om},{cos(deg(\om*\xx))*\x2+sin(deg(\om*\xx))*\u2/\om});          
                       \draw  [color=black,smooth,domain=0:\t,samples=10,variable=\x]  plot ({cos(deg(\om*\x))*\y1+sin(deg(\om*\x))*\v1/\om},{cos(deg(\om*\x))*\y2+sin(deg(\om*\x))*\v2/\om});
\filldraw[color=black,fill=black] (\x1,\x2) circle (0.1) node [left,black, scale=1.5]  {$x$};
\filldraw[color=black,fill=black] (\y1,\y2) circle (0.1) node [left,black, scale=1.5]  {$y$};
\filldraw[color=black,fill=black] (\z1,\z2) circle (0.1) node [above right,black, scale=1.3] {$\tilde q_{T,h}(x,v)$};
\draw[->,-{Latex[length=3mm]}](\x1,\x2)--(\x1+0.25*\u1,\x2+0.25*\u2) node [above,black, scale=1.5] {$v$};
\draw[->,-{Latex[length=3mm]}](\y1,\y2)--(\y1+0.25*\v1,\y2+0.25*\v2) node [below right,black, scale=1.2] {$\Phi_{x,y}(v)$};
 \end{scope}
\end{tikzpicture}
\end{center}
\end{minipage}
\end{minipage}
\begin{minipage}{\textwidth}
\vspace{1em}
      \begin{minipage}[b]{0.48\textwidth}
      \centering
{\normalsize (a) $\tilde q_{T,h}(x,v) = \tilde q_{T, h}(y, \varphi_{x,y}(v))$}
\end{minipage}
\begin{minipage}[b]{0.48\textwidth}
\centering
{\normalsize (b) $\tilde q_{T,h}(x,v) = q_T(y, \Phi_{x,y}(v))$}
\end{minipage}
\end{minipage}
\caption{One-shot couplings. (a) To obtain regularization bounds for \uHMCv, the initial velocities are coupled such that $\tilde q_{T,h}(x,v) = \tilde q_{T, h}(y, \varphi_{x,y}(v))$. (b) To obtain cross-regularization bounds for \uhmcv, the initial velocities are coupled such that $\tilde q_{T,h}(x,v) = q_T(y, \Phi_{x,y}(v))$.}
  \label{fig:coupling}
\end{figure}

\subsection{One-shot map for cross-regularization}\label{sec:TargetMap}

For any $x,y\in\R^d$, let $\Phi_{x,y}:\R^d\to\R^d$ denote a map satisfying
\begin{equation}\label{eq:TargetMap}
    \tilde q_{T,h}(x,v) = q_T\bigl(y,\Phi_{x,y}(v)\bigr),
    \qquad \forall\, v\in\R^d.
\end{equation}
That is, the one-shot map $\Phi_{x,y}$ maps an initial velocity $v$ at position $x$ for the
discretized dynamics to a velocity at position $y$ for the exact Hamiltonian
flow such that the two trajectories coincide in position after time $T$.

This map is illustrated in Figure~\ref{fig:coupling}(b).  For any $x,y\in\R^d$,
the existence of a map $\Phi_{x,y}$ satisfying~\eqref{eq:TargetMap} under the
assumptions that $f$ is $L$--smooth and $LT^2\le \frac{2}{5}\pi^2$ is established
in Lemma~\ref{lem:ExistenceOfTargetMap} in
Section~\ref{app:TargetMap}.  As described in
Section~\ref{app:TargetMap}, such a map can be analyzed by viewing it as a
composition of the one-shot coupling maps for \uhmcv studied
in~\cite{bou2023mixing}.

We have the following second-order regularity estimates for the map and its Jacobian which are helpful in obtaining asymptotic bias guarantees in KL divergence (Theorem~\ref{thm:KLTarget}).  Lemmas~\ref{lem:TargetMapPtWise}
and~\ref{lem:TargetMapJacobian} are proved in
Section~\ref{app:TargetMap}.

\begin{lemma}\label{lem:TargetMapPtWise}
    Suppose Assumption~\ref{assumption:potential}(a)-(c) holds. Further suppose $h \geq 0$ and that $L(T^2+Th) \leq \frac{1}{12}$. Then for any $x,y, v \in \R^d$
    \[
    \|\Phi_{x,y}(v)-v\| \leq \frac{3}{2T}\|x-y\| + 2h^2(LT^{-1}\|x\|+L\|v\|+MT^{-1}\|x\|^2+MT\|v\|^2)\,.
    \]
\end{lemma}

\begin{lemma}\label{lem:TargetMapJacobian}
    Suppose Assumption~\ref{assumption:potential}(a)-(d) holds. Further suppose $h \geq 0$ and that $L(T^2+Th) \leq \frac{1}{12}$. Then for any $x,y, v \in \R^d$
    \[
    \|\nabla\Phi_{x,y}(v)-I\|_\op \leq \min \left\{ \frac{15}{18}, \frac{11}{2}MT^2\|x-y\| + \frac{22}{9}h^2 Q(x,v) \right\}
    \]
    where 
    \[
    Q(x,v) \coloneqq L+M\|x\|+MT\|v\|+(M^2T^2+N)\|x\|^2+(M^2T^2+N)T^2\|v\|^2\,.
    \]
\end{lemma}

Lemmas~\ref{lem:TargetMapPtWise} and~\ref{lem:TargetMapJacobian} are obtained by
analyzing a second-order trapezoidal approximation and therefore exhibit
$h^2$--dependence on the step size, together with quadratic dependence on the
position $x$ and velocity $v$.  The quadratic dependence on $v$ poses
difficulties for R\'enyi divergence bounds, since controlling the R\'enyi
divergence between a perturbed Gaussian and a standard Gaussian would require
bounds on moments of the form
$\E_{v\sim\gamma_d}[e^{\lambda\|v\|^4}]$, which diverge for all $\lambda>0$.

To obtain finite cross-regularization bounds in R\'enyi divergence
(Lemma~\ref{lem:Renyi_Target_OneStepRegularity}), we therefore derive weaker
pointwise and Jacobian estimates for $\Phi_{x,y}$ based on a first-order
approximation, which do not exploit the second-order accuracy of the velocity
Verlet integrator.  The resulting bounds are presented in
Lemmas~\ref{lem:TargetMap-First-PtWise}
and~\ref{lem:TargetMap-First-Jacobian} and are proved in
Appendix~\ref{app:TargetMap-First}.

\begin{lemma}\label{lem:TargetMap-First-PtWise}
    Suppose Assumption~\ref{assumption:potential}(a)-(b) holds. Further suppose $h \geq 0$ and that $L(T^2+Th) \leq \frac{1}{12}$. Then for any $x,y, v \in \R^d$
    \[
    \|\Phi_{x,y}(v)-v\| \leq \frac{3}{2T}\|x-y\| + \frac{7}{5}h \left( \frac{1}{5T}\|v\| + \frac{7}{36}L\|x\| \right).
    \]
\end{lemma}

\begin{lemma}\label{lem:TargetMap-First-Jacobian}
    Suppose Assumption~\ref{assumption:potential}(a)-(c) holds. Further suppose $h \geq 0$ and that $L(T^2+Th) \leq \frac{1}{12}$. Then for any $x,y, v \in \R^d$
    \[
    \|\nabla\Phi_{x,y}(v)-I\|_\op \leq \min \left\{ \frac{15}{18}, \frac{11}{2}MT^2\|x-y\| + \frac{22}{135}h\, \left( \frac{2}{T} + \frac{16}{5}MT \|x\| + 20MT^2 \|v\| \right) \right\}\,.
    \]
\end{lemma}

To simplify notation in the remainder of the paper, we introduce the following
constants, which collect the coefficients appearing in the pointwise and
Jacobian estimates of Lemmas~\ref{lem:TargetMap-First-PtWise}
and~\ref{lem:TargetMap-First-Jacobian}.

\begin{definition}\label{def:Constants}
    Define the following constants which correspond to the coefficients of the pointwise and Jacobian estimates in Lemmas~\ref{lem:TargetMap-First-PtWise} and~\ref{lem:TargetMap-First-Jacobian}\,:
    \begin{align*}
        &\sfp_{xy} \coloneqq \frac{3}{2 \, T} \enspace,~~
        &\sfp_v \coloneqq \frac{7}{25 \, T} \enspace,~~\qquad
        &\sfp_x \coloneqq \frac{49}{180}L \enspace,\\
        &\sfj_{xy} \coloneqq \frac{11}{2}MT^2 \enspace,~~
        &\sfj_c \coloneqq \frac{44}{135 \, T} \enspace,~~\qquad
        &\sfj_v \coloneqq \frac{440}{135}MT^2 \enspace,~~\qquad
        \sfj_x \coloneqq \frac{352}{675}MT \enspace.
    \end{align*}
\end{definition}

With this notation, Lemma~\ref{lem:TargetMap-First-PtWise} can be written as
\[
\|\Phi_{x,y}(v)-v\|
\le
\sfp_{xy}\|x-y\| + h\,\sfp_v\|v\| + h\,\sfp_x\|x\|,
\]
and Lemma~\ref{lem:TargetMap-First-Jacobian} as
\[
\|\nabla \Phi_{x,y}(v)-I\|_\op
\le
\min\!\left\{
\frac{15}{18},\;
\sfj_{xy}\|x-y\| + h\,\sfj_c + h\,\sfj_v\|v\| + h\,\sfj_x\|x\|
\right\}.
\]

\section{Main Results and Examples}

In this section we present our main results establishing mixing time and asymptotic bias bounds for \uhmc in both KL and Rényi divergences. We then instantiate these results and discuss examples for strongly log-concave targets in Section~\ref{subsec:Examples}, and present an application to studying information contraction in Section~\ref{subsec:InformationContraction}.

\subsection{Guarantees in KL Divergence}\label{subsec:Main-KL}

We begin by stating the mixing time and asymptotic bias guarantees in KL divergence. 
Recall that our proof strategy upgrades Wasserstein--$2$ guarantees to KL divergence; therefore, to obtain mixing time guarantees for \uHMCv in KL divergence, we require the \uHMCv kernel $\tilde \bP_{T,h}$ to satisfy:

\begin{assumption}\label{assumption:W2Mixing}
    Let $\tilde \bQ$ be a discrete-time Markov transition kernel with stationary distribution $\tilde \nu_{\tilde \bQ}$ for which there exist constants $c_1 \geq 1$, $c_2 > 0$ such that for any integer $k \geq 1$ and initial probability distribution $\mu$
    \[
    W_2(\mu \tilde \bQ^k, \tilde \nu_{\tilde \bQ}) \leq c_1 e^{-c_2k} \,W_2(\mu, \tilde \nu_{\tilde \bQ})\,.
    \]
\end{assumption}

Our main result describing the mixing time of \uHMCv in KL divergence is the following theorem. We detail the main steps required to prove Theorem~\ref{thm:KLMixing} in Section~\ref{subsec:KLMixingSketch} and prove Theorem~\ref{thm:KLMixing} in Section~\ref{app:PfOf_ThmKLMixing}.

\begin{theorem}\label{thm:KLMixing}
    Let $\mu \in \P(\R^d)$.
    Assume Assumption~\ref{assumption:potential}(a)-(c) holds.  Assume that Assumption~\ref{assumption:W2Mixing} holds for
    $\tilde\bQ = \tilde\bP_{T,h}$, with invariant distribution $\tilde\nu_h$. 
    Assume $h \geq 0$ and $L(T^2 + Th) \leq \frac{1}{12}$.
    Then for any integer $k \geq 0$
    \[
    \KL(\mu \tilde \bP_{T,h}^{k+1} \dvert \tilde \nu_h) \leq c_1^2 e^{-2c_2 k} \left( \frac{9}{4T^2} + 20dM^2T^4 \right) W_2^2(\mu, \tilde \nu_h)\,.
    \]
\end{theorem}

Recall that $\tilde \nu_h$ is the invariant distribution of \uhmcv and Theorem~\ref{thm:KLMixing} shows that the law along \uhmcv converges to $\tilde \nu_h$ exponentially fast.
We verify that Assumption~\ref{assumption:W2Mixing} holds for \uhmcv for the case when $\nu$ is strongly log-concave in Proposition~\ref{prop:W2mixing-Verlet} in Section~\ref{subsubsec:Examples-Verlet}. Theorem~\ref{thm:KLMixing} together with Proposition~\ref{prop:W2mixing-Verlet} yields Corollary~\ref{cor:mixing-KL-verlet}, describing the number of iterations of \uhmcv required to make the KL divergence to the stationary distribution within the desired accuracy.

Asymptotic bias bounds in KL divergence (Theorem~\ref{thm:KLTarget}) hold for general kernels $\tilde \bQ$ so long as the last step of the Markov chain is the \uHMCv kernel $\tilde \bP_{T,h}$; these guarantees require the kernel $\tilde \bQ$ to satisfy Assumptions~\ref{assumption:W2Mixing} and~\ref{assumption:W2Bias}.

\begin{assumption}\label{assumption:W2Bias}
Let $\tilde \bQ$ be a discrete-time Markov transition kernel with stationary distribution $\tilde \nu_{\tilde \bQ}$ for which $W_2(\tilde \nu_{\tilde \bQ}, \nu) \leq \Delta_h$, where $\Delta_h \geq 0$ and $\Delta_h \to 0$ as $h \to 0$. 
\end{assumption}

Our main result describing the asymptotic bias of \uhmc in KL divergence is the following theorem. We sketch the proof of Theorem~\ref{thm:KLTarget} in Section~\ref{subsec:KLTargetSketch} and prove Theorem~\ref{thm:KLTarget} in Section~\ref{app:PfOf_ThmKLTarget}.

\begin{theorem}\label{thm:KLTarget}
    Let $\mu \in \P(\R^d)$ and suppose Assumption~\ref{assumption:potential}(a)-(d) holds. Further suppose $h \geq 0$, $L(T^2 + Th) \leq \frac{1}{12}$, and that Assumptions~\ref{assumption:W2Mixing} and~\ref{assumption:W2Bias} hold for a Markov kernel $\tilde \bQ$, with invariant distribution
$\tilde\nu_{\tilde\bQ}$. Then for any integer $k \geq 0$
    \[
    \KL(\mu \tilde \bQ^k \tilde \bP_{T,h} \dvert \nu) \leq 2c_1^2 e^{-2c_2 k} \Big( \frac{9}{4T^2} + \frac{363}{2}dM^2T^4 \Big) W_2^2(\mu, \tilde \nu_{\tilde \bQ}) + \mathsf{bias}
    \]
where 
\begin{align*}
        \mathsf{bias} &\leq 2 \Delta_h^2\Big( \frac{9}{4T^2} + \frac{363}{2}dM^2T^4 \Big) + 4h^4 \Bigg[ 45 dL^2 + \Big(4 \frac{L^2}{T^2} + 45dM^2 \Big)m_2(\mu \tilde \bQ^{k}) +(4L^2 + 45dM^2T^2)d\\
        & + \Big(4\frac{M^2}{T^2} + 45d(M^2T^2 +N)^2\Big)m_4(\mu \tilde \bQ^{k}) + \big(4M^2T^2 + 45d(M^2T^4 + NT^2)^2\big)d(d+2) \Bigg]\,.
\end{align*}
\end{theorem}

We verify Assumptions~\ref{assumption:W2Mixing} and~\ref{assumption:W2Bias} for \uhmcv in Section~\ref{subsubsec:Examples-Verlet}, and for uHMC implemented with a randomized integration scheme \cite{bou2025unadjusted} in Section~\ref{subsubsection:Examples-stratified}. As a consequence, we obtain explicit bounds on the total number of gradient evaluations required by each chain to produce a sample $X$ whose law satisfies $\KL(\law(X) \dvert \nu)\le \epsilon$; see Corollaries~\ref{cor:complexity-KL-verlet} and~\ref{cor:complexity-KL-stratified}. We find that the Verlet integration scheme exhibits more favorable dependence on the target accuracy $\epsilon$, whereas the stratified scheme~\cite{bou2025unadjusted} achieves better dependence on the dimension $d$ in the regime when~$\epsilon < \frac{1}{d}$.

\subsection{Guarantees in Rényi Divergence}\label{subsec:Main-Renyi}

Our results in Rényi divergence are obtained by upgrading Orlicz-Wasserstein guarantees to Rényi divergence. The mixing time of \uHMCv in Rényi divergence (Theorem~\ref{thm:RenyiMixing}) requires $\tilde \bP_{T,h}$ to satisfy:

\begin{assumption}\label{assumption:OWMixing}
    Let $\tilde \bQ$ be a discrete-time Markov transition kernel with stationary distribution $\tilde \nu_{\tilde \bQ}$ for which there exist constants $c_1 \geq 1$, $c_2 > 0$ such that for any integer $k \geq 1$ and initial probability distribution $\mu$
    \[
    W_{\psi}(\mu \tilde \bQ^k, \tilde \nu_{\tilde \bQ}) \leq c_1 e^{-c_2k} \,W_{\psi}(\mu, \tilde \nu_{\tilde \bQ})\,.
    \]
\end{assumption}

The following theorem describes the mixing time of \uhmcv in Rényi divergence. We outline the proof of Theorem~\ref{thm:RenyiMixing} in Section~\ref{subsec:RenyiMixingSketch} and prove Theorem~\ref{thm:RenyiMixing} in Section~\ref{app:PfOf_ThmRenyiMixing}.

\begin{theorem}\label{thm:RenyiMixing}
    Let $\mu \in \P(\R^d)$ and suppose Assumption~\ref{assumption:potential}(a)-(c) holds. Further suppose $h \geq 0$, $L(T^2 + Th) \leq \frac{1}{12}$, and that Assumption~\ref{assumption:OWMixing} holds for
    $\tilde\bQ = \tilde\bP_{T,h}$, with invariant distribution $\tilde\nu_h$. For any $1<q< \infty$, define
    \[
    \delta_1 \coloneqq (q-1)\left( 8d MT^2 + \frac{3\sqrt{d}}{2T} \right),\, \delta_2 \coloneqq \frac{9q(q-1)}{8T^2}\,.
    \]
    Then for any integer 
    \[
    k \geq \frac{1}{c_2} \log \left( \frac{c_1 W_{\psi}(\mu, \tilde \nu_h)}{4 \sqrt{\log 2}} \left(\delta_1 + \sqrt{\delta_1^2 + 16 (\log 2) \delta_2}\right) \right)
    \]
    it holds that
    \[
    \sfR_q(\mu \tilde \bP_{T,h}^{k+1} \dvert \tilde \nu_h) \leq  \frac{\delta_1 + \delta_2}{q-1} \sqrt{\log 2}\, c_1^2 e^{-c_2 k}\, \max\{1, W_{\psi}^2 (\mu, \tilde \nu_h) \}\,.
    \]
\end{theorem}

Recall that $\tilde \nu_h$ is the stationary distribution of \uhmcv and Theorem~\ref{thm:RenyiMixing} shows that the law of \uhmcv converges to $\tilde \nu_h$ exponentially fast once the number of iterations exceeds a certain threshold. This threshold is mild, exhibiting only logarithmic dependence on the parameters.
We verify that Assumption~\ref{assumption:OWMixing} holds for \uhmcv for strongly log-concave targets in Proposition~\ref{prop:OWmixing-Verlet}. In this setting, Corollary~\ref{cor:mixing-Renyi-verlet} describes the number of steps of \uhmcv required to output a sample within the desired accuracy to the stationary distribution in Rényi divergence.

Asymptotic bias guarantees in Rényi divergence hold for general kernels $\tilde \bQ$ as long as the last step of the Markov chain is the \uHMCv kernel $\tilde \bP_{T,h}$; these guarantees require the kernel $\tilde \bQ$ to satisfy Assumptions~\ref{assumption:OWMixing} and~\ref{assumption:OWBias}.

\begin{assumption}\label{assumption:OWBias}
    Let $\tilde \bQ$ be a discrete-time Markov transition kernel with stationary distribution $\tilde \nu_{\tilde \bQ}$ for which $W_{\psi}(\tilde \nu_{\tilde \bQ}, \nu) \leq \Delta_h$, where $\Delta_h \geq 0$ and $\Delta_h \to 0$ as $h \to 0$.
\end{assumption}

Our main result describing the asymptotic bias of \uhmc in Rényi divergence is the following theorem. We sketch the proof of Theorem~\ref{thm:RenyiTarget} in Section~\ref{subsec:RenyiTargetSketch} and prove Theorem~\ref{thm:RenyiTarget} in Section~\ref{app:PfOf_ThmRenyiTarget}.

\begin{theorem}\label{thm:RenyiTarget}
    Let $\mu \in \P(\R^d)$, suppose $X \sim \nu$ satisfies $\|X\|_{\psi} \leq K_\nu$ and that Assumption~\ref{assumption:potential}(a)-(c) holds. Also recall the constants defined in Definition~\ref{def:Constants}. Further suppose Assumptions~\ref{assumption:OWMixing} and~\ref{assumption:OWBias} hold for a Markov kernel $\tilde \bQ$ with stationary distribution $\tilde \nu_{\tilde \bQ}$. 
    Let $1 < q < \infty$ and define
    \[
    s \coloneqq 2q-1,~~~u^2 \coloneqq 1 + \frac{1}{4s},~~~ \delta_1 \coloneqq s\left( 8d MT^2 + \frac{3\sqrt{d}}{2T} \right),~~~ \delta_2 \coloneqq \frac{9qs}{4T^2}\,.
    \]
    Suppose
\[
\Delta_h \leq \min \left\{ \frac{\sqrt{2 \log 2}}{4s(\frac{\sqrt{d}\,(6\sqrt{d}\,\sfj_{xy} + \sfp_{xy}u)}{\sqrt{2}}+h\sqrt{d}\, (6\sqrt{d}\, \sfj_x + \sfp_xu))} , \frac{1}{\sqrt{8s(\sfp_{xy}^2(1+3su^2) + 2 h^2\sfp_x^2(1+3su^2))}}  \right\}\,,
\]
\[
h \leq \min \left\{ \frac{u-1}{\sfp_v} ,  \frac{\sqrt{\log 2}}{2 \sqrt{2}s\sqrt{d}\, (6\sqrt{d}\, \sfj_x + \sfp_xu) K_\nu} , \frac{1}{4K_\nu\sfp_x \sqrt{(1+3su^2)s}} \right\}\,,
\]
    and that $L(T^2+Th) \leq \frac{1}{12}$.
    Then for any integer 
    \[
    k \geq \frac{1}{c_2} \log \left( \frac{c_1 W_{\psi}(\mu, \tilde \nu_{\tilde \bQ})}{4 \sqrt{\log 2}} \left(\delta_1 + \sqrt{\delta_1^2 + 16 (\log 2) \delta_2}\right) \right)
    \]
    it holds that
    \begin{align*}
    \sfR_q (\mu \tilde \bQ^k \tilde \bP_{T,h} \dvert \nu) &\leq \frac{3(\delta_1 + \delta_2)}{2(2q-1)} \sqrt{\log 2}\, c_1^2 e^{-c_2 k}\, \max\{1, W_{\psi}^2 (\mu, \tilde \nu_{\tilde \bQ}) \} + \mathsf{bias}
    \end{align*}
    where 
    \begin{align*}
        \mathsf{bias} &= dh(6 \sfj_c + 6\sqrt{d}\, \sfj_v + 108sdh \sfj_v^2 + h \sfp_v^2 +2 \sfp_v) + 2\sqrt{d}\,(6\sqrt{d}\,\sfj_{xy} + \sfp_{xy}u)\Delta_h\\
        &+ \sfp_{xy}^2(1+3su^2)\Delta_h^2 + 2h\sqrt{d}\, (6\sqrt{d}\, \sfj_x + \sfp_xu)(\Delta_h + K_\nu + \sqrt{\Delta_h^2 + K_\nu^2}\,)\\
        &+ 2h^2\sfp_x^2(1+3su^2)(\Delta_h^2 + K_\nu^2)\,.
    \end{align*}
\end{theorem}

The upper bound on $\Delta_h$ in Theorem~\ref{thm:RenyiTarget} amounts to an upper bound on the step-size $h$ having fixed an integration scheme (and having verified Assumption~\ref{assumption:OWBias} for it).
We verify Assumption~\ref{assumption:OWBias} for strongly log-concave targets for \uhmcv in Proposition~\ref{prop:OWbias-Verlet}.
Upon verifying Assumptions~\ref{assumption:OWMixing} and~\ref{assumption:OWBias}, Corollary~\ref{cor:complexity-Renyi-verlet} states the first-order oracle complexity for \uhmcv to output a sample within $\epsilon$--accuracy to the target distribution in Rényi divergence. 

\begin{remark}[Rényi vs.\ KL complexity]
The bounds on the total number of gradient evaluations required to achieve a
given accuracy in Rényi divergence
(Corollary~\ref{cor:complexity-Renyi-verlet}) are less favorable than the
corresponding bounds in KL divergence
(Corollaries~\ref{cor:complexity-KL-verlet}
and~\ref{cor:complexity-KL-stratified}).  This gap is intrinsic to our analysis:
in the Rényi divergence setting, the one-shot coupling construction can exploit
only the first-order accuracy of the velocity Verlet integrator, whereas the KL
analysis is able to leverage its second-order accuracy; see
Section~\ref{sec:TargetMap}.
\end{remark}

\subsection{Examples}\label{subsec:Examples}

We consider target distributions $\nu \propto e^{-f}$ which are $\alpha$-strongly log-concave ($\alpha$-SLC), which means that the potential function $f \in \cC^2(\R^d)$ is $\alpha$-strongly convex for some $\alpha >0$, i.e., $\nabla^2 f (x) \succeq \alpha I$ for all $x \in \R^d$.
We now discuss our theorems for this setting with two different discretization schemes, the velocity Verlet scheme (\uhmcv), and the stratified integration scheme~\cite{bou2025unadjusted} which is a randomized integrator.

\subsubsection{Strongly log-concave target with velocity Verlet integration scheme}\label{subsubsec:Examples-Verlet}

Under the condition that the target distribution is $\alpha$-SLC, we verify Assumptions~\ref{assumption:W2Mixing},~\ref{assumption:W2Bias},~\ref{assumption:OWMixing},~\ref{assumption:OWBias} for the \uhmcv kernel $\tilde \bP_{T,h}$ with stationary distribution $\tilde \nu_h$. 

\paragraph{Assumptions~\ref{assumption:W2Mixing} and~\ref{assumption:W2Bias} for Guarantees in KL Divergence.}

Mixing time guarantees for \uhmcv in Wasserstein--$p$ distance are established in~\cite[Theorem~18]{bou2023convergence}, which verifies Assumption~\ref{assumption:W2Mixing}. We repeat the result below.

\begin{proposition}[Theorem~$18$ from \cite{bou2023convergence} (adapted)]\label{prop:W2mixing-Verlet}
    Let $\mu, \pi \in \P(\R^d)$, $\nu$ be $\alpha$-strongly log-concave, and suppose Assumption~\ref{assumption:potential}(a)-(b) holds. Further suppose $h \geq 0$ and that $LT^2 \leq \frac{1}{20}$\,. Then for any $p \in [1, \infty)$ and integer $k \geq 0$
    \[
    W_p (\mu \tilde \bP_{T,h}^k, \pi \tilde \bP_{T,h}^k) \leq \Big( 1-\frac{\alpha T^2}{10} \Big)^k W_p(\mu, \pi)\,.
    \]
\end{proposition}

Hence, as $\tilde \nu_h \tilde \bP_{T,h} = \tilde \nu_h$, Assumption~\ref{assumption:W2Mixing} holds with $c_1 = 1$ and $c_2 = \frac{\alpha T^2}{10}$.

To verify Assumption~\ref{assumption:W2Bias}, we begin with the triangle inequality for Wasserstein--$2$ distance which implies
\begin{align*}
    W_2(\tilde \nu_h, \nu) &\leq W_2(\nu, \nu \tilde \bP_{T,h}) + W_2(\nu\tilde \bP_{T,h}, \tilde \nu_h)\\
    &\leq W_2(\nu, \nu \tilde \bP_{T,h}) + \Big(1-\frac{\alpha T^2}{10} \Big) W_2(\nu, \tilde \nu_h)
\end{align*}
where the second inequality is by Proposition~\ref{prop:W2mixing-Verlet}. Hence we get that
\begin{equation}\label{eq:W2bias-triangleineq}
\frac{\alpha T^2}{10} W_2(\tilde \nu_h, \nu) \leq W_2(\nu, \nu \tilde \bP_{T,h})\,.
\end{equation}
To bound this,~\cite[Lemma~23]{bou2023mixing} (with simplified constants) states that under Assumption~\ref{assumption:potential}(a)-(c) and $L(T^2 + Th) \leq \frac{1}{6}$, for any $x,v \in \R^d$
\begin{equation}\label{eq:L23_statement}
\max_{s \leq T} \|q_s(x,v) - \tilde q_{s,h}(x,v)\| \leq h^2 \Big( \frac{1}{5}L\|x\| + \frac{9}{10}LT\|v\| + \frac{1}{120}M\|x\|^2 + \frac{3}{10}MT^2\|v\|^2 \Big)\,.
\end{equation}
So, suppose $X_0 \sim \nu$ and $\xi \sim \gammad$, and define $X_1 \coloneqq q_T(X_0, \xi), \tilde X_1 \coloneqq \tilde q_{T,h}(X_0, \xi)$. Then~\eqref{eq:L23_statement} implies
\[
\|X_1 - \tilde X_1\| \leq h^2 \Big( \frac{1}{5}L\|X_0\| + \frac{9}{10}LT\|\xi\| + \frac{1}{120}M\|X_0\|^2 + \frac{3}{10}MT^2\|\xi\|^2 \Big)\,.
\]
Squaring both sides and applying $(a_1 + a_2 + a_3 + a_4)^2 \leq 4(a_1^2 + a_2^2 + a_3^2 + a_4^2)$ for $a_1, a_2, a_3, a_4 \in \R$, we obtain
\[
\|X_1 - \tilde X_1\|^2 \leq 4h^4 \Big( \frac{1}{25}L^2\|X_0\|^2 + \frac{81}{100}L^2T^2\|\xi\|^2 + \frac{1}{120^2}M^2\|X_0\|^4 + \frac{9}{100}M^2T^4\|\xi\|^4 \Big)\,.
\]
Taking expectation over $X_0 \sim \nu$ and $\xi \sim \gammad$, and from Definition~\ref{def:WassersteinDistance}, we obtain
\[
W_2^2(\nu, \nu \tilde \bP_{T,h}) \leq 4h^4 \Big( \frac{1}{25}L^2\E_\nu \big[\|X_0\|^2\big] + \frac{81}{100}L^2T^2\E_{\gammad}\big[\|\xi\|^2\big] + \frac{1}{120^2}M^2\E_{\nu}\big[\|X_0\|^4\big] + \frac{9}{100}M^2T^4\E_{\gammad}\big[\|\xi\|^4\big] \Big)\,.
\]
Using $\E_{\gammad}\big[\|\xi\|^2\big] = d$ and $\E_{\gammad}\big[\|\xi\|^4\big] = d(d+2)$ and combining with~\eqref{eq:W2bias-triangleineq} gives us the following proposition.

\begin{proposition}\label{prop:W2bias-Verlet}
    Let $\nu$ be $\alpha$-strongly log-concave, and suppose Assumption~\ref{assumption:potential}(a)-(c) holds. Further suppose $h \geq 0$ and that $L(T^2+Th) \leq \frac{1}{20}$\,. Then
    \[
    W_2(\tilde \nu_h, \nu) \leq h^2\frac{20}{\alpha T^2} \Big( \frac{1}{25}L^2\E_\nu \big[\|X_0\|^2\big] + \frac{81}{100}dL^2T^2 + \frac{1}{120^2}M^2\E_{\nu}\big[\|X_0\|^4\big] + \frac{9}{100}d(d+2)M^2T^4\big] \Big)^{\frac{1}{2}}\,.
    \]
\end{proposition}

This verifies Assumption~\ref{assumption:W2Bias} for \uhmcv. Hence, as Propositions~\ref{prop:W2mixing-Verlet} and~\ref{prop:W2bias-Verlet} verify Assumptions~\ref{assumption:W2Mixing} and~\ref{assumption:W2Bias} respectively for \uhmcv, we obtain the following corollaries from Theorems~\ref{thm:KLMixing} and~\ref{thm:KLTarget}.

\begin{corollary}[Mixing time of \uhmcv in KL divergence]\label{cor:mixing-KL-verlet}
    Under conditions such that Theorem~\ref{thm:KLMixing} and Proposition~\ref{prop:W2mixing-Verlet} hold, \uhmcv initialized from $\mu \in \P(\R^d)$ outputs a sample $X_k$ such that $\KL(\law(X_k) \dvert \tilde \nu_h) \leq \epsilon$ for
    \[
    k \geq 1 + \frac{5}{\alpha T^2} \log{\left( \frac{\left( \frac{9}{4T^2} + 20dM^2T^4 \right) W_2^2(\mu, \tilde \nu_h)}{\epsilon} \right)}\,.
    \]
\end{corollary}

We also state the first-order gradient complexity of \uhmcv to obtain a $\epsilon$--accurate approximation to the target distribution $\nu$ in KL divergence. We specify the dependence on the accuracy $\epsilon$ and the dimension $d$ assuming that $W_2(\mu, \tilde \nu_h)$ is polynomial in $d$ where $\mu \in \P(\R^d)$ is the initial distribution.

\begin{corollary}[Oracle complexity for guarantee in KL divergence]\label{cor:complexity-KL-verlet}
    Under conditions such that Theorem~\ref{thm:KLTarget} and Propositions~\ref{prop:W2mixing-Verlet} and~\ref{prop:W2bias-Verlet} hold, \uhmcv outputs a sample $X$ such that $\KL(\law(X) \dvert \nu) \leq \epsilon$ with $O \left( d^{\nicefrac{3}{4}} \epsilon^{\nicefrac{-1}{4}} \log (\nicefrac{d}{\epsilon}) \right)$ first-order oracle calls.
\end{corollary}

\paragraph{Assumptions~\ref{assumption:OWMixing} and~\ref{assumption:OWBias} for Guarantees in Rényi Divergence.}

To verify Assumption~\ref{assumption:OWMixing}, we note that~\cite[Lemma~19]{bou2023convergence}, the key lemma in proving~\cite[Theorem~18]{bou2023convergence} (i.e., Proposition~\ref{prop:W2mixing-Verlet}), can be extended to establish mixing time guarantees for \uhmcv in Orlicz-Wasserstein distance. To this end,~\cite[Lemma~19]{bou2023convergence} claims that for any $x,y,v \in \R^d$
\[
\| \tilde q_{T,h}(x,v)-\tilde q_{T,h}(y,v)\| \leq \Big( 1-\frac{\alpha T^2}{10} \Big)\|x-y\|\,.
\]
Now suppose $X_0 \sim \mu$, $Y_0 \sim \pi$, and $\xi \sim \gammad$ and define $X_1 \coloneqq \tilde q_{T,h}(X_0, \xi)$, $Y_1 \coloneqq \tilde q_{T,h}(Y_0, \xi)$. From~\cite[Lemma~19]{bou2023convergence} we know that 
\[
\|X_1 - Y_1\| \leq \Big( 1-\frac{\alpha T^2}{10} \Big)\|X_0-Y_0\|\,.
\]
Due to Definition~\ref{def:OrliczNorm}, note that for two random variables $Z, Z'$, if $\|Z\| \leq \rho\|Z'\|$ for some $\rho > 0$, then $\|Z\|_{\psi} \leq \rho\|Z'\|_{\psi}$. Hence we obtain
\[
\|X_1 - Y_1\|_{\psi} \leq \Big( 1-\frac{\alpha T^2}{10} \Big)\|X_0-Y_0\|_{\psi}\,.
\]
Taking infimum over initial couplings of $X_0, Y_0$ and by Definition~\ref{def:OrliczWassersteinDistance}, we get that
\[
W_{\psi}(\mu \tilde \bP_{T,h}, \pi \tilde \bP_{T,h}) \leq \Big( 1-\frac{\alpha T^2}{10} \Big) W_{\psi}(\mu, \pi)\,.
\]
Iterating this over multiple steps leads to the following proposition.

\begin{proposition}\label{prop:OWmixing-Verlet}
    Let $\mu, \pi \in \P(\R^d)$, $\nu$ be $\alpha$-strongly log-concave, and suppose Assumption~\ref{assumption:potential}(a)-(b) holds. Further suppose $h \geq 0$ and that $LT^2 \leq \frac{1}{20}$\,. Then for any integer $k \geq 0$
    \[
    W_{\psi} (\mu \tilde \bP_{T,h}^k, \pi \tilde \bP_{T,h}^k) \leq \Big( 1-\frac{\alpha T^2}{10} \Big)^k W_{\psi}(\mu, \pi)\,.
    \]
\end{proposition}

It remains to verify Assumption~\ref{assumption:OWBias}. 
This follows similarly to the sketch of Proposition~\ref{prop:W2bias-Verlet}.
We begin by using triangle inequality along with Proposition~\ref{prop:OWmixing-Verlet} to obtain
\begin{equation}\label{eq:OWbias-triangleineq}
\frac{\alpha T^2}{10} W_{\psi}(\tilde \nu_h, \nu) \leq W_{\psi}(\nu, \nu \tilde \bP_{T,h})\,.
\end{equation}
Let $X_0 \sim \nu$, $\xi \sim \gammad$, and define $X_1 \coloneqq q_T(X_0, \xi), \tilde X_1 \coloneqq \tilde q_{T,h}(X_0, \xi)$. Then Lemma~\ref{lem:L23analogue} implies
\[
\|X_1 - \tilde X_1\| \leq h \left( \frac{6}{25}\|\xi\| + \frac{7}{30}LT\|X_0\| \right)\,.
\]
Dividing both sides by $\lambda > 0$ (determined later) and squaring we obtain
\[
\frac{\|X_1 - \tilde X_1\|^2}{\lambda^2} \leq \frac{2h^2}{\lambda^2} \left( \frac{36}{625}\|\xi\|^2 + \frac{49}{900}L^2T^2\|X_0\|^2 \right)
\]
and hence,
\[
\exp \left(\frac{\|X_1 - \tilde X_1\|^2}{\lambda^2} \right) \leq \exp \left(\frac{2h^2}{\lambda^2} \left( \frac{36}{625}\|\xi\|^2 + \frac{49}{900}L^2T^2\|X_0\|^2 \right) \right)\,.
\]
Taking expectation on both sides over $X_0 \sim \nu$ and $\xi \sim \gammad$ yields
\[
\E \left[ \exp \left(\frac{\|X_1 - \tilde X_1\|^2}{\lambda^2} \right) \right] \leq \E \left[ \exp \left(\frac{h^2}{\lambda^2} \frac{72}{625}\|\xi\|^2 \right) \right] \E \left[ \exp \left(\frac{h^2}{\lambda^2}  \frac{98}{900}L^2T^2\|X_0\|^2  \right) \right]
\]
Using properties of Gaussian distributions and Lemma~\ref{lem:Helper_Orlicz}, this can be further upper bounded by
\[
\E \left[ \exp \left(\frac{\|X_1 - \tilde X_1\|^2}{\lambda^2} \right) \right] \leq \left( 1-2 \frac{h^2}{\lambda^2} \frac{72}{625} \right)^{-\frac{d}{2}} 2^{\frac{h^2}{\lambda^2}  \frac{98}{900}L^2T^2}\,.
\]
so long as $\frac{h^2}{\lambda^2} \frac{72}{625} < \frac{1}{2}$ and $\frac{h^2}{\lambda^2}  \frac{98}{900}L^2T^2 \leq \frac{1}{K_\nu^2}$\,.
To ensure this, and to make each term on the right hand side less than $\sqrt{2}$, we choose $\lambda$ satisfying the following
\[
\lambda \geq h\, \max \left\{ \sqrt{d}, \frac{7}{15}LT K_\nu  \right\}\,.
\]
Hence, under this condition, the right hand side is $\leq 2$. Therefore, from Definition~\ref{def:OrliczWassersteinDistance} we get that
\[
W_{\psi}(\nu, \nu \tilde \bP_{T,h}) \leq h\, \max \left\{  \sqrt{d}, \frac{7}{15}LT K_\nu  \right\}\,,
\]
which upon combining with~\eqref{eq:OWbias-triangleineq} yields the following proposition.

\begin{proposition}\label{prop:OWbias-Verlet}
    Let $\nu$ be $\alpha$-strongly log-concave, suppose $X \sim \nu$ satisfies $\|X\|_{\psi} \leq K_{\nu}$, and suppose Assumption~\ref{assumption:potential}(a)-(b) holds. Further suppose $h \geq 0$ and that $L(T^2+Th) \leq \frac{1}{20}$\,. Then
    \[
    W_{\psi}(\tilde \nu_h, \nu) \leq h\frac{10}{\alpha T^2} \max \left\{  \sqrt{d}, \frac{7}{15}LT K_\nu  \right\} \,.
    \]
\end{proposition}

Propositions~\ref{prop:OWmixing-Verlet} and~\ref{prop:OWbias-Verlet} verify Assumptions~\ref{assumption:OWMixing} and~\ref{assumption:OWBias} for \uhmcv, leading to the following corollaries.

\begin{corollary}[Mixing time of \uhmcv in Rényi divergence]\label{cor:mixing-Renyi-verlet}
    Under conditions such that Theorem~\ref{thm:RenyiMixing} and Proposition~\ref{prop:OWmixing-Verlet} hold, \uhmcv initialized from $\mu \in \P(\R^d)$ outputs a sample $X_k$ such that $\sfR_q(\law(X_k) \dvert \tilde \nu_h) \leq \epsilon$ for
    \[
    k \geq 1 + \frac{10}{\alpha T^2} \log{\left( \frac{2 \left( \delta_1 + 2 \sqrt{\log 2} \max\left \{\delta_2, \sqrt{\delta_2}\right\} \right) \max \left\{1,  W_{\psi}^2(\mu, \tilde \nu_h) \right\}}{\epsilon} \right)}\,.
    \]
\end{corollary}

We now state the first-order gradient complexity of \uhmcv to obtain a $\epsilon$--accurate approximation to the target distribution $\nu$ in Rényi divergence. We specify the dependence on the accuracy $\epsilon$ and the dimension $d$ assuming that $W_{\psi}(\mu, \tilde \nu_h)$ is polynomial in $d$ where $\mu \in \P(\R^d)$ is the initial distribution.

\begin{corollary}[Oracle complexity for guarantee in Rényi divergence]\label{cor:complexity-Renyi-verlet}
    Under conditions such that Theorem~\ref{thm:RenyiTarget} and Propositions~\ref{prop:OWmixing-Verlet} and~\ref{prop:OWbias-Verlet} hold, \uhmcv outputs a sample $X$ such that $\sfR_q(\law(X) \dvert \nu) \leq \epsilon$ with $O \left( d^{\nicefrac{3}{2}} \epsilon^{-1} \log (\nicefrac{d}{\epsilon}) \right)$ first-order oracle calls.
\end{corollary}

\subsubsection{Strongly log-concave target with stratified integration scheme}\label{subsubsection:Examples-stratified}

We verify Assumptions~\ref{assumption:W2Mixing},~\ref{assumption:W2Bias},~\ref{assumption:OWMixing} for the \uhmc scheme given in~\cite{bou2025unadjusted}. We refer to the scheme as \uhmcs (for stratified), denote the transition kernel as $\tilde \bQ_{T,h}$, and the stationary distribution for the Markov chain by $\tilde \nu_{\tilde \bQ_{T,h}}$.

Assumption~\ref{assumption:W2Mixing} and~\ref{assumption:OWMixing} for \uhmcs follows from~\cite[Lemma~5]{bou2025unadjusted}. The almost sure contractivity in~\cite[Lemma~5]{bou2025unadjusted} immediately yields mixing time guarantees in any Wasserstein--$p$ distance (for finite $p$) and in Orlicz-Wasserstein distance (similar to the sketch of Proposition~\ref{prop:OWmixing-Verlet}). Theorem~6 in~\cite{bou2025unadjusted} is explicitly written for Wasserstein--$2$ distance, however we state the more general theorem following from~\cite[Lemma~5]{bou2025unadjusted} below. 

\begin{proposition}[Consequence of Lemma~5 from~\cite{bou2025unadjusted}]\label{prop:shmc-mixing}
    Let $\mu, \pi \in \P(\R^d)$, $\nu$ be $\alpha$-strongly log-concave, and suppose Assumption~\ref{assumption:potential}(a)-(b) holds. Further suppose $h \geq 0$ and that $LT^2 \leq \frac{1}{8}$\,. Then for $W \in \{ W_p, W_{\psi}\}$ for any $p \in [1, \infty)$, and any integer $k \geq 0$
    \[
    W (\mu \tilde \bQ_{T,h}^k, \pi \tilde \bQ_{T,h}^k) \leq \Big( 1-\frac{\alpha T^2}{6} \Big)^k W(\mu, \pi)\,.
    \]
\end{proposition}

Assumption~\ref{assumption:W2Bias} follows from~\cite[Theorem~10]{bou2025unadjusted}, stated below.

\begin{proposition}[Theorem~10 from~\cite{bou2025unadjusted}]\label{prop:shmc-bias}
    Let $\nu$ be $\alpha$-strongly log-concave, and suppose Assumption~\ref{assumption:potential}(a)-(b) holds. Further suppose $h \geq 0$ and that $LT^2 \leq \frac{1}{8}$\,. Then
    \[
    W_2(\tilde \nu_{\tilde \bQ_{T,h}}, \nu) \leq h^{\frac{3}{2}} 142 \sqrt{d} \frac{6}{\alpha T^2} \Big( \frac{L}{\alpha} \Big)^{\frac{1}{2}} L^{\frac{1}{4}}\,.
    \]
\end{proposition}

One approach to verify Assumption~\ref{assumption:OWBias} for \uhmcs is to extend~\cite[Lemma~7]{bou2025unadjusted} to $\psi$-Orlicz accuracy; we do not pursue this question in this work.
Propositions~\ref{prop:shmc-mixing} and~\ref{prop:shmc-bias}, together with Theorem~\ref{thm:KLTarget} yield the following first-order complexity guarantee for \uhmcs for accuracy to the target distribution in KL divergence. As expected, when compared to Corollary~\ref{cor:complexity-KL-verlet} for the Verlet integration scheme, we can see the improved dimension dependence with worse dependence on the error $\epsilon$, as described in~\cite{bou2025unadjusted} for Wasserstein--$2$ distance.
We show the dependence on the accuracy $\epsilon$ and the dimension $d$ assuming that $W_2(\mu, \tilde \nu_h)$ is polynomial in $d$ where $\mu \in \P(\R^d)$ is the initial distribution.

\begin{corollary}[Oracle complexity for guarantee in KL divergence for stratified scheme]\label{cor:complexity-KL-stratified}
    Under conditions such that Theorem~\ref{thm:KLTarget} and Propositions~\ref{prop:shmc-mixing} and~\ref{prop:shmc-bias} hold, \uhmcs (with last step as \uhmcv) outputs a sample $X$ such that $\KL(\law(X) \dvert \nu) \leq \epsilon$ with $O \left( \max\{d^{\nicefrac{2}{3}} \epsilon^{\nicefrac{-1}{3}}, d^{\nicefrac{3}{4}} \epsilon^{\nicefrac{-1}{4}} \} \log (\nicefrac{d}{\epsilon}) \right)$ first-order oracle calls.
\end{corollary}

\subsection{Application to Information Contraction}\label{subsec:InformationContraction}

The mixing time and asymptotic bias guarantees we study in this paper describe the divergence between the marginal distribution along the Markov chain and the stationary or target distribution, however, they do not immediately quantify when two iterates along the Markov chain are (approximately) independent, which requires analyzing the joint distribution of two iterates and checking when it is close to a product distribution. The mutual information between two iterates therefore characterizes their dependence, and is a strong notion of independence in that it implies small covariance between the iterates~\cite[Lemma~9]{liang25PhiMI} and independence against test functions.
Quantifying the dependence between samples generated by a Markov chain is helpful when using them for downstream applications, and motivated~\cite{liang25PhiMI} to study the contraction of the mutual information functional along the Langevin dynamics and its discretizations such as the Unadjusted Langevin Algorithm~\cite{roberts1996exponential, VW23} and Proximal Sampler~\cite{lee2021structured, chen2022improved}. Using the general reduction from mixing time to information contraction discussed in~\cite[Appendix~I]{liang25PhiMI}, we use Theorem~\ref{thm:KLMixing} to obtain bounds on the information contraction along \uhmcv.

The mutual information between two random variables is a function of their joint distribution and is defined as follows.

\begin{definition}\label{def:MutualInformation}
    Given two random variables $X, Y$ on $\R^d$ such that $(X,Y) \sim \rho^{XY} \in \P(\R^{2d})$, the mutual information functional $\MI : \P(\R^{2d}) \to \R$ is defined as
    \begin{equation*}\label{eq:MIPhi}
        \MI(X; Y) \equiv \MI (\rho^{XY}) \coloneqq \KL \left(\rho^{XY} \dvert \rho^X \otimes \rho^Y \right)\,.
    \end{equation*}
\end{definition}

Let $X_0 \sim \mu$ where $\mu$ is arbitrary, and let $X_k \sim \mu \tilde \bP_{T,h}^k$, then, we would like to study $\MI (X_0; X_k)$. Note that when $\mu = \tilde \nu_h$, this will tell us when we obtain approximately independent and identically distributed samples along the \uhmcv Markov chain.
We now repeat~\cite[Lemma~13]{liang25PhiMI}, describing the reduction from mixing time to information contraction for a general Markov transition kernel $\tilde \bP$. 

\begin{proposition}[Lemma~13 from~\cite{liang25PhiMI}]\label{prop:InformationContraction} 
    Let $\tilde \bP$ be a Markov chain with stationary distribution $\tilde \nu$. Let $X_0 \sim \mu \in \P(\R^d)$ be the initial iterate and define successive iterates by $X_k \sim \mu \tilde \bP^k$ for integer $k \geq 1$. Given an error threshold $\epsilon \geq 0$, suppose there exists $\tau \in \bbN$ such that $\KL(\delta_x \tilde \bP^\tau \dvert  \tilde \nu) \leq \epsilon$ for all $x \in \R^d$. Then, we have
    \[
    \MI(X_0; X_\tau) \leq \epsilon\,.
    \]
\end{proposition}

As Theorem~\ref{thm:KLMixing}, and Corollary~\ref{cor:mixing-KL-verlet} in particular, satisfies the requirement of Proposition~\ref{prop:InformationContraction} of mixing guarantees from Dirac initializations, we obtain the following corollary describing the rate of information contraction for \uhmcv.

\begin{corollary}\label{cor:InformationContraction-Verlet}
Suppose that the conditions required for Theorem~\ref{thm:KLMixing} and Proposition~\ref{prop:W2mixing-Verlet} hold. Given $X_0 \sim \mu \in \P(\R^d)$, define $X_k \sim \mu \tilde \bP_{T,h}^k$ for integer $k \geq 1$. Then the following holds
    \[
    \MI(X_0; X_k) \leq e^{-\frac{\alpha T^2}{5}(k-1)}\left( \frac{9}{4T^2} + 20dM^2T^4 \right) \E_{(X,Y) \sim \mu \otimes \tilde \nu_h} \left[\|X-Y\|^2\right]\,.
    \]
\end{corollary}

We prove Corollary~\ref{cor:InformationContraction-Verlet} in Section~\ref{app:InformationContraction}.
Note that if a Markov chain is initialized at $X_0 \sim \delta_x$ for some $x \in \R^d$, then the mutual information between $X_0 = x$ and any successive iterate is $0$.

\section{Key Steps for Obtaining Guarantees in KL Divergence}

\subsection{Mixing Time}\label{subsec:KLMixingSketch}

In this subsection, we fix the Markov transition kernel to be \uHMCv, i.e., $\tilde \bP_{T,h}$.
We begin from Lemma~\ref{lem:DiracToPerturbed_Mixing}, which shows that for any two points $x,y \in \R^d$
\[
\KL(\delta_x \tilde \bP_{T,h} \dvert \delta_y \tilde \bP_{T,h}) \leq \KL((\varphi_{x,y})_\# \gammad \dvert \gammad)\,,
\]
where recall that $\varphi_{x,y}$ is a map satisfying~\eqref{eq:MixingMap}.
Next, we bound the KL divergence between a standard and perturbed standard Gaussian in Lemma~\ref{lem:KL_Perturbed_Mixing}. We prove Lemma~\ref{lem:KL_Perturbed_Mixing} in Section~\ref{app:PfOf_KL_Perturbed_Mixing}.

\begin{lemma}\label{lem:KL_Perturbed_Mixing}
    Suppose $\psi : \R^d \to \R^d$ is a differentiable and bijective map. Assume there exist constants $0 \leq \mfm_1 < \infty$ and $0 \leq \mfm_2 <1$ such that for all $v \in \R^d$, $\|\psi(v)-v\| \leq \mfm_1$ and $\|\nabla \psi(v) - I\|_{\op} \leq \mfm_2$.
    Then
    \[
    \KL(\psi_{\#}\gammad \dvert \gammad) \leq \frac{1}{2}\mfm_1^2 + \frac{d\, \mfm_2^2}{2(1-\mfm_2)}\,.
    \]
\end{lemma}

Combining Lemma~\ref{lem:DiracToPerturbed_Mixing} and Lemma~\ref{lem:KL_Perturbed_Mixing} with $\psi = \varphi_{x,y}$~\eqref{eq:MixingMap} gives us a regularization bound from Dirac initializations for \uHMCv in KL divergence, stated in Lemma~\ref{lem:KL_Mixing_OneStepRegularity}. We prove Lemma~\ref{lem:KL_Mixing_OneStepRegularity} in Section~\ref{app:PfOf_KL_Mixing_OneStepRegularity}.

\begin{lemma}\label{lem:KL_Mixing_OneStepRegularity}
    Let $x,y \in \R^d$ and suppose Assumption~\ref{assumption:potential}(a)-(c) holds. Further suppose $h \geq 0$ and that $L(T^2 + Th) \leq \frac{1}{12}$. Then
    \[
    \KL(\delta_x \tilde \bP_{T,h} \dvert \delta_y \tilde \bP_{T,h}) \leq \left( \frac{9}{4T^2} + 20dM^2T^4 \right) \|x-y\|^2.
    \]
\end{lemma}

We can now extend Lemma~\ref{lem:KL_Mixing_OneStepRegularity} from Dirac initializations to general initializations using Lemma~\ref{lem:KLConvexityOneStep} with $\bP = \bQ = \tilde \bP_{T,h}$, which yields a regularization guarantee for \uhmcv in KL divergence. We state this in Lemma~\ref{lem:KL_Mixing_OneStepGeneral} and prove it in Section~\ref{app:PfOf_KL_Mixing_OneStepGeneral}.

\begin{lemma}\label{lem:KL_Mixing_OneStepGeneral}
    Let $\mu, \pi \in \P(\R^d)$ and suppose Assumption~\ref{assumption:potential}(a)-(c) holds. Further suppose $h \geq 0$ and that $L(T^2 + Th) \leq \frac{1}{12}$. Then
    \[
    \KL(\mu \tilde \bP_{T,h} \dvert \pi \tilde \bP_{T,h}) \leq \left( \frac{9}{4T^2} + 20dM^2T^4 \right) W_2^2(\mu, \pi)\,.
    \]
\end{lemma}

Combining Lemma~\ref{lem:KL_Mixing_OneStepGeneral} with Assumption~\ref{assumption:W2Mixing} for $\tilde \bQ = \tilde \bP_{T,h}$ yields Theorem~\ref{thm:KLMixing}. We prove Theorem~\ref{thm:KLMixing} in Section~\ref{app:PfOf_ThmKLMixing}.

\subsection{Asymptotic Bias}\label{subsec:KLTargetSketch}

We begin with Lemma~\ref{lem:DiracToPerturbed_Target}, which shows that for any two points $x,y \in \R^d$
\[
\KL(\delta_x \tilde \bP_{T,h} \dvert \delta_y \bP_T) \leq \KL((\Phi_{x,y})_\# \gammad \dvert \gammad)\,,
\]
where recall that $\Phi_{x,y}$ is a map satisfying~\eqref{eq:TargetMap}.
We bound $\KL((\Phi_{x,y})_\# \gammad \dvert \gammad)$ in the following lemma. We prove Lemma~\ref{lem:KL_Target_OneStepRegularity} in Section~\ref{app:PfOf_KL_Target_OneStepRegularity}.

\begin{lemma}\label{lem:KL_Target_OneStepRegularity}
    Suppose Assumption~\ref{assumption:potential}(a)-(d) holds. Further suppose $h \geq 0$ and that $L(T^2+Th) \leq \frac{1}{12}$. Then for any $x,y \in \R^d$
    \begin{align*}
        \KL((\Phi_{x,y})_\# \gammad \dvert \gammad) \leq \Big( \frac{9}{4T^2} + &\frac{363}{2}dM^2T^4 \Big)\|x-y\|^2 +\\
        4h^4 \Bigg[ 45 dL^2& + \Big(4 \frac{L^2}{T^2} + 45dM^2 \Big)\|x\|^2 + \Big(4\frac{M^2}{T^2} + 45d(M^2T^2 +N)^2\Big)\|x\|^4 \\
         +(4&L^2 + 45dM^2T^2)d + \big(4M^2T^2 + 45d(M^2T^4 + NT^2)^2\big)d(d+2) \Bigg]\,.
    \end{align*}
\end{lemma}

Now we can obtain a cross-regularization bound for \uHMCv in KL divergence. We prove Lemma~\ref{lem:KL_Target_OneStepGeneral} in Section~\ref{app:PfOf_KL_Target_OneStepGeneral}.

\begin{lemma}\label{lem:KL_Target_OneStepGeneral}
    Let $\mu$ be a distribution such that $\E_\mu[\|x\|^2] = m_2(\mu) < \infty$ and $\E_\mu[\|x\|^4] = m_4(\mu) < \infty$ and $\pi$ be a distribution such that $W_2(\mu, \pi) < \infty$. Suppose Assumption~\ref{assumption:potential}(a)-(d) holds. Further suppose $h \geq 0$ and that $L(T^2+Th) \leq \frac{1}{12}$. Then
    \begin{align*}
        \KL(\mu \tilde \bP_{T,h} \dvert \pi \bP_T) \leq \Big( \frac{9}{4T^2} + &\frac{363}{2}dM^2T^4 \Big)W_2^2(\mu, \pi) +\\
        4h^4 \Bigg[ 45 dL^2& + \Big(4 \frac{L^2}{T^2} + 45dM^2 \Big)m_2(\mu) + \Big(4\frac{M^2}{T^2} + 45d(M^2T^2 +N)^2\Big)m_4(\mu) \\
         +(4&L^2 + 45dM^2T^2)d + \big(4M^2T^2 + 45d(M^2T^4 + NT^2)^2\big)d(d+2) \Bigg]\,.
    \end{align*}
\end{lemma}

Lemma~\ref{lem:KL_Target_OneStepGeneral} together with Assumptions~\ref{assumption:W2Mixing} and~\ref{assumption:W2Bias} yields Theorem~\ref{thm:KLTarget}, which we prove in Section~\ref{app:PfOf_ThmKLTarget}.

\section{Key Steps for Obtaining Guarantees in Rényi Divergence}

\subsection{Mixing Time}\label{subsec:RenyiMixingSketch}

In this subsection we fix the transition kernel to be \uhmcv, i.e., $\tilde \bP_{T,h}$. Lemma~\ref{lem:DiracToPerturbed_Mixing} implies that for any two points $x,y \in \R^d$
\[
\sfR_q(\delta_x \tilde \bP_{T,h} \dvert \delta_y \tilde \bP_{T,h}) \leq \sfR_q((\varphi_{x,y})_\# \gammad \dvert \gammad)\,,
\]
where recall that $\varphi_{x,y}$ is a map satisfying~\eqref{eq:MixingMap}.
Next, we bound the Rényi divergence between a perturbed standard Gaussian and a standard Gaussian in Lemma~\ref{lem:Renyi_Perturbed_Mixing}. We prove Lemma~\ref{lem:Renyi_Perturbed_Mixing} in Section~\ref{app:PfOf_Renyi_Perturbed_Mixing}.

\begin{lemma}\label{lem:Renyi_Perturbed_Mixing}
    Suppose $\psi : \R^d \to \R^d$ is a differentiable and bijective map. Assume there exist constants $0 \leq \mfm_1 < \infty$ and $0 \leq \mfm_2 <1$ such that for all $v \in \R^d$, $\|\psi(v)-v\| \leq \mfm_1$ and $\|\nabla \psi(v) - I\|_{\op} \leq \mfm_2$.
    Then for any $1<q< \infty$
    \begin{align*}
    \sfR_q(\psi_\#\gammad \dvert \gammad) &\leq \frac{d \mfm_2}{1-\mfm_2} + \sqrt{d}\,\mfm_1  + \frac{q\mfm_1^2}{2}\,.
    \end{align*}
\end{lemma}

Combining Lemmas~\ref{lem:DiracToPerturbed_Mixing} and~\ref{lem:Renyi_Perturbed_Mixing} with $\psi = \varphi_{x,y}$~\eqref{eq:MixingMap} gives a regularization bound from Dirac initializations for \uHMCv in Rényi divergence, stated in Lemma~\ref{lem:Renyi_Mixing_OneStepRegularity}. We prove Lemma~\ref{lem:Renyi_Mixing_OneStepRegularity} in Section~\ref{app:PfOf_Renyi_Mixing_OneStepRegularity}.

\begin{lemma}\label{lem:Renyi_Mixing_OneStepRegularity}
    Let $x,y \in \R^d$ and suppose Assumption~\ref{assumption:potential}(a)-(c) holds. Further suppose $h \geq 0$ and that $L(T^2 + Th) \leq \frac{1}{12}$. Then for any $1<q< \infty$
    \[
    \sfR_q(\delta_x \tilde \bP_{T,h} \dvert \delta_y \tilde \bP_{T,h}) \leq \left( 8d MT^2 + \frac{3\sqrt{d}}{2T} \right)\|x-y\| + \frac{9q}{8T^2}\|x-y\|^2\,.
    \]
\end{lemma}

We can now extend Lemma~\ref{lem:Renyi_Mixing_OneStepRegularity} from Dirac initializations to general initializations using Lemma~\ref{lem:RenyiConvexityOneStep} with $\bP = \bQ = \tilde \bP_{T,h}$. We state this in Lemma~\ref{lem:Renyi_Mixing_OneStepGeneral} and prove it in Section~\ref{app:PfOf_Renyi_Mixing_OneStepGeneral}.

\begin{lemma}\label{lem:Renyi_Mixing_OneStepGeneral}
    Let $\mu, \pi \in \P(\R^d)$ and suppose Assumption~\ref{assumption:potential}(a)-(c) holds. Further suppose $h \geq 0$ and that $L(T^2 + Th) \leq \frac{1}{12}$. Then for any $1<q< \infty$
    \[
    \sfR_q(\mu \tilde \bP_{T,h} \dvert \pi \tilde \bP_{T,h}) \leq \inf_{\gamma \in \Gamma(\mu, \pi)} \frac{1}{q-1} \log \E_{(x,y) \sim \gamma} \left[ e^{(q-1)\left[ \left( 8d MT^2 + \frac{3\sqrt{d}}{2T} \right)\|x-y\| + \frac{9q}{8T^2}\|x-y\|^2   \right]} \right]\,.
    \]
\end{lemma}

Assumption~\ref{assumption:OWMixing} for $\tilde \bQ = \tilde \bP_{T,h}$ together with Lemma~\ref{lem:Renyi_Mixing_OneStepGeneral} yields Theorem~\ref{thm:RenyiMixing}, which we prove in Section~\ref{app:PfOf_ThmRenyiMixing}.

\subsection{Asymptotic Bias}\label{subsec:RenyiTargetSketch}

We begin with Lemma~\ref{lem:DiracToPerturbed_Target}, which shows that for any two points $x,y \in \R^d$
\[
\sfR_q(\delta_x \tilde \bP_{T,h} \dvert \delta_y \bP_T) \leq \sfR_q((\Phi_{x,y})_\# \gammad \dvert \gammad)\,,
\]
where recall that $\Phi_{x,y}$ is a map satisfying~\eqref{eq:TargetMap}.
We bound $\sfR_q((\Phi_{x,y})_\# \gammad \dvert \gammad)$ in the following lemma. We prove Lemma~\ref{lem:Renyi_Target_OneStepRegularity} in Section~\ref{app:PfOf_Renyi_Target_OneStepRegularity}.

\begin{lemma}\label{lem:Renyi_Target_OneStepRegularity}
    Suppose Assumption~\ref{assumption:potential}(a)-(c) holds and recall the constants defined in Definition~\ref{def:Constants}. Further suppose
    \[
    0 \leq h \leq \frac{25T}{7} \bigg[\Big(1+\frac{1}{4(q-1)}\Big)^{\frac{1}{2}}- 1 \bigg]
    \]
    and that $L(T^2+Th) \leq \frac{1}{12}$. Then for any $x,y \in \R^d$ and $1<q< \infty$
\begin{align*}
        \sfR_q ((\Phi_{x,y})_\# \gammad \dvert \gammad) &\leq \sfp_{xy}^2\|x-y\|^2(1+3(q-1)(h\sfp_v+1)^2) + \sqrt{d}\,\|x-y\|(6\sqrt{d}\,\sfj_{xy} + \sfp_{xy}(h \sfp_v+1))\\
        &+ h^2\sfp_x^2\|x\|^2(1+3(q-1)(h\sfp_v+1)^2) + h\sqrt{d}\, \|x\|(6\sqrt{d}\, \sfj_x + \sfp_x(h \sfp_v +1))\\
        &+dh(6 \sfj_c + 6\sqrt{d}\, \sfj_v + 108(q-1)dh \sfj_v^2 + h \sfp_v^2 +2 \sfp_v)\,.
    \end{align*}

\end{lemma}

We prove Lemma~\ref{lem:Renyi_Target_OneStepGeneral} in Section~\ref{app:PfOf_Renyi_Target_OneStepGeneral}.

\begin{lemma}\label{lem:Renyi_Target_OneStepGeneral}
    Let $\mu, \pi \in \P(\R^d)$, suppose Assumption~\ref{assumption:potential}(a)-(c) holds, and recall the constants defined in Definition~\ref{def:Constants}. Further suppose
    \[
    0 \leq h \leq \frac{25T}{7} \bigg[\Big(1+\frac{1}{4(q-1)}\Big)^{\frac{1}{2}}- 1 \bigg]
    \]
    and that $L(T^2+Th) \leq \frac{1}{12}$.
    Then for any $1<q< \infty$
    \[
    \sfR_q(\mu \tilde \bP_{T,h} \dvert \pi \bP_T) \leq \inf_{\gamma \in \Gamma(\mu, \pi)} \frac{1}{q-1} \log \E_{(x,y) \sim \gamma} \left[ e^{(q-1)(\alpha_0 + \alpha_1 \|x-y\| + \alpha_2 \|x-y\|^2 + \alpha_3 \|x\| + \alpha_4 \|x\|^2)} \right]
    \]
    where
    \begin{align*}
        \alpha_0 &\coloneqq dh(6 \sfj_c + 6\sqrt{d}\, \sfj_v + 108(q-1)dh \sfj_v^2 + h \sfp_v^2 +2 \sfp_v)\\
        \alpha_1 &\coloneqq \sqrt{d}\,(6\sqrt{d}\,\sfj_{xy} + \sfp_{xy}(h \sfp_v+1))\\
        \alpha_2 &\coloneqq \sfp_{xy}^2(1+3(q-1)(h\sfp_v+1)^2)\\
        \alpha_3 &\coloneqq h\sqrt{d}\, (6\sqrt{d}\, \sfj_x + \sfp_x(h \sfp_v +1))\\
        \alpha_4 &\coloneqq h^2\sfp_x^2(1+3(q-1)(h\sfp_v+1)^2)\,.
    \end{align*}
\end{lemma}

Lemma~\ref{lem:Renyi_Target_OneStepGeneral} together with Assumptions~\ref{assumption:OWMixing} and~\ref{assumption:OWBias} yields Theorem~\ref{thm:RenyiTarget}, proven in Section~\ref{app:PfOf_ThmRenyiTarget}.

\section{Discussion}\label{sec:Discussion}

In this work, we study the mixing time and asymptotic bias of \uhmc in KL divergence and Rényi divergence. 
To obtain our results, we use one-shot couplings to derive regularization and cross-regularization results for \uhmcv in KL and Rényi divergence, which enable us to upgrade Wasserstein guarantees to KL and Rényi divergence.

We now discuss our theoretical guarantees and their consequences, compare them with those of other sampling methods, notably the unadjusted Langevin Algorithm (uLA)~\cite{roberts1996exponential, roberts1998optimal, VW23}, and outline directions for future work.

We summarize gradient oracle complexities for different sampling methods to output a sample $X$ such that $\sfD(\law(X) \dvert \nu) \leq \epsilon$ in Table~\ref{table:complexity}, for the setting of strongly log-concave target distributions $\nu$, accuracy $\epsilon > 0$, and divergence $\sfD$; recall that we verify all of our assumptions for the strongly log-concave setting in Section~\ref{subsec:Examples}. Rows $1$--$6$ in Table~\ref{table:complexity} state guarantees for the \emph{fixed kernel scheme}, where the same Markov kernel is used throughout the sampling procedure, and Rows $7$--$10$ correspond to a \emph{composite kernel scheme}, corresponding to the modularity of our asymptotic bias bounds, where the first Markov kernel drives the Wasserstein contraction (e.g., \uhmcs in Row~7) and second Markov kernel enables upgrading the guarantee to a strong divergence via cross-regularization (e.g., \uhmcv in Row~7).

To the best of our knowledge, the only prior work studying guarantees strictly for \uhmcv in KL divergence is~\cite{camrud2023second} (Table~\ref{table:complexity}, Row~3), where they provide asymptotic bias guarantees in KL divergence under a log-Sobolev inequality assumption on the target distribution using hypocoercivity techniques. In comparison to this work, for the common setting of strongly log-concave target distributions, we obtain an improved gradient complexity in Corollary~\ref{cor:complexity-KL-verlet} (Table~\ref{table:complexity}, Row~2). Additionally, asymptotic bias bounds do not in general imply mixing time guarantees, and as far as we are aware, Theorem~\ref{thm:KLMixing} is the first mixing time guarantee for \uhmcv in KL divergence, which has several applications, e.g., studying the contraction of mutual information along the Markov chain (see Section~\ref{subsec:InformationContraction}). 
Apart from our mixing time and asymptotic bias bounds in Rényi divergence (Section~\ref{subsec:Main-Renyi}), we know of no guarantees for \ehmc or \uhmc that go beyond KL divergence. As described in Section~\ref{sec:Introduction}, our guarantees in Rényi divergence enable \uhmcv to generate warm starts for adjusted Markov chains (Table~\ref{table:complexity}, Row~4).
As highlighted in Section~\ref{subsec:MainResults}, analyzing the velocity Verlet integrator is challenging due to its structural complexities, and we find it worth noting that our gradient complexity for \uhmcv in KL divergence (Table~\ref{table:complexity}, Row~2) matches the prior best known complexity for \uhmcv in total variation distance (Table~\ref{table:complexity}, Row~1). Row~2 in Table~\ref{table:complexity} also indicates how \uhmcv surpasses the gradient complexity of uLA in KL divergence (Table~\ref{table:complexity}, Row~5) in both dimension and accuracy dependence; in fact, \uhmcv possesses the best accuracy dependence among all of the methods presented in Table~\ref{table:complexity} when the second order accuracy of the velocity Verlet integrator can be exploited (Table~\ref{table:complexity}, Rows~1~to~3).

\begin{table}[t]
\centering
\renewcommand{\arraystretch}{1.3}
\begin{tabular}{ccccc}
 & \textbf{{Algorithm}} & \textbf{{Gradient Complexity}} & \textbf{{Divergence}} ($\sfD$) & \textbf{{Reference}} \\
\cmidrule(lr){2-5}
& \underline{\emph{fixed kernel scheme}}\\
$1$ & \uhmcv & $d^{3/4}/\epsilon^{1/4}$ & $\TV^2$ & \cite{bou2023mixing} \\
$2$ & \uhmcv & $d^{3/4}/\epsilon^{1/4}$ & $\KL$ & Corollary~\ref{cor:complexity-KL-verlet} \\
$3$ & \uhmcv & $d/\epsilon^{1/4}$ & $\KL$ & \cite{camrud2023second} \\
$4$ & \uhmcv & $d^{3/2}/\epsilon$ & $\sfR_q$ & Corollary~\ref{cor:complexity-Renyi-verlet} \\
$5$ & uLA & $d/\epsilon$ & $\KL$ & \cite{VW23}\\
$6$ & uLA & $d/\epsilon$ & $\sfR_q$ & \cite{chewi2024analysis}\\
\cmidrule(lr){2-5}
& \underline{\emph{composite kernel scheme}}\\
$7$ & \uhmcs + \uhmcv  & $\max\{d^{2/3}/\epsilon^{1/3}, d^{3/4}/\epsilon^{1/4}\}$ & $\KL$ & Corollary~\ref{cor:complexity-KL-stratified}\\
$8$ & \uhmcv + uLA & $d^{5/8}/\epsilon^{3/8}$ & $\KL$ & Section~\ref{app:uLA-CR-KL}\\
$9$ & \uhmcs + uLA & $d^{1/2}/\epsilon^{1/2}$ & $\KL$ & Section~\ref{app:uLA-CR-KL}\\
$10$ & \uhmcv + uLA & $d^{3/4}/\epsilon^{3/4}$ & $\sfR_q$ & Section~\ref{app:uLA-CR-Renyi}
\end{tabular}
\caption{Gradient oracle complexities for outputting a sample $X$ such that $\sfD(\law(X) \dvert \nu) \leq \epsilon$ where $\nu$ is a strongly log-concave target distribution and $\sfD$ denotes a statistical distance or divergence. For ease of presentation, the Gradient Complexity column omits logarithmic factors and indicates dependence only on dimension $d$ and accuracy $\epsilon$. 
Rows~1--6 correspond to MCMC methods where a fixed kernel is used throughout the sampling procedure and Rows $7$--$10$ indicate the modular aspect of our framework and correspond to methods where the Wasserstein contraction is driven by the first Markov kernel and the cross-regularization by the second kernel, e.g., Row $7$ corresponds to the Wasserstein contraction driven by \uhmcs and the cross-regularization from \uhmcv. In Row~1, we state guarantees for squared total variation distance to put it on a comparable scale to KL and Rényi divergence --- Pinsker's inequality and~\eqref{eq:RenyiMonotonicity} implies $2\TV^2(\mu, \pi) \leq \KL(\mu \dvert \pi) \leq \sfR_q(\mu \dvert \pi)$ for $q>1$, $\mu, \pi \in \P(\R^d)$.}
\label{table:complexity}
\end{table}

A key aspect of our proof strategy to obtain mixing time and asymptotic bias bounds is the decoupling of the relevant Wasserstein assumptions from the regularization and cross-regularization properties of the kernel.
This decoupling is central to our modular framework, which allows asymptotic bias bounds to be established for composite kernel schemes in which the Wasserstein contraction is driven by one kernel, while the final cross-regularization step responsible for upgrading the guarantees to KL or Rényi divergence is implemented by a different kernel.
This modularity is particularly valuable in settings where establishing a cross-regularization property in KL or Rényi divergence for a given kernel is difficult, yet Wasserstein mixing and asymptotic bias bounds can still be derived.
To the best of our knowledge, the only prior method which extends Wasserstein guarantees to asymptotic bias bounds in stronger divergences is the local error framework studied in~\cite{altschuler2024shifted}. However, as the authors point out in~\cite[Section~A.2]{altschuler2024shifted}, the approach is currently limited to KL divergence as the Wasserstein--$\infty$ assumption in their framework for Rényi divergence is too restrictive to yield asymptotic bias bounds in Rényi divergence.
Rows~7--10 in Table~\ref{table:complexity} correspond to complexity guarantees derived using our framework where two different kernels are responsible for the Wasserstein contraction and the cross-regularization --- Rows~7--9 pertain to KL divergence and Row~10 to Rényi divergence.
Row~7 corresponds to the Wasserstein contraction governed by \uhmcs~\cite{bou2025unadjusted} and cross-regularization provided by \uhmcv, which we discuss in Section~\ref{subsubsection:Examples-stratified}.
We sketch the gradient complexities when the uLA kernel provides the cross-regularization in Appendix~\ref{app:uLA-CrossRegularization}, which correspond to the last three rows in Table~\ref{table:complexity}.
Guarantees corresponding to cross-regularization via uLA highlight that using \uhmc methods to drive the Wasserstein contraction shows significant improvements over using uLA throughout (Table~\ref{table:complexity}, Row~5), and that the cross-regularization properties of uLA enjoy milder dimension dependence than \uhmcv cross-regularization, making it appealing when dimension dependence is more important than accuracy or when guarantees are desired in Rényi divergence.
In future work, it would be interesting to characterize the gradient complexity guarantees that can be obtained by considering other Markov kernels, and also study how this framework interacts with other problem parameters such as the strong log-concavity and smoothness, which we do not focus on here.

Looking ahead, it would also be of interest to investigate if mixing time and asymptotic bias bounds can be obtained for \uhmc in $\Phi$-divergence~\cite{chafai2004entropies}, as would be obtaining Wasserstein--$p$ guarantees for \uhmc beyond strong log-concavity.
We note that we state Assumptions~\ref{assumption:W2Mixing} and~\ref{assumption:OWMixing} as exponentially fast contractions for convenience, however, our proof technique extends to Wasserstein contraction at weaker rates too, with that rate impacting the eventual KL or Rényi convergence.

\paragraph{Acknowledgements.} SM and AW are supported by NSF awards CCF-2403391 and CCF-2443097.

\newpage

\appendix

\section{Background on Hamiltonian Monte Carlo}\label{app:HMC-background}

We discuss several key properties of exact HMC~\cite{duane1987hybrid, neal2011mcmc} for sampling from a target probability distribution $\nu \propto e^{-f}$ supported on $\R^d$.
Throughout, we use the notation from Section~\ref{subsec:HMC}.

\begin{itemize}
    \item \textbf{Conservation of Hamiltonian.} Let $(x,v) \in \R^{2d}$ be some initial data. Then for any $t \geq 0$, $H(\HF_t(x,v)) = H(x,v)$ and therefore the Hamiltonian $H$ is conserved along the Hamiltonian dynamics~\eqref{eq:HamiltonianDynamics}. To see this, denote $(x_t, v_t) \equiv (x_t(x,v), v_t(x,v)) =  \HF_t(x,v)$. Then, using chain rule and the definition of Hamiltonian dynamics~\eqref{eq:HamiltonianDynamics}, we see that
    \[
    \frac{\d}{\dt} H(x_t, v_t) = \langle \nabla f(x_t), v_t \rangle + \langle v_t, -\nabla f(x_t) \rangle = 0 \,,
    \]
    and hence, the Hamiltonian is conserved along the Hamiltonian dynamics~\eqref{eq:HamiltonianDynamics}.
    \item \textbf{Symplecticity.} 
    Define the following $2d \times 2d$ matrix
    \[
    J \coloneq \begin{pmatrix}
        0 & I\\
        -I & 0
    \end{pmatrix}\,.
    \]
    The flow $\HF_t$ is symplectic if for all $(x,v) \in \R^{2d}$ and $t \geq 0$,
    \begin{equation}\label{eq:Symplectic}
    (\nabla \HF_t(x,v))^\top J \, \nabla \HF_t(x,v) = J. 
    \end{equation}
    Fix $(x,v) \in \R^{2d}$ and denote $M_t \coloneqq (\nabla \HF_t(x,v))^\top J \, \nabla \HF_t(x,v)$. Note that $M_0 = J$. It can be checked that $\dot M_t = 0$ and therefore, for all $t \geq 0$, $M_t = M_0 = J$. Hence the symplecticity condition~\eqref{eq:Symplectic} is true for all $(x,v) \in \R^{2d}$ and $t \geq 0$.

    An important consequence of the symplectic structure is that the Hamiltonian flow preserves volume, i.e., $\det (\nabla \HF_t(x,v)) = 1$ for all $(x,v) \in \R^{2d}$ and $t \geq 0$. To see this, fix $(x,v) \in \R^{2d}$ and take determinants on both sides of~\eqref{eq:Symplectic} to obtain $(\det (\nabla \HF_t(x,v)))^2 = 1$. Hence, $\det (\nabla \HF_t(x,v)) = \pm 1$. As $t \mapsto \det (\nabla \HF_t(x,v))$ is continuous and $\det (\nabla \HF_0(x,v)) = 1$, $\det (\nabla \HF_t(x,v)) = 1$ for all $t \geq 0$.
    \item \textbf{Time-reversal.} Let the solution along Hamiltonian dynamics~\eqref{eq:HamiltonianDynamics} at time $t \geq 0$ when starting from $(x,v)$ be $(x_t(x,v), v_t(x,v))$. Then the Hamiltonian dynamics satisfies the following time reversal property: For $t \geq 0$,
    \[
    \HF_t(x_t(x,v), -v_t(x,v)) = (x, -v)\,.
    \]
    \item \textbf{Stationary distribution.} We want to show that the target distribution $\nu$ is stationary for the \eHMC kernel $\bP_T$, i.e., $\nu \bP_T = \nu$. Our proof will use the previously mentioned properties of Hamiltonian dynamics.
    
    Let $\nu_{\sfP}(x,v) \propto  e^{-H(x,v)} \in \P(\R^{2d})$ denote the joint target distribution supported on phase space $\R^{2d}$, and therefore, $\nu_\sfP = \nu \otimes \gammad$. Showing that $\nu_\sfP$ is invariant under one step of exact HMC on phase space implies that the target distribution $\nu$ is stationary for the \eHMC kernel $\bP_T$ on position space. Recall that exact HMC consists of (i) fresh velocity refreshment drawn from $\gammad$ (independent of position) and (ii) evolution along Hamiltonian flow. Step (i) does not affect the position marginal and samples the velocity from $\gammad$, which is the velocity marginal of $\nu_\sfP$, therefore $\nu_\sfP$ is invariant under step (i). Thus it remains to check that $\nu_\sfP$ is unchanged under Hamiltonian flow, i.e., $(\HF_t)_\# \nu_\sfP = \nu_{\sfP}$.

    Let $(x,v) \in \R^{2d}$ and define $(x_t, v_t) \equiv (x_t(x,v), v_t(x,v)) =  \HF_t(x,v)$. 
    Let $ \tilde \nu_\sfP \coloneqq (\HF_t)_\# \nu_\sfP$. Then, by the change of variable formula we have that
    \[
    \nu_\sfP(x,v) = \tilde \nu_\sfP(\HF_t(x,v))\, \big| \det (\nabla \HF_t(x,v))\big| = \tilde \nu_\sfP(\HF_t(x,v))\,,
    \]
    where the last equality is because of the volume preservation under Hamiltonian flow. So, it follows that
    \[
    \tilde \nu_\sfP(\HF_t(x,v)) = \nu_\sfP(x,v) = Z^{-1}e^{-H(x,v)} = Z^{-1}e^{-H(\HF_t(x,v))}
    \]
    where the final equality is due to conservation of the Hamiltonian along the Hamiltonian dynamics, and $Z$ is some normalizing constant. Therefore, we conclude that $\tilde \nu_\sfP = \nu_\sfP$.

\end{itemize}

Unadjusted HMC with the velocity Verlet integrator also has the symplecticity and time-reversal properties, however, it does not conserve the Hamiltonian and consequently leads to a biased stationary distribution.

\section{Regularization Properties of Probability Kernels}\label{app:Regularization}

Let $\bP$ be a one-step probability kernel acting on distributions $\mu \in \P(\R^d)$ via
\[
(\mu \bP) (x) = \int \mu(z) \bP(z, x) \d z
\]
and acting on a real valued functions or observables $f : \R^d \to \R$ via
\[
(\bP f) (x) = \int f(z) \bP(x, z) \d z\,.
\]
In this way, transition kernels act both on distributions and functions to output distributions and functions respectively.

Throughout the section, we use $f,g$ to denote real valued functions and $\mu, \pi$ to denote probability distributions. We also denote the expectation of a function $f$ under distribution $\mu$ as
\[
\mu(f) = \int f(x) \mu(x) \dx\,.
\]
By the above definitions and Fubini's theorem,
\[
(\mu \bP)(f) = \int f(z) \Big( \int \mu(x) \bP(x,z) \dx \Big) \d z = \int \mu(x) \Big( \int f(z) \bP(x,z) \d z \Big) \dx = \int \mu(x) (\bP f)(x) \dx = \mu (\bP f)\,,
\]
and therefore for a transition kernel $\bP$, probability distribution $\mu$, and function $f$,
\begin{equation}\label{eq:SemigroupIdentity}
(\mu \bP)(f) = \mu (\bP f)\,.
\end{equation}

\subsection{Wasserstein--1 to TV}

In this section we study~\eqref{eq:Regularization} for the case where $\sfD = \TV$ and $W = W_1$. The total variation (TV) distance between $\mu, \pi$ is defined as
\[
\TV (\mu, \pi) = \frac{1}{2} \int | \mu(x) - \pi(x)| \dx = \frac{1}{2} \sup_{\|f\|_\infty \leq 1} \Big| \int f(x) (\mu(x) - \pi(x))  \dx \Big|
\]
where $\|f\|_\infty = \sup_{x \in \R^d}|f(x)|$\,. 
Let
\[
\|f\|_\Lip = \sup_{x \neq y} \frac{|f(x)-f(y)|}{\|x-y\|}\,.
\]
Now suppose that~\eqref{eq:Regularization} is true for $\sfD = \TV$ and $W = W_1$, i.e.,
\begin{equation}\label{eq:RegularizationTV}
\TV(\mu \bP, \pi \bP) \leq C\, W_1(\mu, \pi)
\end{equation}
holds for a kernel $\bP$ and for any distributions $\mu, \pi$. We show an implication of the regularization~\eqref{eq:RegularizationTV} in the following lemma.

\begin{lemma}
    Suppose~\eqref{eq:RegularizationTV} holds for a Markov kernel $\bP$ and for any two distributions $\mu, \pi \in \P(\R^d)$. Let $f$ be a real valued function such that $\|f\|_\infty < \infty$. Then
    \begin{equation}\label{eq:BdToLipRegularization}
    \|\bP f\|_\Lip \leq 2C\, \|f\|_\infty\,.
    \end{equation}
\end{lemma}

\begin{proof}
By the definition of TV distance,
\begin{align*}
    \TV(\mu \bP, \pi \bP) &= \frac{1}{2} \sup_{\|g\|_\infty \leq 1} \Big| \int g(x) ((\mu \bP)(x) - (\pi\bP)(x))  \dx \Big|\\
    &= \frac{1}{2} \sup_{\|g\|_\infty \leq 1} \big| \mu(\bP g) - \pi(\bP g) \big|\\
    &\geq \frac{1}{2} \big| \mu(\bP g) - \pi(\bP g) \big|
\end{align*}
where the second equality holds due to~\eqref{eq:SemigroupIdentity}, and the final claim is therefore
\[
\big| \mu(\bP g) - \pi(\bP g) \big| \leq 2 \TV(\mu\bP, \pi \bP)\,,
\]
for any bounded function $g$ such that $\|g\|_\infty \leq 1$. Using~\eqref{eq:RegularizationTV}, we obtain
\[
\big| \mu(\bP g) - \pi(\bP g) \big| \leq 2C \, W_1(\mu, \pi)
\]
for any $\mu, \pi \in \P(\R^d)$.
Let $x, y \in \R^d$ be arbitrary such that $x \neq y$. Selecting $\mu = \delta_x$ and $\pi = \delta_y$ implies that $W_1(\delta_x, \delta_y) = \|x-y\|$ and consequently,
\[
\big| (\bP g)(x) - (\bP g)(y)\big| \leq 2C\, \|x-y\|\,.
\]
Hence for any $g$ such that $\|g\|_\infty < 1$ we conclude that
\[
\|\bP g\|_\Lip \leq 2C\,.
\]
For a function $f$ such that $\|f\|_\infty < \infty$, the above argument applies to $g = \frac{f}{\|f\|_\infty}$ and we can conclude that for all bounded functions $f$,
\begin{equation*}
\|\bP f\|_\Lip \leq 2C\, \|f\|_\infty\,.
\end{equation*}
This completes the proof.
\end{proof}

Hence we can see that under~\eqref{eq:RegularizationTV}, the Markov kernel regularizes bounded functions to Lipschitz functions. The converse also holds, i.e., we can conclude~\eqref{eq:RegularizationTV} from~\eqref{eq:BdToLipRegularization}, but we do not specify that here.

\subsection{Wasserstein--2 to KL}\label{app:W2ToKL}

Consider~\eqref{eq:Regularization} with $\sfD = \KL$ and $W = W_2^2$, i.e., suppose that there exists a constant $C>0$ such that
\begin{equation}\label{eq:RegularizationKL}
    \KL(\mu \bP \dvert \pi \bP) \leq C \, W_2^2(\mu, \pi)
\end{equation}
holds for all $\mu, \pi \in \P(\R^d)$\,.
The Donsker-Varadhan variational representation for KL divergence states that for all $\mu, \pi \in \P(\R^d)$,
\begin{equation}\label{eq:DV}
\KL(\pi_1 \dvert \pi_2) = \sup_{f: \R^d \to \R} \Big( \E_{\pi_1} [f] - \log \E_{\pi_2}[e^f] \Big)\,,
\end{equation}
where the supremum is taken over all measurable functions.  Hence,~\eqref{eq:DV} applied to~\eqref{eq:RegularizationKL} along with~\eqref{eq:SemigroupIdentity} implies that for all $f$,
\[
\mu (\bP f) - \log \pi(\bP e^f) \leq C\, W_2^2(\mu, \pi)\,.
\]
In particular, taking $\mu = \delta_x$, $\pi = \delta_y$, and $f = \log g$ for a function $g$ implies that
\begin{equation}\label{eq:logHarnack}
(\bP\log g)(x) \leq \log (\bP g)(y) + C\, \|x-y\|^2\,.
\end{equation}
For $x=y$,~\eqref{eq:logHarnack} reduces to Jensen's inequality for the concave function $\log$. For $x \neq y$,~\eqref{eq:logHarnack} is a log-Harnack inequality~\cite{wang1997logarithmic} and quantifies the one-step entropic smoothing~\eqref{eq:RegularizationKL} in terms of observables; see~\cite{altschuler2023shifted} for a detailed discussion on the duality between regularization~\eqref{eq:RegularizationKL} and Harnack inequalities~\eqref{eq:logHarnack}, along with analogous statements for Rényi divergence.

\section{Regularization via One-shot Couplings}\label{app:RegularityProofs}

\begin{proof}[Proof of Lemma~\ref{lem:DiracToPerturbed_Mixing}]
    By definition of the \uHMC transition kernel, we have the following
    \begin{align*}
        \sfD(\delta_x \tilde \bP_{T,h} \dvert \delta_y \tilde \bP_{T,h}) &\stackrel{\eqref{eq:uHMCv_update}}= \sfD (\law (\tilde q_{T,h}(x, \xi)) \dvert \law (\tilde q_{T,h}(y, \xi)))\\
        &\stackrel{\eqref{eq:MixingMap}}= \sfD (\law (\tilde q_{T,h}(y, \varphi_{x,y}(\xi))) \dvert \law (\tilde q_{T,h}(y, \xi)))\\
        &\leq \sfD (\law (\tilde \HF_{T,h}(y, \varphi_{x,y}(\xi))) \dvert \law (\tilde \HF_{T,h}(y, \xi)))\\
        &= \sfD (\law (y, \varphi_{x,y}(\xi)) \dvert \law (y, \xi))\\
        &= \sfD (\law (\varphi_{x,y}(\xi)) \dvert \law (\xi))\\
        &= \sfD((\varphi_{x,y})_\# \gammad \dvert \gammad)
    \end{align*}
    where the inequality is because projection cannot increase divergence and the third equality is because $\tilde \HF_{T,h}$ is a bijective map.
    The assumptions in the statement ensure the existence and uniqueness of the map~\eqref{eq:MixingMap}, which are due to~\cite[Corollary~14]{bou2023mixing}.
\end{proof}

\begin{proof}[Proof of Lemma~\ref{lem:DiracToPerturbed_Target}]
    By definition of the \uHMC and \eHMC transition kernels, we have the following
    \begin{align*}
        \sfD(\delta_x \tilde \bP_{T,h} \dvert \delta_y  \bP_{T}) &\stackrel{\eqref{eq:eHMC_update}, \eqref{eq:uHMCv_update}}= \sfD (\law (\tilde q_{T,h}(x, \xi)) \dvert \law ( q_{T}(y, \xi)))\\
        &\stackrel{\eqref{eq:TargetMap}}= \sfD (\law ( q_{T}(y, \Phi_{x,y}(\xi))) \dvert \law ( q_{T}(y, \xi)))\\
        &\leq \sfD (\law ( \HF_{T}(y, \Phi_{x,y}(\xi))) \dvert \law ( \HF_{T}(y, \xi)))\\
        &= \sfD (\law (y, \Phi_{x,y}(\xi)) \dvert \law (y, \xi))\\
        &= \sfD (\law (\Phi_{x,y}(\xi)) \dvert \law (\xi))\\
        &= \sfD((\Phi_{x,y})_\# \gammad \dvert \gammad)
    \end{align*}
    where the inequality is because projection cannot increase divergence and the third equality is because $ \HF_{T}$ is a bijective map. The assumptions in the statement ensure the existence and uniqueness of the map~\eqref{eq:TargetMap}, which are due to Lemma~\ref{lem:ExistenceOfTargetMap}.
\end{proof}

\section{One-shot Coupling Map for Mixing Time}\label{app:MixingOneShotMap}

\begin{proof}[Proof of Lemma~\ref{lem:MixingMapPtWise}]
    The assumptions stated in Lemma~\ref{lem:MixingMapPtWise} satisfy the assumptions of~\cite[Lemma~25]{bou2023mixing} and hence~\cite[Lemma~25]{bou2023mixing} implies~\eqref{eq:MixingMapPtWise}.
\end{proof}

\begin{proof}[Proof of Lemma~\ref{lem:MixingMapJacobian}]
    The assumptions stated in Lemma~\ref{lem:MixingMapJacobian} satisfy the assumptions of~\cite[Lemma~26]{bou2023mixing}. Hence the $\frac{11}{2}MT^2\|x-y\|$ upper bound in~\eqref{eq:MixingMapJacobian} follows from~\cite[Lemma~26]{bou2023mixing}. It remains to obtain the (global) upper bound of $\frac{2}{9}$ as mentioned in~\eqref{eq:MixingMapJacobian}. 
    As the assumptions in Lemma~\ref{lem:MixingMapJacobian} imply those of~\cite[Lemma~26]{bou2023mixing}, the global analysis of~\cite[Lemma~26]{bou2023mixing} states
    \[
    \|\nabla \varphi_{x,y}(v)-I\|_\op \leq \frac{(12/5)LT^2}{1-(6/5)LT^2}. 
    \]
    Using $LT^2 \leq \frac{1}{12}$, one gets
    \[
    \|\nabla \varphi_{x,y}(v)-I\|_\op \leq \frac{2}{9}.
    \]
\end{proof}

\section{One-shot Coupling Map for Asymptotic Bias}\label{app:TargetMap}

Recall the one-shot map for obtaining bounds on the asymptotic bias~\eqref{eq:TargetMap}. For any $x,y \in \R^d$, it is a map $\Phi_{x,y}: \R^d \to \R^d$ such that for any $v \in \R^d$
\[
\tilde q_{T,h}(x,v) = q_T(y, \Phi_{x,y}(v))\,.
\]

We now illustrate how such a map can be obtained by composing the one-shot map for bounding mixing time~\eqref{eq:MixingMap} and another map $\phi_x$, defined as follows. For any $x \in \R^d$, let $\phi_x : \R^d \to \R^d$ be a map such that for any $v \in \R^d$
\begin{equation}\label{eq:BiasMap}
    \tilde q_{T,h}(x,v) = q_T (x, \phi_x(v))\,.
\end{equation}
The existence of $\phi_x$ satisfying~\eqref{eq:BiasMap} holds under $L$--smoothness of the potential and $LT^2 \leq \frac{2}{5}\pi^2$ due to~\cite[Corollary~14]{bou2023mixing}.

Given~\eqref{eq:MixingMap} and~\eqref{eq:BiasMap}, note the following
\begin{equation}\label{eq:TargetMapAsComposition}
\tilde q_{T,h}(x,v) \stackrel{\eqref{eq:BiasMap}}= q_T (x, \phi_x(v)) \stackrel{\eqref{eq:MixingMap}}= q_T (y, \varphi_{x,y}(\phi_x(v))) = q_T(y, \Phi_{x,y}(v))
\end{equation}
and therefore a map $\Phi_{x,y}$ satisfying~\eqref{eq:TargetMap} is $\Phi_{x,y} = \varphi_{x,y} \circ \phi_x$. Note that when using~\eqref{eq:MixingMap} in the sequence of equalities, we use it for $h =0$.

\begin{lemma}\label{lem:ExistenceOfTargetMap}
    For any $x,y \in \R^d$, a map $\Phi_{x,y} : \R^d \to \R^d$ satisfying~\eqref{eq:TargetMap} exists if the potential satisfies Assumption~\ref{assumption:potential}(b) and if $LT^2 \leq \frac{2}{5}\pi^2$.
\end{lemma}

\begin{proof}
    From~\eqref{eq:TargetMapAsComposition} we know that for any $x,y \in \R^d$, a map $\Phi_{x,y}$ satisfying 
\[
\tilde q_{T,h}(x,v) = q_T(y, \Phi_{x,y}(v))
\]
is $\Phi_{x,y} = \varphi_{x,y} \circ \phi_x$ where $\varphi_{x,y}$ satisfies~\eqref{eq:MixingMap} and $\phi_x$ satisfies~\eqref{eq:BiasMap}. The existence of $\varphi_{x,y}$ and $\phi_x$ holds under $f$ being $L$--smooth and $LT^2 \leq \frac{2}{5}\pi^2$ due to~\cite[Corollary~14]{bou2023mixing}, thereby implying the existence of a map $\Phi_{x,y}$ satisfying~\eqref{eq:TargetMap} under the same conditions.
\end{proof}

Before proceeding to derive regularity bounds for the map $\Phi_{x,y}$, recall that pointwise and Jacobian regularity estimates for the map $\varphi_{x,y}$ are in Lemmas~\ref{lem:MixingMapPtWise} and~\ref{lem:MixingMapJacobian} respectively. Next we state similar bounds for the map $\phi_x$ satisfying~\eqref{eq:BiasMap}, from~\cite{bou2023mixing}.

\begin{lemma}\label{lem:BiasMapPtWise}
    Suppose Assumption~\ref{assumption:potential}(a)-(c) holds. Further suppose $h \geq 0$ and that $L(T^2+Th) \leq \frac{1}{6}$. Then for any $x,v \in \R^d$
    \[
    \|\phi_x(v)-v\| \leq 2h^2(LT^{-1}\|x\|+L\|v\|+MT^{-1}\|x\|^2+MT\|v\|^2)\,.
    \]
\end{lemma}

\begin{lemma}\label{lem:BiasMapJacobian}
    Suppose Assumption~\ref{assumption:potential}(a)-(d) holds. Further suppose $h \geq 0$ and that $L(T^2+Th) \leq \frac{1}{6}$. Then for any $x,v \in \R^d$
    \[
    \|\nabla \phi_x(v)-I\|_\op \leq \min \left\{ \frac{1}{2}, 2h^2(L+M\|x\|+MT\|v\|+(M^2T^2+N)\|x\|^2+(M^2T^4+NT^2)\|v\|^2)   \right\}\,.
    \]
\end{lemma}

Lemmas~\ref{lem:BiasMapPtWise} and~\ref{lem:BiasMapJacobian} correspond to Lemmas~27 and~28 from~\cite{bou2023mixing} respectively, but with simplified constants.
We can now obtain pointwise and Jacobian estimates for the map $\Phi_{x,y}$ satisfying~\eqref{eq:TargetMap}.

\begin{proof}[Proof of Lemma~\ref{lem:TargetMapPtWise}]
    Fix an arbitrary $v \in \R^d$. Note that
    \[
    \|\Phi_{x,y}(v)-v\| = \|\varphi_{x,y}(\phi_x(v)) - \phi_x(v)+\phi_x(v)-v\| \leq \|\varphi_{x,y}(\phi_x(v)) - \phi_x(v)\| + \|\phi_x(v)-v\|\,.
    \]
    The result now follows from Lemmas~\ref{lem:MixingMapPtWise} and~\ref{lem:BiasMapPtWise}.
\end{proof}

\begin{proof}[Proof of Lemma~\ref{lem:TargetMapJacobian}]
    Fix an arbitrary $v \in \R^d$. We know that
    \[
    \nabla \Phi_{x,y}(v) = \underbrace{\nabla \varphi_{x,y}(\phi_x(v))}_A\cdot \underbrace{\nabla\phi_x(v)}_B\,.
    \]
    Therefore
    \begin{align*}
        \|\nabla\Phi_{x,y}(v)-I\|_\op &= \|AB-I\|_\op\\
        &= \|A(B-I) + (A-I)\|_\op\\
        &\leq \|A\|_{\op}\|B-I\|_\op + \|A-I\|_\op\\
        &\leq (1+\|A-I\|_\op)\|B-I\|_\op + \|A-I\|_\op\\
        &= \|A-I\|_\op + \|B-I\|_\op + \|A-I\|_\op\|B-I\|_\op\,.
    \end{align*}
    We now use Lemmas~\ref{lem:MixingMapJacobian} and~\ref{lem:BiasMapJacobian} as they provide bounds on $\|A-I\|_\op$ and $\|B-I\|_\op$ respectively.
    Define
    \[
    Q(x,v) \coloneqq L+M\|x\|+MT\|v\|+(M^2T^2+N)\|x\|^2+(M^2T^4+NT^2)\|v\|^2\,.
    \]
    Then we have that $\|\nabla\Phi_{x,y}(v)-I\|_\op \leq$
    \[
    \min \left\{ \frac{2}{9}, \frac{11}{2}MT^2 \|x-y\|   \right\} + \min \left\{ \frac{1}{2}, 2h^2 Q(x,v)   \right\} + \min \left\{ \frac{1}{9}, \frac{4}{9}h^2 Q(x,v),  \frac{11}{4}MT^2 \|x-y\|  \right\}\,.
    \]
    Hence we have that
    \[
    \|\nabla\Phi_{x,y}(v)-I\|_\op \leq \min \left\{ \frac{15}{18}, \frac{11}{2}MT^2\|x-y\| + \frac{22}{9}h^2 Q(x,v) \right\}\,.
    \]

\end{proof}

\subsection{First-order Regularity}\label{app:TargetMap-First}

Recall from~\eqref{eq:TargetMapAsComposition} how a map $\Phi_{x,y}$ satisfying~\eqref{eq:TargetMap} can arise from the composition of maps $\varphi_{x,y}$ and $\phi_x$ satisfying~\eqref{eq:MixingMap} and~\eqref{eq:BiasMap} respectively. The quadratic dependence on $v$ in Lemmas~\ref{lem:BiasMapPtWise} and~\ref{lem:BiasMapJacobian} is therefore propagated through to the regularity estimates for $\Phi_{x,y}$ in Lemmas~\ref{lem:TargetMapPtWise} and~\ref{lem:TargetMapJacobian}. This leads to sharp bounds on the asymptotic bias of \uhmc in KL divergence (Theorem~\ref{thm:KLTarget}), however extending the proof to Rényi divergence is challenging as obtaining an analogue of Lemma~\ref{lem:KL_Target_OneStepRegularity} requires evaluating $\E_{v \sim \gammad}[e^{\lambda\|v\|^4}]$ which diverges for any $\lambda > 0$. 

The quadratic dependence on $v$ in Lemmas~\ref{lem:BiasMapPtWise} and~\ref{lem:BiasMapJacobian} arises due to a second-order trapezoidal approximation~\cite[Lemma~22]{bou2023mixing}; therefore, to obtain milder dependence on $v$ in the regularity estimates of the map $\phi_x$, we consider a first-order trapezoidal approximation and state the corresponding estimates in Lemmas~\ref{lem:BiasMap-First-PtWise} and~\ref{lem:BiasMap-First-Jacobian}. We then push this through to the map $\Phi_{x,y}$ in Lemmas~\ref{lem:TargetMap-First-PtWise} and~\ref{lem:TargetMap-First-Jacobian}. This leads to asymptotic bias bounds for \uhmc in Rényi divergence in Theorem~\ref{thm:RenyiTarget}.
First-order error bounds for the Verlet scheme required to prove Lemmas~\ref{lem:BiasMap-First-PtWise},~\ref{lem:BiasMap-First-Jacobian},~\ref{lem:TargetMap-First-PtWise}, and~\ref{lem:TargetMap-First-Jacobian} are in Section~\ref{app:First-order-error-Verlet}.

\begin{lemma}\label{lem:BiasMap-First-PtWise}
    Suppose Assumption~\ref{assumption:potential}(a)-(b) holds. Further suppose $h \geq 0$ and that $L(T^2+Th) \leq \frac{1}{6}$. Then for any $x,v \in \R^d$
    \[
    \|\phi_x(v)-v\| \leq \frac{7}{5}h \left[ \frac{1}{5}T^{-1}\|v\| + \frac{7}{36}L\|x\| \right].
    \]
\end{lemma}

\begin{proof}
    Following~\cite[Lemma~27]{bou2023mixing}, we have that
    \[
    \| \phi_x(v)-v\| \leq \frac{7}{6T} \max_{s \leq T}\| \tilde q_{s,h}(x,v) - q_s (x,v)\|.
    \]
    Therefore, our Lemma~\ref{lem:L23analogue} implies the stated bound.
\end{proof}

\begin{lemma}\label{lem:BiasMap-First-Jacobian}
    Suppose Assumption~\ref{assumption:potential}(a)-(c) holds. Further suppose $h \geq 0$ and that $L(T^2+Th) \leq \frac{1}{6}$. Then for any $x,v \in \R^d$
    \[
    \|\nabla \phi_x(v)-I\|_\op \leq \min \left\{ \frac{1}{2}, \frac{2}{15}h \left( 2T^{-1} + \frac{16}{5}MT \|x\| + 20MT^2 \|v\|  \right)   \right\}\,.
    \]
\end{lemma}

\begin{proof}
    The global bound of $\frac{1}{2}$ follows in the same way as in~\cite[Lemma~28]{bou2023mixing}. To get the local bound,~\cite[Lemma~28]{bou2023mixing} states
    \[
    \| \nabla \phi_x(v)-I\|_\op \leq \frac{10}{9T}\| \nabla_2 \tilde q_{T,h}(x,v) - \nabla_2 q_T(x,v)\|_\op + \frac{14}{15} MT^3\| \phi_x(v)-v\|\,.
    \]
    Substituting from Lemmas~\ref{lem:BiasMap-First-PtWise} and~\ref{lem:L24analogue} and simplifying gives us
    \[
    \|\nabla \phi_x(v)-I\|_\op \leq \frac{2}{15}h \left( 2T^{-1} + \frac{16}{5}MT \|x\| + 20MT^2 \|v\|  \right)\,.
    \]
    
\end{proof}

We can now prove Lemmas~\ref{lem:TargetMap-First-PtWise} and~\ref{lem:TargetMap-First-Jacobian}.

\begin{proof}[Proof of Lemma~\ref{lem:TargetMap-First-PtWise}]

The proof follows similarly to that of Lemma~\ref{lem:TargetMapPtWise}. 
Fix an arbitrary $v \in \R^d$. Note that
    \[
    \|\Phi_{x,y}(v)-v\| = \|\varphi_{x,y}(\phi_x(v)) - \phi_x(v)+\phi_x(v)-v\| \leq \|\varphi_{x,y}(\phi_x(v)) - \phi_x(v)\| + \|\phi_x(v)-v\|\,.
    \]
    The result now follows from Lemmas~\ref{lem:MixingMapPtWise} and~\ref{lem:BiasMap-First-PtWise}.

\end{proof}

\begin{proof}[Proof of Lemma~\ref{lem:TargetMap-First-Jacobian}]

The proof follows similarly to that of Lemma~\ref{lem:TargetMapJacobian}. 
Fix an arbitrary $v \in \R^d$.
We know that
\[
\|\nabla\Phi_{x,y}(v)-I\|_\op \leq \|A-I\|_\op + \|B-I\|_\op + \|A-I\|_\op\|B-I\|_\op
\]
where $A \coloneqq  \nabla \varphi_{x,y}(\phi_x(v))$, $B \coloneqq \nabla\phi_x(v)$.

We now use Lemmas~\ref{lem:MixingMapJacobian} and~\ref{lem:BiasMap-First-Jacobian} as they provide bounds on $\|A-I\|_\op$ and $\|B-I\|_\op$ respectively.
Define
\[
Q(x,v) \coloneqq 2T^{-1} + \frac{16}{5}MT \|x\| + 20MT^2 \|v\|.
\]
Then we have that $\|\nabla\Phi_{x,y}(v)-I\|_\op \leq$
\[
\min \left\{\frac{2}{9}, \frac{11}{2}MT^2\|x-y\| \right\} + \min \left\{ \frac{1}{2}, \frac{2}{15}h\, Q(x,v)   \right\} + \min \left\{  \frac{1}{9}, \frac{4}{135}h \, Q(x,v), \frac{11}{4} MT^2 \|x-y\| \right\}.
\]
Hence we have that
\[
    \|\nabla\Phi_{x,y}(v)-I\|_\op \leq \min \left\{ \frac{15}{18}, \frac{11}{2}MT^2\|x-y\| + \frac{22}{135}h\, Q(x,v) \right\}\,.
\]

\end{proof}

\subsection{First-order error bounds for Verlet scheme}\label{app:First-order-error-Verlet}

Recall from Section~\ref{subsec:HMC} that $\tilde q_{t_1, t_2}$ for any $t_2$ that divides $t_1$ is defined as $\tilde q_{t_1,t_2} \coloneqq \Pi_1 \circ \tilde \HF_{t_1,t_2}$. Similarly let $\tilde v_{t_1, t_2} \coloneqq \Pi_2 \circ \tilde \HF_{t_1,t_2}$. When $t_2 = 0$ we denote $\tilde v_{t_1, 0}$ as $v_{t_1}$. Further define
\[
\lb{t} \coloneqq \max\{ s\in h\mathbb{Z} \ : \ s\leq t \} \text{ ~~and~~ } \ub{t} \coloneqq \min\{ s\in
h\mathbb{Z} \ : \ s\geq t \} \text{ ~~for } h > 0.
\]

The following lemma is the first-order analogue of~\cite[Lemma~22]{bou2023mixing}.

\begin{lemma}\label{lem:trapezoidal}
    Fix $T > 0$ and let $u : [0,T] \to \R$ be a differentiable function such that $\max_{s \in [0,T]}|u'(s)| \leq B_1$. Then for any $h > 0$ such that $h$ divides $T$, we have
    \[
    \left\|\int_0^T (T-s)u(s) \ds - \frac{1}{2}\int_0^T (T-s)[u(\lfloor s \rfloor_h) + u(\lceil s \rceil_h)] \ds \right\| \leq \frac{B_1 h T^2}{2}\,.
    \]
\end{lemma}

\begin{proof}

Let $t_k = kh$ and $t_{k+1} = (k+1)h$ and consider the interval $[t_k, t_{k+1}]$. As per~\cite[Lemma~22]{bou2023mixing}, the error over this interval is
\[
\epsilon_k = \left\| \int_{t_k}^{t_{k+1}} \left[ \frac{s^2}{2} - Ts - \alpha_k  \right] u'(s) \ds  \right\|
\]
where
\[
\alpha_k = \frac{1}{4} (t_{k+1}^2 - 2T t_{k+1} - 2T t_k + t_k^2)\,.
\]
We then have
\begin{align*}
    \epsilon_k \leq B_1 \int_{t_k}^{t_{k+1}} \left| \frac{s^2}{2} - Ts - \alpha_k \right| \ds\,.
\end{align*}
We plan to bound the integral by its maximum value times the width $h$. Let 
\[
F(s) = \frac{s^2}{2} - Ts - \alpha_k
\]
and note that $F$ is convex and decreasing on $[t_k, t_{k+1}]$ as $F'(s) = s-T \leq 0$ for $s \in [t_k, t_{k+1}]$. Hence the maximum value of $F$ over the interval $[t_k, t_{k+1}]$ will be achieved at one of its endpoints.
We see that
\[
F(t_k) = \frac{h}{2} \left( T - \frac{(t_k + t_{k+1})}{2} \right)
\]
and
\[
F(t_{k+1}) = -\frac{h}{2}\left( T - \frac{(t_k + t_{k+1})}{2} \right).
\]
Hence
\[
\epsilon_k \leq B_1 h \left[  \frac{h}{2} \left( T - \frac{(t_k + t_{k+1})}{2} \right) \right] \leq \frac{B_1 h^2 T}{2}\,.
\]
To get the overall error we multiply by $\frac{T}{h}$ to get that
\[
    \left\|\int_0^T (T-s)u(s) \ds - \frac{1}{2}\int_0^T (T-s)[u(\lfloor s \rfloor_h) + u(\lceil s \rceil_h)] \ds \right\| \leq \frac{B_1 h T^2}{2}\,.
\]

\end{proof}

The following lemma is the analogue of~\cite[Lemma~23]{bou2023mixing}.

\begin{lemma}\label{lem:L23analogue}
    Suppose Assumption~\ref{assumption:potential}(a)-(b) holds. Further suppose $h \geq 0$ and that $L(T^2 + Th) \leq \frac{1}{6}$. For any $x,v \in \R^d$
    \[
    \max_{s \leq T} \| q_s(x,v)- \tilde q_{s,h}(x,v)\| \leq h \left( \frac{6}{25}\|v\| + \frac{7}{30}LT\|x\| \right).
    \]
\end{lemma}

\begin{proof}
We begin with~\cite[Lemma~20]{bou2023mixing} with $h=0$ and under the assumption $LT^2 \leq \frac{1}{6}$ to obtain~\cite[(57)]{bou2023mixing}
\begin{equation}\label{eq:v-magnitude-bound}
    \max_{s \leq T}\|v_s(x,v)\| \leq \frac{7}{6}LT \|x\| + \frac{6}{5}\|v\|\,.
\end{equation}
Consider the shorthand $q_t$ and $\tilde q_t$ for $q_t(x,v)$ and $\tilde q_{t,h}(x,v)$ respectively. 
Following~\cite[Lemma~23]{bou2023mixing}, we have that
\begin{align*}
 q_T &  - \tilde q_T \ = \ \rn{1} + \rn{2} + \rn{3}  \text{~~where} \\
\rn{1} \ &:= \ \frac{1}{2} \int_0^T (T-s) [ \nabla f(\tilde q_{\lb{s}} ) - \nabla f(q_{\lb{s}} ) + \nabla f(\tilde q_{\ub{s}}) - \nabla f(q_{\ub{s}}) ] \ds   \\
&  - \frac{1}{2} \int_0^T ( s - \lb{s}) [ \nabla f(\tilde q_{\ub{s}}) - \nabla f(q_{\ub{s}}) - (\nabla f(\tilde q_{\lb{s}}) - \nabla f(q_{\lb{s}}) ) ] \ds   \\
\rn{2} \ &:= \ -\frac{1}{2} \int_0^T ( s - \lb{s}) [ \nabla f(q_{\ub{s}}) - \nabla f(q_{\lb{s}}) ] \ds     \\
\rn{3} \ &:= \ - \int_0^T (T-s) \nabla f(q_s) \ds + \frac{1}{2} \int_0^T (T-s) [ \nabla f(q_{\lb{s}} ) + \nabla f(q_{\ub{s}}) \ds  \,.
\end{align*} 
To bound $|\rn{1}|$ we have~\cite[(60)]{bou2023mixing}
\[
|\rn{1}| \leq L T^2 \max_{s \le T} \| q_s - \tilde q_s\| \leq \ \frac{1}{6} \, \max_{s \le T} \|q_s - \tilde q_s\|
\]
and $|\rn{2}|$ is bounded in~\cite[(61)]{bou2023mixing} by
\begin{align*}
    |\rn{2}| \leq Lh^2 \left( \frac{7}{72}\|x\| + \frac{3}{5}T\|v\| \right).
\end{align*}
Next we apply Lemma~\ref{lem:trapezoidal} with $u(s) = -\nabla f(q_s(x,v))$, for which
\[
|u'(s)| \leq L \|v_s\| \stackrel{\eqref{eq:v-magnitude-bound}}\leq \frac{7}{6}L^2T \|x\| + \frac{6}{5}L\|v\| \eqqcolon B_1\,.
\]
Therefore, additionally using $LT^2 \leq \frac{1}{6}$\,, we get
\[
|\rn{3}| \leq \frac{h}{12} \left[ \frac{6}{5}\|v\| + \frac{7}{6}LT\|x\| \right].
\]
Hence
\[
\max_{s \leq T} \| q_s(x,v)- \tilde q_{s,h}(x,v)\| \leq \frac{6}{5} \left[ |\rn{2}| + |\rn{3}| \right]
\]
which simplifies to
\[
\max_{s \leq T} \| q_s(x,v)- \tilde q_{s,h}(x,v)\| \leq \frac{6}{5}h \left[ \left( LhT \frac{3}{5} + \frac{1}{10} \right)\|v\| + \left( Lh \frac{7}{72} + LT \frac{7}{72}\right)\|x\| \right].
\]
Upper bounding $h \leq T$ and using $LTh \leq \frac{1}{6}$ yields the following upon simplification
\[
\max_{s \leq T} \| q_s(x,v)- \tilde q_{s,h}(x,v)\| \leq h \left( \frac{6}{25}\|v\| + \frac{7}{30}LT\|x\| \right).
\]

\end{proof}

The following lemma is the analogue of~\cite[Lemma~24]{bou2023mixing}.

\begin{lemma}\label{lem:L24analogue}
    Suppose Assumption~\ref{assumption:potential}(a)-(c) holds. Further suppose $h \geq 0$ and that $L(T^2 + Th) \leq \frac{1}{6}$. For any $x,v \in \R^d$
\[
\max_{s \leq T} \| \nabla_2q_s(x,v)- \nabla_2\tilde q_{s,h}(x,v)\|_\op \leq \frac{6}{25}h \left( 1 + \frac{7}{5}MT^2 \|x\| + \frac{216}{25}MT^3\|v\| \right).
\]
\end{lemma}

\begin{proof}
Consider the shorthand $q_t$ and $\tilde q_t$ for $q_t(x,v)$ and $\tilde q_{t,h}(x,v)$ respectively. 
Following~\cite{bou2023mixing}, we can write $\nabla_2 q_T - \nabla_2 \tilde q_T = \rn{1} + \rn{2} + \rn{3} + \rn{4}$ with $\rn{1}, \rn{2} , \rn{3}, \rn{4}$ given in~\cite[(64)]{bou2023mixing}.

Equations $(65)$, $(66)$, and $(67)$ in~\cite{bou2023mixing} give bounds on $\|\rn{1}\|_\op$, $\|\rn{2}\|_\op$, and $\|\rn{3}\|_\op$ respectively, which are as follows
\begin{align*}
&\|\rn{1}\|_\op \leq \frac{1}{6} \max_{s \leq T}\|\nabla_2 q_s - \nabla_2 \tilde q_s\|_\op \\
&\|\rn{2}\|_\op \leq \frac{6}{5}MT^3 \max_{s \leq T}\|q_s - \tilde q_s \|\\
&\|\rn{3}\|_\op \leq \frac{3}{5}h^2 LT + h^2 M \left( \frac{7}{60}T \|x\| + \frac{18}{25}T^2 \|v\| \right)\,.
\end{align*}

To bound $\|\rn{4}\|_\op$ we apply Lemma~\ref{lem:trapezoidal} with $u(s) = -\nabla^2 f(q_s) \nabla_2 q_s$ for which 
\[
\|u'(s)\|_\op \leq \frac{6}{5}\left( M \left( \frac{7}{36}\|x\| + \frac{6}{5}T\|v\| \right) + L \right) \eqqcolon B_1\,.
\]
Hence
\[
\|\rn{4}\|_\op \leq \frac{3hT^2}{5}\left( M \left( \frac{7}{36}\|x\| + \frac{6}{5}T\|v\| \right) + L \right).
\]
Therefore we get
\[
\max_{s \leq T} \| \nabla_2q_s(x,v)- \nabla_2\tilde q_{s,h}(x,v)\|_\op \leq \frac{6}{5} \left[\|\rn{2}\|_\op + \|\rn{3}\|_\op + \|\rn{4}\|_\op  \right].
\]
Using Lemma~\ref{lem:L23analogue} to bound $\|\rn{2}\|_\op$ and simplifying, we get
\begin{align*}
    \max_{s \leq T} \| \nabla_2q_s(x,v)- \nabla_2\tilde q_{s,h}(x,v)\|_\op \leq \frac{6}{25}hT &\Bigg[ 3hL + 3LT + \|x\| \Big( \frac{7}{10}LMT^2 (h+T) + \frac{7}{12}hM + \frac{7}{12}TM \Big)\\
    & + \|v\| \Big( \frac{36}{5}MT^2 \big(\frac{3}{5} LhT + \frac{1}{10} \big) + \frac{18}{5}hMT + \frac{18}{5}MT^2 \Big)\Bigg].
\end{align*}
Upper bounding $h \leq T$ and using $L(T^2 + Th) \leq \frac{1}{6}$ we get
\[
\max_{s \leq T} \| \nabla_2q_s(x,v)- \nabla_2\tilde q_{s,h}(x,v)\|_\op \leq \frac{6}{25}h \left( 1 + \frac{7}{5}MT^2 \|x\| + \frac{216}{25}MT^3\|v\| \right).
\]

\end{proof}

\section{Mixing Time in KL Divergence}

\subsection{Proof of Lemma~\ref{lem:KL_Perturbed_Mixing}}\label{app:PfOf_KL_Perturbed_Mixing}

\begin{proof}
    By the change of variables formula, for any $v \in \R^d$
    \[
    ((\psi^{-1})_\# \gammad )(v) = \gammad(\psi (v)) \det \nabla \psi(v)
    \]
    where $\det\nabla\psi(v)>0$ since $\|\nabla\psi(v)-I\|_{\op}<1$.
    Using invariance of the KL divergence under simultaneous
bijective transformations \cite[Corollary~2.18]{polyanskiywu_book}, we bound $\KL (\gammad \dvert (\psi^{-1})_\# \gammad)$ instead of $\KL(\psi_\# \gammad \dvert \gammad)$. 
    We have
    \begin{align*}
        \KL(\psi_\# \gammad \dvert \gammad) &= \KL (\gammad \dvert (\psi^{-1})_\# \gammad)\\
        &= \E_{v \sim \gammad} \Big[ \log \frac{\gammad(v)}{((\psi^{-1})_\# \gammad )(v)} \Big]\\
        &= \E_{v \sim \gammad} \Big[ \log \frac{\gammad(v)}{\gammad(\psi (v))|\det \nabla \psi(v)|} \Big]\\
        &= \E_{v \sim \gammad} \Big[  -\frac{1}{2}\|v\|^2 + \frac{1}{2}\|\psi(v)\|^2 - \log |\det \nabla \psi(v)|  \Big]\\
        &= \E_{v \sim \gammad} \Big[ \frac{1}{2}\|\psi(v)-v\|^2 + \langle \psi(v)-v, v\rangle - \log |\det \nabla \psi(v)| \Big]\\
        &= \E_{v \sim \gammad} \Big[ \frac{1}{2}\|\psi(v)-v\|^2 + \Tr (\nabla \psi (v) - I) - \log |\det \nabla \psi(v)| \Big],
    \end{align*}
    where the last step is by integration by parts.
    As $\|\nabla \psi(v) - I\|_{\op} \leq \mfm_2 < 1$ for all $v \in \R^d$,~\cite[Theorem~1.1]{rump2018estimates} implies
    \[
    \Tr (\nabla \psi(v) - I) - \log |\det \nabla \psi (v)| \leq \frac{\|\nabla \psi (v) - I \|_F^2}{2(1-\|\nabla \psi (v) - I\|_\op)}\,.
    \]
    Using $\|\cdot\|_F^2 \leq d \|\cdot\|_\op^2$ and the bounds $\mfm_1$ and $\mfm_2$, we get
    \[
    \KL(\psi_\# \gammad \dvert \gammad) \leq \frac{1}{2}\mfm_1^2 + \frac{d\, \mfm_2^2}{2(1-\mfm_2)}\,.
    \]

\end{proof}

\subsection{Proof of Lemma~\ref{lem:KL_Mixing_OneStepRegularity}}\label{app:PfOf_KL_Mixing_OneStepRegularity}

\begin{proof}
    Applying Lemmas~\ref{lem:DiracToPerturbed_Mixing} and~\ref{lem:KL_Perturbed_Mixing} with $\psi = \varphi_{x,y}$~\eqref{eq:MixingMap} and bounds $\mfm_1$ and $\mfm_2$ given by Lemmas~\ref{lem:MixingMapPtWise} and~\ref{lem:MixingMapJacobian} respectively, gives us
    \[
    \KL(\delta_x \tilde \bP_{T,h} \dvert \delta_y \tilde \bP_{T,h}) \leq \frac{9}{4T^2}\|x-y\|^2 + \frac{1089}{56} dM^2T^4 \|x-y\|^2.
    \]
    Upper bounding $\frac{1089}{56} < 20$ finishes the proof.
\end{proof}

\subsection{Proof of Lemma~\ref{lem:KL_Mixing_OneStepGeneral}}\label{app:PfOf_KL_Mixing_OneStepGeneral}

\begin{proof}
    Combining Lemmas~\ref{lem:KL_Mixing_OneStepRegularity} and~\ref{lem:KLConvexityOneStep}, and picking $\gamma$ to be the optimal $W_2$--coupling between $\mu$ and $\pi$ completes the proof.
\end{proof}

\subsection{Proof of Theorem~\ref{thm:KLMixing}}\label{app:PfOf_ThmKLMixing}

\begin{proof}
    When $k=0$, the result holds via Lemma~\ref{lem:KL_Mixing_OneStepGeneral} with $\pi = \tilde \nu_h$. 
    For $k \geq 1$, note that
    \[
    \KL(\mu \tilde \bP_{T,h}^{k+1} \dvert \tilde \nu_h) = \KL((\mu \tilde \bP_{T,h}^{k})\tilde \bP_{T,h} \dvert \tilde \nu_h) \leq \left( \frac{9}{4T^2} + 20dM^2T^4 \right) W_2^2(\mu \tilde \bP_{T,h}^{k}, \tilde \nu_h)
    \]
    where the inequality is by Lemma~\ref{lem:KL_Mixing_OneStepGeneral}. Applying Assumption~\ref{assumption:W2Mixing} finishes the proof.
\end{proof}

\section{Asymptotic Bias in KL Divergence}

\subsection{Proof of Lemma~\ref{lem:KL_Target_OneStepRegularity}}\label{app:PfOf_KL_Target_OneStepRegularity}

\begin{proof}

Fix arbitrary points $x,y \in \R^d$. Following the change of variable and simplification as in the proof of Lemma~\ref{lem:KL_Perturbed_Mixing} in Section~\ref{app:PfOf_KL_Perturbed_Mixing}, we get that
\[
\KL((\Phi_{x,y})_\# \gammad \dvert \gammad) = \E_{v \sim \gammad} \Big[ \frac{1}{2}\|\Phi_{x,y}(v)-v\|^2 + \Tr (\nabla \Phi_{x,y} (v) - I) - \log |\det \nabla \Phi_{x,y}(v)| \Big]\,.
\]
Due to Lemma~\ref{lem:TargetMapJacobian}, $\|\nabla \Phi_{x,y}(v) - I\|_{\op}  < 1$ for all $v \in \R^d$, hence~\cite[Theorem~1.1]{rump2018estimates} implies
    \begin{align*}
    \Tr (\nabla \Phi_{x,y}(v) - I) - \log |\det \nabla \Phi_{x,y} (v)| &\leq \frac{\|\nabla \Phi_{x,y} (v) - I \|_F^2}{2(1-\|\nabla \Phi_{x,y} (v) - I\|_\op)}\\
    &\leq \frac{ d\|\nabla \Phi_{x,y} (v) - I \|_\op^2}{2(1-\|\nabla \Phi_{x,y} (v) - I\|_\op)}\\
    &\leq 3d\|\nabla \Phi_{x,y} (v) - I \|_\op^2
    \end{align*}
    where in the second inequality we use that $\| \cdot\|_F^2 \leq d \| \cdot\|_\op^2$ and in the last inequality we use Lemma~\ref{lem:TargetMapJacobian}.
    Therefore we have that
    \[
    \KL((\Phi_{x,y})_\# \gammad \dvert \gammad) \leq \E_{v \sim \gammad} \Big[ \frac{1}{2}\|\Phi_{x,y}(v)-v\|^2 + 3d\|\nabla \Phi_{x,y} (v) - I \|_\op^2 \Big]\,.
    \]
    Lemma~\ref{lem:TargetMapPtWise} states that
    \[
    \|\Phi_{x,y}(v)-v\| \leq \frac{3}{2T}\|x-y\| + 2h^2R(x,v)\,,
    \]
    where
    \[
    R(x,v) \coloneqq LT^{-1}\|x\|+L\|v\|+MT^{-1}\|x\|^2+MT\|v\|^2\,,
    \]
    and Lemma~\ref{lem:TargetMapJacobian} states that
    \[
    \|\nabla\Phi_{x,y}(v)-I\|_\op \leq \frac{11}{2}MT^2\|x-y\| + \frac{22}{9}h^2 Q(x,v) \,,
    \]
    where 
    \[
    Q(x,v) \coloneqq L+M\|x\|+MT\|v\|+(M^2T^2+N)\|x\|^2+(M^2T^4+NT^2)\|v\|^2\,.
    \]
    Therefore we have that
    \[
    \KL((\Phi_{x,y})_\# \gammad \dvert \gammad) \leq \E_{v \sim \gammad} \left[ \frac{1}{2} \left( \frac{3}{2T}\|x-y\| + 2h^2R(x,v) \right)^2 + 3d \left(  \frac{11}{2}MT^2\|x-y\| + \frac{22}{9}h^2 Q(x,v)   \right)^2 \right]\,.
    \]
    Using $(a_1+a_2)^2 \leq 2a_1^2 + 2a_2^2$ for $a_1,a_2 \in \R$ we get
    \[
    \KL((\Phi_{x,y})_\# \gammad \dvert \gammad) \leq \left( \frac{9}{4T^2} + \frac{363}{2}dM^2T^4 \right)\|x-y\|^2 + 4h^4 \E_{v \sim \gammad} \left[ R^2(x,v) + d\frac{242}{27} Q^2(x,v)  \right]\,.
    \]
    Applying $(a_1 + a_2 + \dots + a_k)^2 \leq k(a_1^2 + a_2^2 + \dots + a_k^2)$ for $a_1,\dots,a_k \in \R$ with $k=4$ and $5$ on $R^2(x,v)$ and $Q^2(x,v)$ respectively, and simplifying constants, we get that
    \begin{align*}
        \E_{v \sim \gammad}\Big[R^2(x,v) + d \frac{242}{27} Q^2(x,v)\Big] \leq ~&45 dL^2 + \Big(4 \frac{L^2}{T^2} + 45dM^2 \Big)\|x\|^2 + \Big(4\frac{M^2}{T^2} + 45d(M^2T^2 +N)^2\Big)\|x\|^4\\
        &+\E_{v \sim \gammad} \Big[ (4L^2 + 45dM^2T^2)\|v\|^2 + \big(4M^2T^2 + 45d(M^2T^4 + NT^2)^2\big)\|v\|^4 \Big]\,.
    \end{align*}
    Using $\E_{\gammad}[\|v\|^2]= d$ and $\E_{\gammad}[\|v\|^4] = d(d+2)$ gives us the desired bound.

\end{proof}

\subsection{Proof of Lemma~\ref{lem:KL_Target_OneStepGeneral}}\label{app:PfOf_KL_Target_OneStepGeneral}

\begin{proof}
We start with Lemma~\ref{lem:KLConvexityOneStep} with $\bP = \tilde \bP_{T,h}$ and $\bQ = \bP_T$. Picking $\gamma$ to be the optimal $W_2$--coupling between $\mu$ and $\pi$, the statement follows from Lemmas~\ref{lem:DiracToPerturbed_Target} and~\ref{lem:KL_Target_OneStepRegularity}.

\end{proof}

\subsection{Proof of Theorem~\ref{thm:KLTarget}}\label{app:PfOf_ThmKLTarget}

\begin{proof}
    Let $\tilde \bP$ and $\bP$ be shorthand for $\tilde \bP_{T,h}$ and $\bP_T$ respectively.
    
    For $k=0$, Lemma~\ref{lem:KL_Target_OneStepGeneral} gives a bound on $\KL(\mu \tilde \bP \dvert \nu)$ and therefore Theorem~\ref{thm:KLTarget} holds.

    Now suppose $k \geq 1$. 
    As $\KL(\mu \tilde \bQ^k \tilde \bP \dvert \nu) = \KL((\mu \tilde \bQ^k)\tilde \bP \dvert \nu \bP)$ we can apply Lemma~\ref{lem:KL_Target_OneStepGeneral} to obtain
    \begin{align*}
        \KL(\mu \tilde \bQ^k \tilde \bP \dvert \nu) \leq \Big( \frac{9}{4T^2} + &\frac{363}{2}dM^2T^4 \Big)W_2^2(\mu \tilde \bQ^k, \nu) +\\
        4h^4 \Bigg[ 45 dL^2& + \Big(4 \frac{L^2}{T^2} + 45dM^2 \Big)m_2(\mu \tilde \bQ^k) + \Big(4\frac{M^2}{T^2} + 45d(M^2T^2 +N)^2\Big)m_4(\mu \tilde \bQ^k) \\
         +(4&L^2 + 45dM^2T^2)d + \big(4M^2T^2 + 45d(M^2T^4 + NT^2)^2\big)d(d+2) \Bigg]\,.
    \end{align*}
    Using triangle inequality and $(a_1+a_2)^2 \leq 2a_1^2 + 2a_2^2$ for $a_1,a_2 \in \R$, we have that
    \begin{align*}
    W_2^2(\mu \tilde \bQ^k, \nu) &\leq 2 W_2^2(\mu \tilde \bQ^k, \tilde \nu_{\tilde \bQ}) + 2W_2^2(\tilde \nu_{\tilde \bQ}, \nu)\\
    &\leq 2 c_1^2 e^{-2c_2 k}\,W_2^2(\mu, \tilde \nu_{\tilde \bQ}) + 2 \Delta_h^2
    \end{align*}
    where the last inequality follows from Assumptions~\ref{assumption:W2Mixing} and~\ref{assumption:W2Bias}.
    Therefore we obtain

    \[
    \KL(\mu \tilde \bQ^k \tilde \bP \dvert \nu) \leq 2c_1^2 e^{-2c_2 k} \Big( \frac{9}{4T^2} + \frac{363}{2}dM^2T^4 \Big) W_2^2(\mu, \tilde \nu_{\tilde \bQ}) + \mathsf{bias}
    \]
where 
\begin{align*}
        &\mathsf{bias} \leq 2 \Delta_h^2\Big( \frac{9}{4T^2} + \frac{363}{2}dM^2T^4 \Big) + 4h^4 \Bigg[ 45 dL^2 + \Big(4 \frac{L^2}{T^2} + 45dM^2 \Big)m_2(\mu \tilde \bQ^{k})\\
        & + \Big(4\frac{M^2}{T^2} + 45d(M^2T^2 +N)^2\Big)m_4(\mu \tilde \bQ^{k}) +(4L^2 + 45dM^2T^2)d + \big(4M^2T^2 + 45d(M^2T^4 + NT^2)^2\big)d(d+2) \Bigg]\,.
\end{align*}

\end{proof}

\section{Mixing Time in Rényi Divergence}

\subsection{Proof of Lemma~\ref{lem:Renyi_Perturbed_Mixing}}\label{app:PfOf_Renyi_Perturbed_Mixing}

\begin{proof}
Recall that both $\gammad$ and $\psi_\#\gammad$ admit densities with respect to
Lebesgue measure, which we denote by $\gammad(v)$ and $(\psi_\#\gammad)(v)$,
respectively.  Since $\|\nabla\psi(v)-I\|_{\op}<1$, we have
$\det\nabla\psi(v)>0$ for all $v\in\R^d$, and the change-of-variables formula
yields
\begin{equation}\label{eq:ChangeOfVariableRenyi}
(\psi_\#\gammad)(\psi(v))
=
\frac{\gammad(v)}{\det\nabla\psi(v)},
\qquad v\in\R^d.
\end{equation}

By definition of R\'enyi divergence (Definition~\ref{def:RenyiDivergence}),
\[
\sfR_q(\psi_\#\gammad \dvert \gammad)
=
\frac{1}{q-1}
\log
\E_{V\sim\gammad}
\left[
\left(
\frac{(\psi_\#\gammad)(V)}{\gammad(V)}
\right)^q
\right].
\]
Rewriting the integrand,
\[
\left(
\frac{(\psi_\#\gammad)(v)}{\gammad(v)}
\right)^q
\gammad(v)
=
\left(
\frac{(\psi_\#\gammad)(v)}{\gammad(v)}
\right)^{q-1}
(\psi_\#\gammad)(v),
\]
we may equivalently express the R\'enyi divergence as
\[
\sfR_q(\psi_\#\gammad \dvert \gammad)
=
\frac{1}{q-1}
\log
\E_{U\sim\psi_\#\gammad}
\left[
\left(
\frac{(\psi_\#\gammad)(U)}{\gammad(U)}
\right)^{q-1}
\right].
\]

We now use the defining property of the pushforward measure: for any measurable
function $F:\R^d\to\R$,
\[
\E_{U\sim\psi_\#\gammad}[F(U)]
=
\E_{V\sim\gammad}[F(\psi(V))].
\]
Applying this identity with
\[
F(u)
=
\left(
\frac{(\psi_\#\gammad)(u)}{\gammad(u)}
\right)^{q-1},
\]
we obtain
\[
\sfR_q(\psi_\#\gammad \dvert \gammad)
=
\frac{1}{q-1}
\log
\E_{V\sim\gammad}
\left[
\left(
\frac{(\psi_\#\gammad)(\psi(V))}{\gammad(\psi(V))}
\right)^{q-1}
\right].
\]

Substituting the change-of-variables identity~\eqref{eq:ChangeOfVariableRenyi}
then yields an expression entirely in terms of $V\sim\gammad$, which we bound
in the remainder of the proof.  In particular,
\begin{align*}
    \sfR_q(\psi_\#\gammad \dvert \gammad) 
    &= \frac{1}{q-1} \log \E_{v \sim \gammad} \left[ \left(\frac{\gammad(v)|\det \nabla \psi (v)|^{-1}}{\gammad (\psi(v))}\right)^{q-1} \right]\\
    &= \frac{1}{q-1} \log \E_{v \sim \gammad} \left[ \exp \left( \frac{q-1}{2}(\|\psi(v)\|^2 - \|v\|^2) \right) | \det \nabla \psi(v)|^{-(q-1)}   \right]
\end{align*}
We also have the following bounds
\begin{align*}
    \|\psi(v)\|^2 - \|v\|^2 &= 2 \langle v, \psi(v)-v\rangle + \| \psi(v)-v\|^2\\
    &\leq 2 \|v\| \| \psi(v)-v\| + \mfm_1^2\\
    &\leq 2 \mfm_1 \|v\| + \mfm_1^2
\end{align*}
and
\[
| \det \nabla \psi (v)| \geq (1-\mfm_2)^d\,.
\]
Therefore we obtain
\begin{align*}
    \sfR_q(\psi_\#\gammad \dvert \gammad) &= \frac{1}{q-1} \log \E_{v \sim \gammad} \left[ \exp \left( \frac{q-1}{2}(\|\psi(v)\|^2 - \|v\|^2) \right) | \det \nabla \psi(v)|^{-(q-1)}   \right]\\
    &\leq \frac{1}{q-1} \log \E_{v \sim \gammad} \left[ \exp \left( \frac{q-1}{2}(2 \mfm_1 \|v\| + \mfm_1^2) \right) (1-\mfm_2)^{-d(q-1)}   \right]\\
    &= \frac{1}{q-1} \left[ -d(q-1) \log (1-\mfm_2) + \frac{q-1}{2} \mfm_1^2 + \log \E_{v \sim \gammad} \left[e^{(q-1) \mfm_1\|v\|} \right] \right].
\end{align*}
Applying Lemma~\ref{lem:Helper_Chi_d} with $c=(q-1) \mfm_1$, we get
\begin{align*}
    \sfR_q(\psi_\#\gammad \dvert \gammad) &\leq d \log \frac{1}{1-\mfm_2} + \frac{\mfm_1^2}{2} + \sqrt{d}\, \mfm_1  + \frac{(q-1)\mfm_1^2}{2}\\
    &= d \log \frac{1}{1-\mfm_2} + \sqrt{d} \, \mfm_1  + \frac{q\mfm_1^2}{2}
\end{align*}
Using $\log \frac{1}{1-x} \leq \frac{x}{1-x}$ for $x \in [0, 1)$ we get
\begin{align*}
    \sfR_q(\psi_\#\gammad \dvert \gammad) &\leq \frac{d \mfm_2}{1-\mfm_2} + \sqrt{d}\,\mfm_1  + \frac{q\mfm_1^2}{2}\,.
\end{align*}

\end{proof}

\subsection{Proof of Lemma~\ref{lem:Renyi_Mixing_OneStepRegularity}}\label{app:PfOf_Renyi_Mixing_OneStepRegularity}

\begin{proof}
    Applying Lemmas~\ref{lem:DiracToPerturbed_Mixing} and~\ref{lem:Renyi_Perturbed_Mixing} with $\psi = \varphi_{x,y}$~\eqref{eq:MixingMap} and bounds $\mfm_1$ and $\mfm_2$ given by Lemmas~\ref{lem:MixingMapPtWise} and~\ref{lem:MixingMapJacobian} respectively, gives us
    \[
    \sfR_q(\delta_x \tilde \bP_{T,h} \dvert \delta_y \tilde \bP_{T,h}) \leq \frac{99}{14}d M T^2 \|x-y\| + \frac{3\sqrt{d}}{2T}\|x-y\| + \frac{9q}{8T^2}\|x-y\|^2
    \]
    Simplifying the constants gives us an upper bound of
    \[
    \sfR_q(\delta_x \tilde \bP_{T,h} \dvert \delta_y \tilde \bP_{T,h}) \leq \left( 8d MT^2 + \frac{3\sqrt{d}}{2T} \right)\|x-y\| + \frac{9q}{8T^2}\|x-y\|^2\,.
    \]
\end{proof}

\subsection{Proof of Lemma~\ref{lem:Renyi_Mixing_OneStepGeneral}}\label{app:PfOf_Renyi_Mixing_OneStepGeneral}

\begin{proof}
    Direct application of Lemma~\ref{lem:RenyiConvexityOneStep} with Lemma~\ref{lem:Renyi_Mixing_OneStepRegularity} proves the result.
\end{proof}

\subsection{Proof of Theorem~\ref{thm:RenyiMixing}}\label{app:PfOf_ThmRenyiMixing}

\begin{proof}
    Let $\tilde \bP$ and $\tilde \nu$ be shorthand for $\tilde \bP_{T,h}$ and $\tilde \nu_h$ respectively.
    We begin by applying Lemma~\ref{lem:Renyi_Mixing_OneStepGeneral} to obtain
    \[
    \sfR_q(\mu \tilde \bP^{k+1} \dvert \tilde \nu) \leq \inf_{\gamma \in \Gamma(\mu \tilde \bP^{k}, \tilde \nu)} \frac{1}{q-1} \log \E_{(x,y) \sim \gamma} \left[ e^{(q-1)\left[ \left( 8d MT^2 + \frac{3 \sqrt{d}}{2T} \right)\|x-y\| + \frac{9q}{8T^2}\|x-y\|^2   \right]} \right]\,.
    \]
    Let $\delta_1 \coloneqq (q-1)\left( 8d MT^2 + \frac{3\sqrt{d}}{2T} \right)$ and $\delta_2 \coloneqq \frac{9q(q-1)}{8T^2}$. Then taking the infimum inside the logarithm, we obtain
    \[
    \sfR_q(\mu \tilde \bP^{k+1} \dvert \tilde \nu) \leq \frac{1}{q-1} \log \inf_{\gamma \in \Gamma(\mu \tilde \bP^{k}, \tilde \nu)} \E_{(x,y) \sim \gamma} \left[ e^{\delta_1 \|x-y\| + \delta_2 \|x-y\|^2}  \right]\,.
    \]
    For any $\beta > 0$, it holds that $\delta_1\|x-y\| \leq \frac{\beta^2 \delta_1^2}{2} + \frac{\|x-y\|^2}{2 \beta^2}$. We will choose $\beta$ (independent of $x,y$) later. Applying this, we obtain
    \begin{align*}
    \sfR_q(\mu \tilde \bP^{k+1} \dvert \tilde \nu) &\leq \frac{1}{q-1} \log \inf_{\gamma \in \Gamma(\mu \tilde \bP^{k}, \tilde \nu)} \E_{(x,y) \sim \gamma} \left[ e^{\frac{\beta^2 \delta_1^2}{2} + \frac{\|x-y\|^2}{2 \beta^2} + \delta_2 \|x-y\|^2}  \right]\\
    &= \frac{1}{q-1} \log \inf_{\gamma \in \Gamma(\mu \tilde \bP^{k}, \tilde \nu)} \E_{(x,y) \sim \gamma} \left[ e^{\frac{\beta^2 \delta_1^2}{2} + (\delta_2 + \frac{1}{2\beta^2})\|x-y\|^2}  \right]\,.
    \end{align*}
    Breaking up the exponential and the logarithm, we get
    \[
    (q-1)\sfR_q(\mu \tilde \bP^{k+1} \dvert \tilde \nu) \leq \frac{\beta^2 \delta_1^2}{2} +   \log \inf_{\gamma \in \Gamma(\mu \tilde \bP^{k}, \tilde \nu)} \E_{(x,y) \sim \gamma} \left[ e^{(\delta_2 + \frac{1}{2\beta^2})\|x-y\|^2}  \right]\,.
    \]
    Under Assumption~\ref{assumption:OWMixing} for $\tilde \bQ = \tilde \bP_{T,h}$, we know that
    \[
    W_{\psi}(\mu \tilde \bP^k, \tilde \nu) \leq c_1 e^{-c_2 k}\, W_{\psi} (\mu, \tilde \nu)\,.
    \]
    Due to Definition~\ref{def:OrliczWassersteinDistance}, this means that there exists $\gamma^* \in \Gamma (\mu \tilde \bP^k, \tilde \nu)$ and some $\lambda^* \leq c_1 e^{-c_2 k}\, W_{\psi} (\mu, \tilde \nu)$ such that 
    \[
\E_{(x,y) \sim \gamma^*} \left[ e^{\frac{\|x-y\|^2}{{\lambda^{*}}^2}} \right] \leq 2\,.
    \]
    Furthermore, using Lemma~\ref{lem:Helper_Orlicz}, we know that for any $a \in [0, \frac{1}{{\lambda^*}^2}]$
    \[
    \E_{(x,y) \sim \gamma^*} \left[ e^{a\|x-y\|^2} \right] \leq 2^{a {\lambda^*}^2}\,.
    \]
    Hence, as long as $\delta_2 + \frac{1}{2\beta^2} \leq \frac{1}{{\lambda^*}^2}$, we can write
    \begin{align*}
       (q-1) \sfR_q(\mu \tilde \bP^{k+1} \dvert \tilde \nu) &\leq \frac{\beta^2 \delta_1^2}{2} + \log \E_{(x,y) \sim \gamma^*} \left[ e^{(\delta_2 + \frac{1}{2\beta^2})\|x-y\|^2}  \right]\\
        &\leq \frac{\beta^2 \delta_1^2}{2} + \log 2^{(\delta_2 + \frac{1}{2\beta^2}){\lambda^*}^2}\\
        &= \frac{\beta^2 \delta_1^2}{2} + \left(\delta_2 + \frac{1}{2\beta^2}\right){\lambda^*}^2 \log 2 \\
        &\leq \delta_2 (\log 2) c_1^2 e^{-2c_2 k}\, W_{\psi}^2 (\mu, \tilde \nu) + \frac{\beta^2 \delta_1^2}{2} + c_1^2 e^{-2c_2 k}\, W_{\psi}^2 (\mu, \tilde \nu)\frac{\log 2}{2 \beta^2}\,.
    \end{align*}
    To minimize the upper bound, we pick 
    \[
    \beta^2 = \frac{c_1 e^{-c_2 k}\, W_{\psi}(\mu, \tilde \nu) \sqrt{\log 2}}{\delta_1}\,.
    \]
    Plugging this in, our upper bound is
    \[
    (q-1)\sfR_q(\mu \tilde \bP^{k+1} \dvert \tilde \nu) \leq \delta_2 (\log 2) c_1^2 e^{-2c_2 k}\, W_{\psi}^2 (\mu, \tilde \nu) + \delta_1 \sqrt{\log 2}\, c_1 e^{-c_2 k}\, W_{\psi}(\mu, \tilde \nu)\,.
    \]
    Basic simplifications lead to the upper bound
    \[
    (q-1)\sfR_q(\mu \tilde \bP^{k+1} \dvert \tilde \nu) \leq (\delta_1 + \delta_2) \sqrt{\log 2}\, c_1^2 e^{-c_2 k}\, \max\{1, W_{\psi}^2 (\mu, \tilde \nu) \}\,.
    \]
    However this holds provided 
    \[
    \delta_2 + \frac{\delta_1 e^{c_2 k}}{2 c_1 W_{\psi}(\mu, \tilde \nu) \sqrt{\log 2}} \leq \frac{e^{2 c_2 k}}{c_1^2 W_{\psi}^2(\mu, \tilde \nu)}
    \]
    where we have plugged in the optimal value of $\beta^2$ and also used that $\lambda^* \leq c_1 e^{-c_2 k}\, W_{\psi} (\mu, \tilde \nu)$.
    This holds so long as
    \[
    k \geq \frac{1}{c_2} \log \left( \frac{c_1 W_{\psi}(\mu, \tilde \nu)}{4 \sqrt{\log 2}} \left(\delta_1 + \sqrt{\delta_1^2 + 16 (\log 2) \delta_2}\right) \right)\,.
    \]

\end{proof}

\section{Asymptotic Bias in Rényi Divergence}

\subsection{Proof of Lemma~\ref{lem:Renyi_Target_OneStepRegularity}}\label{app:PfOf_Renyi_Target_OneStepRegularity}

\begin{proof}
    We use $\Phi$ as shorthand for $\Phi_{x,y}$.
    We start with a change of variable step similar to the proof of Lemma~\ref{lem:Renyi_Perturbed_Mixing}, i.e., we have that
    \[
    \sfR_q (\Phi_\# \gammad \dvert \gammad) = \frac{1}{q-1} \log \E_{v \sim \gammad} \left[ \exp \left( \frac{q-1}{2}(\|\Phi(v)\|^2 - \|v\|^2) \right) | \det \nabla \Phi(v)|^{-(q-1)}   \right].
    \]
    Lemma~\ref{lem:TargetMap-First-PtWise} states that
    \[
    \|\Phi(v)-v\| \leq \sfp_{xy}\|x-y\| + h\sfp_v\|v\| + h\sfp_x\|x\|
    \]
    where $\sfp_{xy} \coloneqq \frac{3}{2T}$, $\sfp_v \coloneqq \frac{7}{25T}$, and $\sfp_x \coloneqq \frac{49}{180}L$\,.
    Lemma~\ref{lem:TargetMap-First-Jacobian} states that
    \[
    \|\nabla \Phi(v)-I\|_\op \leq \min \left\{ \frac{15}{18}, \sfj_{xy}\|x-y\| + h\sfj_c + h\sfj_v\|v\| + h\sfj_x\|x\| \right\}
    \]
    where $\sfj_{xy} \coloneqq \frac{11}{2}MT^2$, $\sfj_c \coloneqq \frac{44}{135T}$, $\sfj_v \coloneqq \frac{440}{135}MT^2$, and $\sfj_x \coloneqq \frac{352}{675}MT$\,.

    We have that
    \begin{align*}
        \|\Phi(v)\|^2 - \|v\|^2 &\leq 2 \|v\| \|\Phi(v)-v\| + \|\Phi(v)-v\|^2 \\
        &\leq 2\|v\| (\sfp_{xy}\|x-y\| + h\sfp_v\|v\| + h\sfp_x\|x\|) + (\sfp_{xy}\|x-y\| + h\sfp_v\|v\| + h\sfp_x\|x\|)^2 \\
        &=  h \sfp_v\|v\|^2 (h \sfp_v + 2) + 2(h \sfp_v + 1)\|v\|(\sfp_{xy}\|x-y\| + h \sfp_x\|x\|) + (\sfp_{xy}\|x-y\| + h \sfp_x\|x\|)^2
    \end{align*}

    To bound the Jacobian, note that for any matrix $A$ with $\|A-I\|_\op < 1$, $|\det A| \geq (1-\|A-I\|_\op)^d$. Hence
    \begin{align*}
    |\det\nabla \Phi(v)|^{-(q-1)} &\leq (1-\|\nabla \Phi(v)-I\|_\op)^{-(q-1)d}\\
    &\leq \exp{ \left( (q-1)d \frac{\|\nabla \Phi(v)-I\|_\op}{1-\|\nabla \Phi(v)-I\|_\op}\right)}\\
    &\leq \exp{\left( 6(q-1)d \|\nabla \Phi(v)-I\|_\op \right)}\\
    &\leq \exp{\left( 6(q-1)d (\sfj_{xy}\|x-y\| + h\sfj_c + h\sfj_v\|v\| + h\sfj_x\|x\|) \right)}
    \end{align*}

    Putting it all together we get that
    \[
    \exp \left( \frac{q-1}{2}(\|\Phi(v)\|^2 - \|v\|^2) \right) | \det \nabla \Phi(v)|^{-(q-1)}    \leq \exp{ \left( (q-1)(\sfc_0 + \sfc_1\|v\| + \sfc_2\|v\|^2) \right)}
    \]
    where
    \begin{align*}
        \sfc_0 &\coloneqq \frac{(\sfp_{xy}\|x-y\| + h \sfp_x \|x\|)^2}{2} + 6d\sfj_{xy}\|x-y\| + 6d h \sfj_c + 6d h \sfj_x\|x\|\\
        \sfc_1 &\coloneqq (h \sfp_v +1)(\sfp_{xy}\|x-y\| + h \sfp_x \|x\|) + 6dh \sfj_v\\
        \sfc_2 &\coloneqq \frac{h \sfp_v(h \sfp_v +2)}{2}\,.
    \end{align*}
    Let $s \coloneqq q-1$. Hence we have that
    \begin{align*}
        \E_{v \sim \gammad} \left[ e^{ \left( s(\sfc_0 + \sfc_1\|v\| + \sfc_2\|v\|^2) \right)} \right] &= e^{s \sfc_0} \E_{v \sim \gammad} \left[ e^{ \left( s(\sfc_1\|v\| + \sfc_2\|v\|^2) \right)} \right]\\
        &\leq e^{s \sfc_0} \left( \E_{v \sim \gammad} \left[ e^{2s\sfc_1\|v\|} \right] \right)^{\frac{1}{2}} \left( \E_{v \sim \gammad} \left[ e^{2s\sfc_2\|v\|^2} \right] \right)^{\frac{1}{2}}\\
        &\leq e^{s \sfc_0} e^{s \sfc_1 \sqrt{d}+ s^2\sfc_1^2} (1-4s\sfc_2)^{-\frac{d}{4}}
    \end{align*}
    where the last inequality is by Lemma~\ref{lem:Helper_Chi_d} and properties of Gaussian distribution and holds whenever 
    \[
    s\sfc_2 < \frac{1}{4}\,.
    \]
    Taking log and dividing by $s$ we get
    \[
    \sfR_q (\Phi_\# \gammad \dvert \gammad) \leq \sfc_0 + \sqrt{d}\,\sfc_1  + s \sfc_1^2 - \frac{d}{4s}\log(1-4s\sfc_2)\,.
    \]
    Using $-\log{(1-x)} \leq \frac{x}{1-x} \leq 2x$ for $x \in [0, \frac{1}{2}]$, we get
    \[
    \sfR_q (\Phi_\# \gammad \dvert \gammad) \leq \sfc_0 + \sqrt{d}\,\sfc_1 + s \sfc_1^2 + 2d \sfc_2
    \]
    which holds so long as 
    \[
    s\sfc_2 \leq \frac{1}{8}\,.
    \]
    The condition that $s\sfc_2 \leq \frac{1}{8}$ is satisfied when
    \begin{align*}
    0 < h \leq \frac{1}{\sfp_v} \left(\left(1 + \frac{1}{4s}\right)^{\frac{1}{2}}-1 \right) = \frac{25T}{7} \left(\left(1 + \frac{1}{4s}\right)^{\frac{1}{2}}-1 \right).
    \end{align*}
    To simplify the upper bound
    \[
    \sfR_q (\Phi_\# \gammad \dvert \gammad) \leq \sfc_0 + \sqrt{d}\,\sfc_1  + s \sfc_1^2 + 2d \sfc_2
    \]
    note that
    \begin{align*}
    \sfc_0 &\leq \sfp_{xy}^2\|x-y\|^2 + 6d\sfj_{xy}\|x-y\| + h^2 \sfp_x^2\|x\|^2 + 6dh \sfj_x\|x\| + 6dh \sfj_c\\[0.7em]
    \sqrt{d}\,\sfc_1 &= \sqrt{d}(h \sfp_v +1)\sfp_{xy}\|x-y\| + \sqrt{d}(h \sfp_v +1)h \sfp_x \|x\| + 6d^{3/2}h \sfj_v\\[0.7em]
    s \sfc_1^2 &\leq 3s(h \sfp_v +1)^2\sfp_{xy}^2\|x-y\|^2 + 3s(h \sfp_v +1)^2h^2 \sfp_x^2 \|x\|^2 + 108sd^2h^2 \sfj_v^2 \\[0.7em]
    2d\sfc_2 &\leq dh^2 \sfp_v^2 + 2dh \sfp_v\,.
    \end{align*}
    Therefore, grouping like terms together we get that
    \begin{align*}
        \sfR_q (\Phi_\# \gammad \dvert \gammad) &\leq \sfp_{xy}^2\|x-y\|^2(1+3s(h\sfp_v+1)^2) + \sqrt{d}\,\|x-y\|(6\sqrt{d}\,\sfj_{xy} + \sfp_{xy}(h \sfp_v+1))\\
        &+ h^2\sfp_x^2\|x\|^2(1+3s(h\sfp_v+1)^2) + h\sqrt{d}\, \|x\|(6\sqrt{d}\, \sfj_x + \sfp_x(h \sfp_v +1))\\
        &+dh(6 \sfj_c + 6\sqrt{d}\, \sfj_v + 108sdh \sfj_v^2 + h \sfp_v^2 +2 \sfp_v)\,.
    \end{align*}
    
\end{proof}

\subsection{Proof of Lemma~\ref{lem:Renyi_Target_OneStepGeneral}}\label{app:PfOf_Renyi_Target_OneStepGeneral}

\begin{proof}
    The proof follows by applying Lemma~\ref{lem:RenyiConvexityOneStep} with $\bP = \tilde \bP_{T,h}$ and $\bQ = \bP_T$ along with Lemma~\ref{lem:Renyi_Target_OneStepRegularity}, which states that
    \[
    \sfR_q(\mu \tilde \bP_{T,h} \dvert \pi \bP_T) \leq \inf_{\gamma \in \Gamma(\mu, \pi)} \frac{1}{q-1} \log \E_{(x,y) \sim \gamma} \left[ e^{(q-1)(\alpha_0 + \alpha_1 \|x-y\| + \alpha_2 \|x-y\|^2 + \alpha_3 \|x\| + \alpha_4 \|x\|^2)} \right]
    \]
    with $\alpha_0, \alpha_2, \alpha_3, \alpha_4$ defined as the coefficients of the result from Lemma~\ref{lem:Renyi_Target_OneStepRegularity}.

\end{proof}

\subsection{Proof of Theorem~\ref{thm:RenyiTarget}}\label{app:PfOf_ThmRenyiTarget}

We first state the following corollary, the proof for which closely follows that of Theorem~\ref{thm:RenyiMixing} and is therefore omitted.
The corollary is helpful in the proof of Theorem~\ref{thm:RenyiTarget}.

\begin{corollary}\label{cor:RenyiMixingGeneral}
    Consider the same conditions as Theorem~\ref{thm:RenyiMixing} and suppose Assumption~\ref{assumption:OWMixing} holds for a general Markov kernel $\tilde \bQ$ with stationary distribution $\tilde \nu_{\tilde \bQ}$. Further define $\delta_1$ and $\delta_2$ as in Theorem~\ref{thm:RenyiMixing}. Then for any integer 
    \[
    k \geq \frac{1}{c_2} \log \left( \frac{c_1 W_{\psi}(\mu, \tilde \nu_{\tilde \bQ})}{4 \sqrt{\log 2}} \left(\delta_1 + \sqrt{\delta_1^2 + 16 (\log 2) \delta_2}\right) \right)
    \]
    it holds that
    \[
    \sfR_q(\mu \tilde \bQ^k \bP_{T,h}\dvert \tilde \nu_{\tilde \bQ} \tilde \bP_{T,h}) \leq  \frac{\delta_1 + \delta_2}{q-1} \sqrt{\log 2}\, c_1^2 e^{-c_2 k}\, \max\{1, W_{\psi}^2 (\mu, \tilde \nu_{\tilde \bQ}) \}\,.
    \]
\end{corollary}

\begin{proof}[Proof of Theorem~\ref{thm:RenyiTarget}]
Let $\tilde \bP$ be shorthand for $\tilde \bP_{T,h}$.
Applying the weak triangle inequality for Rényi divergence~\eqref{eq:RenyiWTI} together with the monotonicity in the order~\eqref{eq:RenyiMonotonicity}, we have
\[
\sfR_q(\mu \tilde \bQ^k \tilde \bP \dvert \nu) \leq \frac{3}{2}\sfR_{2q}(\mu \tilde \bQ^k \tilde \bP \dvert \tilde \nu_{\tilde \bQ} \tilde \bP) + \sfR_{2q}(\tilde \nu_{\tilde \bQ} \tilde \bP \dvert \nu)\,.
\]

The first term is bounded directly by Corollary~\ref{cor:RenyiMixingGeneral} and this leads to the first term in the statement of Theorem~\ref{thm:RenyiTarget}. To bound the second term, Lemma~\ref{lem:Renyi_Target_OneStepGeneral} implies
\[
\sfR_{2q}(\tilde \nu_{\tilde \bQ} \tilde \bP \dvert \nu) \leq \inf_{\gamma \in \Gamma(\tilde \nu_{\tilde \bQ},  \nu)} \frac{1}{2q-1} \log \E_{(x,y) \sim \gamma} \left[ e^{(2q-1)(\alpha_0 + \alpha_1 \|x-y\| + \alpha_2 \|x-y\|^2 + \alpha_3 \|x\| + \alpha_4 \|x\|^2)} \right]
\]
where 
\begin{equation}
\label{eq:alpha-defs}
\begin{aligned}
        \alpha_0 &\coloneqq dh(6 \sfj_c + 6\sqrt{d}\, \sfj_v + 108(2q-1)dh \sfj_v^2 + h \sfp_v^2 +2 \sfp_v)\\
        \alpha_1 &\coloneqq \sqrt{d}\,(6\sqrt{d}\,\sfj_{xy} + \sfp_{xy}(h \sfp_v+1))\\
        \alpha_2 &\coloneqq \sfp_{xy}^2(1+3(2q-1)(h\sfp_v+1)^2)\\
        \alpha_3 &\coloneqq h\sqrt{d}\, (6\sqrt{d}\, \sfj_x + \sfp_x(h \sfp_v +1))\\
        \alpha_4 &\coloneqq h^2\sfp_x^2(1+3(2q-1)(h\sfp_v+1)^2)\,.
\end{aligned}
\end{equation}

    Separating out the constant term in the exponent, we get
    \[
    (2q-1)\sfR_{2q} (\tilde \nu_{\tilde \bQ} \tilde \bP \dvert \nu) \leq (2q-1)\alpha_0 + \log \inf_{\gamma \in \Gamma(\tilde \nu_{\tilde \bQ},  \nu)} \E_{(x,y) \sim \gamma} \left[ e^{(2q-1)(\alpha_1 \|x-y\| + \alpha_2 \|x-y\|^2 + \alpha_3 \|x\| + \alpha_4 \|x\|^2)} \right]\,.
    \]
    Applying Young's inequality on the terms $\alpha_1 \|x-y\|$ and $\alpha_3 \|x\|$ we get
    \[
    \alpha_1 \|x-y\| \leq \frac{\beta_1^2 \alpha_1^2}{2} + \frac{\|x-y\|^2}{2 \beta_1^2}
    \]
    for any $\beta_1 > 0$ which we will choose later, and
    \[
    \alpha_3 \|x\| \leq \frac{\beta_3^2 \alpha_3^2}{2} + \frac{\|x\|^2}{2\beta_3^2}
    \]
    for any $\beta_3 > 0$ which we will also choose later.
    Define the following
    \begin{align*}
        \eta_0 &\coloneqq \frac{\beta_1^2 \alpha_1^2}{2} + \frac{\beta_3^2 \alpha_3^2}{2}\\
        \eta_1 &\coloneqq \alpha_2 + \frac{1}{2\beta_1^2}\\
        \eta_2 &\coloneqq \alpha_4 + \frac{1}{2\beta_3^2}
    \end{align*}
    to get that
    \begin{equation}\label{eq:RenyiTargetPf-1}
    (2q-1)\sfR_{2q} (\tilde \nu_{\tilde \bQ} \tilde \bP \dvert \nu) \leq (2q-1)(\alpha_0 + \eta_0) + \log \inf_{\gamma \in \Gamma(\tilde \nu_{\tilde \bQ},  \nu)} \E_{(x,y) \sim \gamma} \left[ e^{(2q-1)(\eta_1 \|x-y\|^2 + \eta_2 \|x\|^2)} \right]\,.
    \end{equation}
    Using $\eta_2 \|x\|^2 \leq 2\eta_2\|x-y\|^2 + 2\eta_2 \|y\|^2$, we get that
    \begin{equation}\label{eq:RenyiTargetPf-2}
    \E_{(x,y) \sim \gamma} \left[ e^{(2q-1)(\eta_1 \|x-y\|^2 + \eta_2 \|x\|^2)} \right] \leq \E_{(x,y) \sim \gamma} \left[ e^{(2q-1)(\eta_1 + 2\eta_2) \|x-y\|^2} e^{ 2(2q-1)\eta_2 \|y\|^2)} \right]\,.
    \end{equation}

    Using Hölder's inequality we get that
    \begin{equation}\label{eq:RenyiTargetPf-3}
    \E_{(x,y) \sim \gamma} \left[ e^{(2q-1)(\eta_1 + 2\eta_2) \|x-y\|^2} e^{ 2(2q-1)\eta_2 \|y\|^2)} \right] \leq \underbrace{\left( \E_{(x,y) \sim \gamma} \left[ e^{2(2q-1)(\eta_1 + 2\eta_2) \|x-y\|^2}  \right] \right)^{1/2}}_{\rn{1}} \underbrace{\left( \E_{y \sim \nu} \left[ e^{ 4(2q-1)\eta_2 \|y\|^2)}  \right] \right)^{1/2}}_{\rn{2}}\,.
    \end{equation}

    To handle $\rn{1}$, note that Assumption~\ref{assumption:OWBias} tells us
    \[
    W_{\psi}(\tilde \nu_{\tilde \bQ}, \nu) \leq  \Delta_h\,.
    \]

    By Definition~\ref{def:OrliczWassersteinDistance}, this means that there exists a coupling $ \gamma^* \in \Gamma(\tilde \nu_{\tilde \bQ}, \nu)$ and some $\lambda^* \leq  \Delta_h$ such that
    \[
\E_{(x,y) \sim \gamma^*} \left[ e^{\frac{\|x-y\|^2}{{\lambda^{*}}^2}} \right] \leq 2\,.
    \]

    Furthermore, using Lemma~\ref{lem:Helper_Orlicz}, we know that for any $a \in [0, \frac{1}{{\lambda^*}^2}]$
    \[
    \E_{(x,y) \sim \gamma^*} \left[ e^{a\|x-y\|^2} \right] \leq 2^{a {\lambda^*}^2}\,.
    \]

    Hence, as long as
    \begin{equation}\label{eq:Constraint-RenyiTarget-1}
    2(2q-1)(\eta_1 + 2\eta_2) \leq \frac{1}{{\lambda^*}^2}
    \end{equation}
    we get that 
    \begin{equation}\label{eq:RenyiTargetPf-4}
    \rn{1} \leq 2^{(2q-1)(\eta_1 + 2\eta_2){\lambda^*}^2}\,.
    \end{equation}
    We verify~\eqref{eq:Constraint-RenyiTarget-1} at the end of the proof.

    To handle $\rn{2}$, recall that $\nu$ has finite $\psi$--norm, i.e., for $X \sim \nu$, $\|X\|_{\psi} \leq K_\nu$.
    Hence, using Lemma~\ref{lem:Helper_Orlicz}, as long as 
    \begin{equation}\label{eq:Constraint-RenyiTarget-2}
    4(2q-1) \eta_2 \leq \frac{1}{K_\nu^2}\,,
    \end{equation}
    we can conclude that
    \begin{equation}\label{eq:RenyiTargetPf-5}
    \rn{2} \leq 2^{2(2q-1)\eta_2 K_\nu^2}\,.
    \end{equation}
    We verify~\eqref{eq:Constraint-RenyiTarget-2} at the end of the proof.
    
    Therefore,~\eqref{eq:RenyiTargetPf-1},~\eqref{eq:RenyiTargetPf-2},~\eqref{eq:RenyiTargetPf-3},~\eqref{eq:RenyiTargetPf-4}, and~\eqref{eq:RenyiTargetPf-5} yield
    \[
    (2q-1)\sfR_{2q} (\tilde \nu_{\tilde \bQ}\tilde \bP \dvert \nu) \leq (2q-1)(\alpha_0 + \eta_0) + (2q-1)(\eta_1 + 2\eta_2){\lambda^*}^2 \log 2 + 2(2q-1)\eta_2 K_\nu^2 \log 2\,.
    \]
    Dividing by $2q-1$ throughout gives us that
    \[
    \sfR_{2q} (\tilde \nu_{\tilde \bQ}\tilde \bP \dvert \nu) \leq \alpha_0 + \eta_0 + (\eta_1 + 2\eta_2){\lambda^*}^2 \log 2 + 2\eta_2 K_\nu^2 \log 2\,.
    \]
The optimal values of $\beta_1$ and $\beta_3$ are therefore $\beta_1^*$ and $\beta_3^*$ given by
    \[
    {\beta_1^*}^2 = \frac{\lambda^* \sqrt{\log 2}}{\alpha_1}\,\,,\,\, {\beta_3^*}^2 = \frac{\sqrt{2({\lambda^*}^2 + K_\nu^2) \log 2}}{\alpha_3}\,.
    \]
    At these optimal choices
    \[
    \eta_0 = \frac{\alpha_1 \lambda^* \sqrt{\log 2}}{2} + \frac{\alpha_3\sqrt{2({\lambda^*}^2 + K_\nu^2) \log 2}}{2}
    \]
    and
    \[
    \eta_1 + 2\eta_2 = \alpha_2 + \frac{\alpha_1}{2 \lambda^* \sqrt{\log 2}} + 2 \alpha_4 + \frac{\alpha_3}{\sqrt{2({\lambda^*}^2 + K_\nu^2) \log 2}}
    \]
    and
    \[
    \eta_2 = \alpha_4  + \frac{\alpha_3}{2\sqrt{2({\lambda^*}^2 + K_\nu^2) \log 2}}\,. 
    \]
    Hence, simplifying constants and using $\lambda^*, K_\nu \leq \sqrt{{\lambda^*}^2 + K_\nu^2}$, we get
    \begin{align*}
        \sfR_{2q} (\tilde \nu_{\tilde \bQ}\tilde \bP \dvert \nu) \leq \alpha_0 + 2 \alpha_1 \lambda^* + \alpha_2 {\lambda^*}^2 + 2 \alpha_3\Big(\lambda^* + K_\nu + \sqrt{{\lambda^*}^2 + K_\nu^2}\, \Big) + 2 \alpha_4({\lambda^*}^2 + K_\nu^2)\,.
    \end{align*}
    Upper bounding $\lambda^* \leq \Delta_h$ gives us
    \[
    \sfR_{2q} (\tilde \nu_{\tilde \bQ}\tilde \bP \dvert \nu) \leq \alpha_0 + 2 \alpha_1 \Delta_h + \alpha_2 \Delta_h^2 + 2 \alpha_3\Big(\Delta_h + K_\nu + \sqrt{\Delta_h^2 + K_\nu^2}\, \Big) + 2 \alpha_4(\Delta_h^2 + K_\nu^2)\,.
    \]
    To simplify further, 
    note that the step-size constraint in Lemma~\ref{lem:Renyi_Target_OneStepGeneral} can be written as  
    \[
    h\sfp_v +1 \leq \left( 1 + \frac{1}{4(2q-1)} \right)^{\frac{1}{2}} \eqqcolon u
    \]
    and further define $s \coloneqq 2q-1$. 
    This provides the following bounds
\begin{equation}
\label{eq:alpha-upperbd}
\begin{aligned}
        \alpha_1 &\leq \sqrt{d}\,(6\sqrt{d}\,\sfj_{xy} + \sfp_{xy}u)\\
        \alpha_2 &\leq \sfp_{xy}^2(1+3su^2)\\
        \alpha_3 &\leq h\sqrt{d}\, (6\sqrt{d}\, \sfj_x + \sfp_xu)\\
        \alpha_4 &\leq h^2\sfp_x^2(1+3su^2)\,.
\end{aligned}
\end{equation}

    It remains to check what conditions are required for the constraints to hold.

\paragraph{\underline{Ensuring~\eqref{eq:Constraint-RenyiTarget-1}, i.e., $2(2q-1)(\eta_1 + 2\eta_2) \leq \frac{1}{{\lambda^*}^2}$}}

To ensure this criterion, we require
\[
2(2q-1)\left(\alpha_2 + 2 \alpha_4 +\frac{\alpha_1}{2 \lambda^* \sqrt{\log 2}} + \frac{\alpha_3}{\sqrt{2({\lambda^*}^2 + K_\nu^2) \log 2}} \right) \leq \frac{1}{{\lambda^*}^2}
\]
Using $\lambda^* \leq \sqrt{{\lambda^*}^2 + K_\nu^2}$ and that $\lambda^* \leq \Delta_h$, a sufficient condition is that
\[
2(2q-1)\left[ \alpha_2 \Delta_h^2 + 2 \alpha_4 \Delta_h^2 + \frac{\alpha_1}{2\sqrt{\log2}}\Delta_h + \frac{\alpha_3}{\sqrt{2 \log 2}}\Delta_h \right] \leq 1
\]
which can be rewritten as
\[
(\alpha_2 + 2 \alpha_4) \Delta_h^2 + \frac{\Delta_h}{\sqrt{2 \log 2}} \Big( \frac{\alpha_1}{\sqrt{2}} + \alpha_3 \Big) - \frac{1}{2(2q-1)} \leq 0\,.
\]
This is a quadratic in $\Delta_h$ and using Lemma~\ref{lem:QuadraticFact} we get that
\[
\Delta_h \leq \min \left\{ \frac{\sqrt{2 \log 2}}{4(2q-1)(\frac{\alpha_1}{\sqrt{2}}+\alpha_3)} , \frac{1}{\sqrt{8(2q-1)(\alpha_2 + 2 \alpha_4)}}  \right\}\,.
\]
A sufficient condition for this to hold is when we substitute the upper bounds from~\eqref{eq:alpha-upperbd}. Substituting these values and using $s \coloneqq 2q-1$ gives
\[
\Delta_h \leq \min \left\{ \frac{\sqrt{2 \log 2}}{4s(\frac{\sqrt{d}\,(6\sqrt{d}\,\sfj_{xy} + \sfp_{xy}u)}{\sqrt{2}}+h\sqrt{d}\, (6\sqrt{d}\, \sfj_x + \sfp_xu))} , \frac{1}{\sqrt{8s(\sfp_{xy}^2(1+3su^2) + 2 h^2\sfp_x^2(1+3su^2))}}  \right\}\,.
\]

\paragraph{\underline{Ensuring~\eqref{eq:Constraint-RenyiTarget-2}, i.e., $4(2q-1) \eta_2 \leq \frac{1}{K_\nu^2}$}}

To ensure this criterion, we require
\[
4(2q-1) \left( \alpha_4  + \frac{\alpha_3}{2\sqrt{2({\lambda^*}^2 + K_\nu^2) \log 2}} \right) \leq \frac{1}{K_\nu^2}
\]
which is the same as
\[
\frac{2\alpha_3(2q-1)K_\nu^2   }{ \sqrt{2 \log 2} (1-4\alpha_4(2q-1)K_\nu^2)    } \leq \sqrt{{\lambda^*}^2 + K_\nu^2}
\]
which is ensured by
\[
\frac{2\alpha_3(2q-1)K_\nu^2   }{ \sqrt{2 \log 2} (1-4\alpha_4(2q-1)K_\nu^2)    } \leq K_\nu\,.
\]
Rearranging this, we get that
\[
4 \sqrt{\log 2} \alpha_4(2q-1)K_\nu^2 + \sqrt{2}(2q-1)\alpha_3 K_\nu - \sqrt{\log 2} \leq 0\,.
\]
A sufficient condition for this to hold is if we substitute the upper bounds on $\alpha_3, \alpha_4$ from~\eqref{eq:alpha-upperbd}. Substituting these values and using $s \coloneqq 2q-1$ yields
\[
4 \sqrt{\log 2}\, h^2\sfp_x^2(1+3su^2)sK_\nu^2 + \sqrt{2}sh\sqrt{d}\, (6\sqrt{d}\, \sfj_x + \sfp_xu) K_\nu - \sqrt{\log 2} \leq 0
\]
which is a quadratic in $h$. Lemma~\ref{lem:QuadraticFact} then gives that
\[
h \leq \min \left\{ \frac{\sqrt{\log 2}}{2 \sqrt{2}s\sqrt{d}\, (6\sqrt{d}\, \sfj_x + \sfp_xu) K_\nu} , \frac{1}{4K_\nu\sfp_x \sqrt{(1+3su^2)s}} \right\}\,.
\]

\end{proof}

\begin{remark}[Proof strategy]
    The proof of Theorem~\ref{thm:RenyiTarget} also goes through by applying Lemma~\ref{lem:Renyi_Target_OneStepGeneral} to bound $\sfR_q(\mu \tilde \bQ^k \tilde \bP \dvert \nu)$ from the beginning. Instead, we apply the weak triangle inequality for Rényi divergence to make the contribution of the mixing time guarantee, which is via Corollary~\ref{cor:RenyiMixingGeneral}, explicit. This leads to the exponentially decaying term in Theorem~\ref{thm:RenyiTarget} closely resemble the exponential contraction in Theorem~\ref{thm:RenyiMixing}. 
\end{remark}

\section{Information Contraction}\label{app:InformationContraction}

\begin{proof}[Proof of Corollary~\ref{cor:InformationContraction-Verlet}]
    Let the joint distribution of $(X_0, X_k)$ be $(X_0, X_k) \sim \rho_{0,k}$ and the marginal distributions be $X_0 \sim \rho_0 = \mu$ and $X_k \sim \rho_k = \mu \tilde \bP_{T,h}^k$\,. Starting from Definition~\ref{def:MutualInformation}, we have
    \begin{align*}
        \MI(X_0; X_k) &= \KL(\rho_{0,k} \dvert \rho_0 \otimes \rho_k)\\
        &= \E_{x \sim \rho_0} \left[ \KL (\rho_{k \mid 0=x} \dvert \rho_k)  \right]\\
        &= \E_{x \sim \rho_0} \left[ \KL (\rho_{k \mid 0=x} \dvert \tilde \nu_h)  \right] - \KL(\rho_k \dvert \tilde \nu_h)\\
        &\leq \E_{x \sim \rho_0} \left[ \KL (\rho_{k \mid 0=x} \dvert \tilde \nu_h)  \right]\\
        &= \E_{x \sim \rho_0} \left[ \KL (\delta_x \tilde \bP_{T,h}^k \dvert \tilde \nu_h)  \right]\\
        &\leq e^{-\frac{\alpha T^2}{5}(k-1)}\Big(\frac{9}{4T^2} + 20dM^2T^4\Big) \E_{x \sim \rho_0} \E_{y \sim \tilde \nu_h} \|x-y\|^2\\
        &= e^{-\frac{\alpha T^2}{5}(k-1)}\Big(\frac{9}{4T^2} + 20dM^2T^4\Big) \E_{(X,Y) \sim \rho_0 \otimes \tilde \nu_h} \left[\|X-Y\|^2\right]
    \end{align*}
    where the second equality is by chain rule of KL divergence, the third equality can be verified by expanding both sides, the first inequality is by non-negativity of KL divergence, and the second inequality is by Theorem~\ref{thm:KLTarget} and Proposition~\ref{prop:W2mixing-Verlet}.

\end{proof}

\section{Convexity of KL Divergence and Rényi Divergence}

The following lemmas help upgrade regularization and cross-regularization guarantees from Dirac initializations to general initializations. Setting $\bP = \bQ$ helps obtain regularization guarantees (Lemmas~\ref{lem:KL_Mixing_OneStepGeneral} and~\ref{lem:Renyi_Mixing_OneStepGeneral}) which yield mixing time results, and $\bP \neq \bQ$ is used to derive cross-regularization guarantees (Lemmas~\ref{lem:KL_Target_OneStepGeneral} and~\ref{lem:Renyi_Target_OneStepGeneral}) which imply asymptotic bias bounds. The following lemmas for the case of $\bP = \bQ$ are mentioned in~\cite[Theorem~3.7]{altschuler2023shifted}.

\begin{lemma}\label{lem:KLConvexityOneStep}
    Let $\mu, \pi \in \P(\R^d)$. Then for any two Markov kernels $\bP$ and $\bQ$
    \[
    \KL(\mu \bP \dvert \pi \bQ) = \KL(\E_{x \sim \mu} [\delta_x \bP] \dvert \E_{y \sim \pi} [\delta_y \bQ]) \leq \inf_{\gamma \in \Gamma(\mu, \pi)} \E_{(x,y) \sim \gamma} \KL(\delta_x \bP \dvert \delta_y \bQ).
    \]
\end{lemma}

\begin{proof}
    The equality follows by definition of $\mu \bP$ and $\pi \bQ$. To see the inequality, fix any coupling $\gamma \in \Gamma(\mu, \pi)$. The inequality then holds for $\gamma$ due to the joint convexity of KL divergence. As it holds for any coupling, taking the infimum over couplings yields the result.
\end{proof}

\begin{lemma}\label{lem:RenyiConvexityOneStep}
    Let $\mu, \pi \in \P(\R^d)$. Then for $q > 1$ and any two Markov kernels $\bP$ and $\bQ$
    \[
    \sfR_q(\mu \bP \dvert \pi \bQ)  \leq \inf_{\gamma \in \Gamma(\mu, \pi)} \frac{1}{q-1} \log \E_{(x,y) \sim \gamma} \left[ e^{(q-1) \sfR_q (\delta_x \bP \dvert \delta_y \bQ)}  \right]\,.
    \]
\end{lemma}

\begin{proof}

    For distributions $\rho, \nu \in \P(\R^d)$, define
    \[
    F_q(\rho \dvert \nu) \coloneqq  \E_{\nu} \left[ \left( \frac{\rho}{\nu} \right)^q \right] = \int \rho(z)^q \nu(z)^{1-q} \dz\,,
    \]
    so that the Rényi divergence (Definition~\ref{def:RenyiDivergence}) is $\sfR_q(\rho \dvert \nu) = \frac{1}{q-1} \log F_q(\rho \dvert \nu)$\,.
    Note that the functional $F_q(\rho \dvert \nu)$ is jointly convex with respect to $(\rho, \nu)$ as the function $g(u,v) = u^q v^{1-q}$ is jointly convex for $u,v \geq 0$, $q>1$.

    Fix a coupling $\gamma \in \Gamma (\mu, \pi)$. Since $F_q$ is jointly convex, we have
    \[
    F_q(\mu \bP \dvert \pi \bQ) \leq \E_{(x,y) \sim \gamma}[F_q(\delta_x \bP \dvert \delta_q \bQ)]\,.
    \]
    
    From the definition of $F_q$, we get
    \[
    \exp((q-1) \sfR_q(\mu \bP \dvert \pi \bQ)) \leq \E_{(x,y) \sim \gamma} [\exp((q-1) \sfR_q(\delta_x \bP \dvert \delta_y \bQ))]\,.
    \]
    
    Taking logarithm on both side and dividing by $q-1$ yields
    \[
    \sfR_q(\mu \bP \dvert \pi \bQ)  \leq  \frac{1}{q-1} \log \E_{(x,y) \sim \gamma} \left[ e^{(q-1) \sfR_q (\delta_x \bP \dvert \delta_y \bQ)}  \right]\,.
    \]

    As the inequality holds for an arbitrary coupling $\gamma \in \Gamma(\mu, \pi)$, it must hold for the infimum over couplings, yielding the lemma.

\end{proof}

\section{Cross-regularization via uLA}\label{app:uLA-CrossRegularization}

The unadjusted Langevin Algorithm (uLA)~\cite{roberts1996exponential, roberts1998optimal, dalalyan2017further,  durmus2019high, VW23} to sample from target distribution $\nu \propto e^{-f}$ on $\R^d$, starting from $X_0 \sim \mu$, is defined by the following update 
\[
X_{k+1} = X_k - \eta \nabla f(X_k) + \sqrt{2\eta}\,  \xi_k
\]
where $k \geq 0$ is an integer, $\eta > 0$ is the step size, and $\xi_k \sim \gammad$ for all $k \geq 0$. As $\eta \to 0$, we obtain the Langevin dynamics given by
\[
\d X_t = -\nabla f(X_t) \dt + \sqrt{2} \d W_t
\]
where $W_t$ is standard Brownian motion. The Langevin dynamics admits the target distribution $\nu$ as its stationary distribution, and we denote the stationary distribution of uLA by $\tilde \nu_{\eta}$. Let $\tilde \bP_{\eta}$ denote the transition kernel for one step of uLA, and let $\bP_{\eta}$ denote the transition kernel for the Langevin dynamics for time $\eta$. Recall that $\tilde \bP_{T,h}$ denotes the \uhmcv transition kernel, $\tilde \bQ_{T,h}$ denotes the \uhmcs transition kernel, and the kernel for \ehmc is $\bP_T$. 

Our goal is to obtain asymptotic bias bounds for two-step Markov chains where the Wasserstein contraction is driven by \uhmc and the cross-regularization is provided by uLA (see Section~\ref{sec:Discussion}), and study the corresponding gradient complexities to output a sample $X$ such that $\sfD(\law(X) \dvert \nu) \leq \epsilon$ for a strongly log-concave target distribution $\nu$, accuracy $\epsilon > 0$, and divergence $\sfD \in \{ \KL, \sfR_q\}$.
As we want to study the dependence on the dimension $d$ and accuracy $\epsilon$ in the final gradient complexities, we simplify the subsequent analysis and include dependence only on the relevant parameters which are the dimension $d$, uLA step size $\eta$, and the \uHMC step size $h$.
We also use $\gtrsim, \lesssim$ to hide absolute constants.

\subsection{Asymptotic bias in KL divergence}\label{app:uLA-CR-KL}

A cross-regularization bound for uLA is obtained in~\cite[Lemma~4.7]{altschuler2024shifted} and they show that under Assumption~\ref{assumption:potential}(b), and for $\eta \lesssim \frac{1}{L}$, the following holds for all $x,y \in \R^d$
\[
\KL(\delta_x \tilde \bP_\eta \dvert \delta_y \bP_\eta) \lesssim \frac{\|x-y\|^2}{\eta} + \eta^3 \|\nabla f(x)\|^2 + \eta^2 d\,.
\]
Taking expectation on both sides with respect to $(x,y)$ drawn from the optimal Wasserstein--$2$ coupling between two distributions $\mu, \pi \in \P(\R^d)$, and using Lemma~\ref{lem:KLConvexityOneStep} yields
\begin{equation}\label{eq:ULA_RegularityGeneral}
\KL(\mu \tilde \bP_{\eta} \dvert \pi \bP_\eta) \lesssim \frac{W_2^2(\mu, \pi)}{\eta} + \eta^3 \E_{\mu}[\| \nabla f(x)\|^2] + \eta^2 d\,,
\end{equation}
giving us the analogue of Lemma~\ref{lem:KL_Target_OneStepGeneral} for uLA. Similar to the proof of Theorem~\ref{thm:KLTarget}, we will now incorporate Wasserstein--$2$ mixing and Wasserstein--$2$ bias guarantees. We consider \uhmcv and \uhmcs separately, assuming strong log-concavity of the target distribution.

\paragraph{Wasserstein contractivity driven by \uhmcv.} Taking $\mu = \rho \tilde \bP_{T,h}^k$ and $\pi = \nu$ (for an initial distribution $\rho$ and the target distribution $\nu$) in~\eqref{eq:ULA_RegularityGeneral}, together with triangle inequality yields
\[
\KL(\rho \tilde \bP_{T,h}^k \tilde \bP_\eta \dvert \nu) \lesssim \frac{W_2^2(\rho \tilde \bP_{T,h}^k, \tilde \nu_h)}{\eta} + \frac{W_2^2(\tilde \nu_h, \nu)}{\eta} + \eta^3 \E_{\rho \tilde \bP_{T,h}^k}[\| \nabla f(x)\|^2] + \eta^2 d\,.
\]
Propositions~\ref{prop:W2mixing-Verlet} and~\ref{prop:W2bias-Verlet}, in addition with $\E_{\rho \tilde \bP_{T,h}^k}[\| \nabla f(x)\|^2] = O(d)$ (see~\cite[Lemma~C.3]{altschuler2024shifted}) yields the following
\[
\KL(\rho \tilde \bP_{T,h}^k \tilde \bP_\eta \dvert \nu) \lesssim \frac{e^{-k} W_2^2(\rho , \tilde \nu_h)}{\eta} + \frac{h^4d^2}{\eta} + \eta^3 d + \eta^2 d\,.
\]
To make the right hand side within the desired accuracy $\epsilon$, take $h \asymp \epsilon^{\nicefrac{3}{8}} d^{-\nicefrac{5}{8}}$ and $\eta \asymp \sqrt{\nicefrac{\epsilon}{d}}$, which combined with $k = O(\log(\nicefrac{d}{\epsilon}))$ many iterations, yields a first-order oracle complexity of $O(d^{\nicefrac{5}{8}}\epsilon^{-\nicefrac{3}{8}} \log (\nicefrac{d}{\epsilon}))$. This is presented in Table~\ref{table:complexity}, Row~8.

\paragraph{Wasserstein contractivity driven by \uhmcs.} We start off by taking $\mu = \rho \tilde \bP_{T,h}^k$ and $\pi = \nu$ (for an initial distribution $\rho$ and the target distribution $\nu$) in~\eqref{eq:ULA_RegularityGeneral}, which together with triangle inequality yields
\[
\KL(\rho \tilde \bQ_{T,h}^k \tilde \bP_\eta \dvert \nu) \lesssim \frac{W_2^2(\rho \tilde \bQ_{T,h}^k, \tilde \nu_{\tilde \bQ_{T,h}})}{\eta} + \frac{W_2^2(\tilde \nu_{\tilde \bQ_{T,h}}, \nu)}{\eta} + \eta^3 \E_{\rho \tilde \bQ_{T,h}^k}[\| \nabla f(x)\|^2] + \eta^2 d\,.
\]
Propositions~\ref{prop:shmc-mixing} and~\ref{prop:shmc-bias} together with~\cite[Lemma~C.3]{altschuler2024shifted} yields
\[
\KL(\rho \tilde \bP_{T,h}^k \tilde \bP_\eta \dvert \nu) \lesssim \frac{e^{-k} W_2^2(\rho , \tilde \nu_{\tilde \bQ_{T,h}})}{\eta} + \frac{h^3d}{\eta} + \eta^3 d + \eta^2 d\,.
\]
Given desired accuracy $\epsilon$, pick $h \asymp \sqrt{\nicefrac{\epsilon}{d}}$, $\eta \asymp \sqrt{\nicefrac{\epsilon}{d}}$, which for $k = O(\log(\nicefrac{d}{\epsilon}))$ iterations yields an oracle complexity of $O(d^{\nicefrac{1}{2}}\epsilon^{-\nicefrac{1}{2}} \log (\nicefrac{d}{\epsilon}))$. This corresponds to Table~\ref{table:complexity}, Row~9.

\subsection{Asymptotic bias in Rényi divergence}\label{app:uLA-CR-Renyi}

A regularization property for uLA in Rényi divergence is described in~\cite[(3.14)]{altschuler2023shifted}, where they show that for all $x,y \in \R^d$
\[
\sfR_q (\delta_x \tilde \bP_\eta \dvert \delta_y \tilde \bP_\eta) \lesssim \frac{\|x-y\|^2}{\eta}\,.
\]
Applying Lemma~\ref{lem:RenyiConvexityOneStep} implies that for any $\mu, \pi \in \P(\R^d)$,
\[
\sfR_q(\mu \tilde \bP_\eta \dvert \pi \tilde \bP_\eta) \lesssim \inf_{\gamma \in \Gamma(\mu, \pi)} \log \E_{(x,y) \sim \gamma} \left[ e^{\frac{\|x-y\|^2}{\eta}}   \right]\,.
\]
Applying this to $\mu = \rho \tilde \bP_{T,h}^k$ (for some initial distribution $\rho$) and $\pi = \tilde \nu_h$ (where $\tilde \nu_h$ is the stationary distribution of \uhmcv) yields
\[
\sfR_q(\rho \tilde \bP_{T,h}^k \tilde \bP_\eta \dvert \tilde \nu_h \tilde \bP_\eta) \lesssim \inf_{\gamma \in \Gamma(\rho \tilde \bP_{T,h}^k, \tilde \nu_h)} \log \E_{(x,y) \sim \gamma} \left[ e^{\frac{\|x-y\|^2}{\eta}}   \right] = \log \inf_{\gamma \in \Gamma(\rho \tilde \bP_{T,h}^k, \tilde \nu_h)} \E_{(x,y) \sim \gamma} \left[ e^{\frac{\|x-y\|^2}{\eta}}   \right].
\]
Under an Orlicz-Wasserstein mixing assumption for \uhmcv (Proposition~\ref{prop:OWmixing-Verlet}), and following similar calculations as in the proof of Theorem~\ref{thm:RenyiMixing}, we see that
\begin{equation}\label{eq:uLA-CR-Renyi-Mixing}
\sfR_q(\rho \tilde \bP_{T,h}^k \tilde \bP_\eta \dvert \tilde \nu_h \tilde \bP_\eta) \lesssim  \frac{e^{-k}\,W_\psi (\rho, \tilde \nu_h)}{\eta}
\end{equation}
as long as $k \gtrsim \log \frac{W_\psi(\rho, \tilde \nu_h)}{\eta}$\,.
This yields the analogue of Corollary~\ref{cor:RenyiMixingGeneral} for uLA cross-regularization and specified to Wasserstein contraction driven by \uhmcv.

We now want to prove the analogue of Theorem~\ref{thm:RenyiTarget} for our setting and begin by stating a cross-regularization property for uLA in Rényi divergence~\cite[Lemma~C.2]{altschuler2024shifted}, where they show that under Assumption~\ref{assumption:potential}(b) and for $\eta \lesssim \frac{1}{L}$, for all $x,y \in \R^d$,
\begin{equation}\label{eq:uLA-CR-Dirac-Renyi}
\sfR_q(\delta_x \tilde \bP_\eta \dvert \delta_y \bP_\eta) \lesssim \frac{\|x-y\|^2}{\eta} + \eta^3 \|\nabla f(x)\|^2 + \eta^2 d\,.
\end{equation}

Applying the weak triangle inequality for Rényi divergence~\eqref{eq:RenyiWTI} we have that
\[
\sfR_q(\rho \tilde \bP_{T,h}^k \tilde \bP_\eta \dvert \nu) \lesssim \sfR_{2q}(\rho \tilde \bP_{T,h}^k \tilde \bP_\eta \dvert \tilde \nu_h \tilde \bP_\eta) + \sfR_{2q}(\tilde \nu_h \tilde \bP_\eta \dvert \nu \bP_\eta)\,.
\]
The first term on the right-hand side is controlled by~\eqref{eq:uLA-CR-Renyi-Mixing}, and for the second term we will proceed from~\eqref{eq:uLA-CR-Dirac-Renyi}.
To that end, Lemma~\ref{lem:RenyiConvexityOneStep} along with~\eqref{eq:uLA-CR-Dirac-Renyi} yields
\[
\sfR_{2q} (\tilde \nu_h \tilde \bP_\eta \dvert \nu) \lesssim  \inf_{\gamma \in \Gamma (\tilde \nu_h, \nu)} \log \E_{(x,y) \sim \gamma} \left[  e^{\frac{\|x-y\|^2}{\eta} + \eta^3 \|\nabla f(x)\|^2 + \eta^2 d} \right]
\]
Assumptions~\ref{assumption:potential}(a)--(b) allow us to bound this by
\[
\sfR_{2q} (\tilde \nu_h \tilde \bP_\eta \dvert \nu ) \lesssim  \inf_{\gamma \in \Gamma (\tilde \nu_h, \nu)} \log \E_{(x,y) \sim \gamma} \left[  e^{\frac{\|x-y\|^2}{\eta} + \eta^3 \|x\|^2 + \eta^2 d} \right]
\]
which can be further simplified to
\[
\sfR_{2q} (\tilde \nu_h \tilde \bP_\eta \dvert \nu ) \lesssim \eta^2 d +  \log \inf_{\gamma \in \Gamma (\tilde \nu_h, \nu)} \E_{(x,y) \sim \gamma} \left[  e^{\|x-y\|^2(\frac{1}{\eta} + 2 \eta^3) + 2\|y\|^2 \eta^3} \right]\,.
\]
Using Hölder's inequality we have
\[
\sfR_{2q} (\tilde \nu_h \tilde \bP_\eta \dvert \nu ) \lesssim \eta^2 d +  \log \inf_{\gamma \in \Gamma (\tilde \nu_h, \nu)} \underbrace{\left(\E_{(x,y) \sim \gamma} \left[  e^{2\|x-y\|^2(\frac{1}{\eta} + 2 \eta^3)} \right]\right)^{1/2}}_{\rn{1}} \underbrace{\left( \E_{y \sim \nu} \left[ e^{ 4\|y\|^2 \eta^3} \right] \right)^{1/2}}_{\rn{2}}\,.
\]
Following similar steps as in the proof of Theorem~\ref{thm:RenyiTarget} to handle $\rn{1}$ and $\rn{2}$, we obtain the following, which holds under suitable assumptions on the step sizes and also uses the Orlicz-Wasserstein bias guarantee for \uhmcv (Proposition~\ref{prop:OWbias-Verlet})
\[
\sfR_{2q} (\tilde \nu_h \tilde \bP_\eta \dvert \nu ) \lesssim \eta^2 d + \Big(\frac{1}{\eta} +  \eta^3\Big) h^2 d + \eta^3 K_\nu^2\,.
\]
Considering $K_\nu = O(\sqrt{d})$, we therefore obtain
\[
\sfR_{2q} (\tilde \nu_h \tilde \bP_\eta \dvert \nu ) \lesssim \eta^2 d + \Big(\frac{1}{\eta} +  \eta^3\Big) h^2 d + \eta^3 d\,.
\]

This yields to an overall asymptotic bias guarantee of
\[
\sfR_q(\rho \tilde \bP_{T,h}^k \tilde \bP_\eta \dvert \nu) \lesssim  \frac{e^{-k}\,W_\psi (\rho, \tilde \nu_h)}{\eta} + \eta^2 d + \Big(\frac{1}{\eta} +  \eta^3\Big) h^2 d + \eta^3 d\,.
\]

To make the total error within accuracy $\epsilon$, we set $\eta \asymp (\epsilon/d)^{1/2}$ and $h \asymp (\epsilon / d)^{3/4} $ to get an overall gradient complexity of $O(d^{3/4} \epsilon^{-3/4} \log (d/\epsilon))$, shown in Table~\ref{table:complexity}, Row~10.

\section{Helper Lemmas}

\begin{lemma}\label{lem:Helper_Chi_d}
Let $V\sim\gammad$ and set $m_d \coloneqq \E\|V\|$.
Then for any $c\ge 0$,
\[
\E\big[e^{c\|V\|}\big]
\le
\exp\!\left(c\,m_d+\frac{c^2}{2}\right).
\]
In particular, since $m_d\le \sqrt d$,
\[
\E\big[e^{c\|V\|}\big]
\le
\exp\!\left(c\sqrt d+\frac{c^2}{2}\right).
\]
\end{lemma}

\begin{proof}
The function $f(v)=\|v\|$ is $1$--Lipschitz. For a standard Gaussian vector
$V\sim\gammad$, the Gaussian concentration inequality \cite{Ledoux2001}
states that for any $c\in\R$,
\[
\E\exp\!\big(c(f(V)-\E f(V))\big)\le \exp\!\left(\frac{c^2}{2}\right).
\]
Multiplying by $\exp(c\,\E f(V))$ and using $f(V)=\|V\|$ gives
\[
\E e^{c\|V\|}
\le
\exp\!\left(c\,\E\|V\|+\frac{c^2}{2}\right).
\]
Finally, Jensen's inequality yields $\E\|V\|\le (\E\|V\|^2)^{1/2}=\sqrt d$.
\end{proof}

\begin{lemma}\label{lem:QuadraticFact}
    Let $f(x) = ax^2 + bx + c$ be a quadratic function in $x \in \R_{\geq 0}$ with $a > 0$, $b>0$, $c<0$. Then for $f(x) \leq 0$ it suffices to have
    \[
    x \leq \min \left\{ \frac{-c}{2b}, \sqrt{\frac{-c}{4a}}  \right\}\,.
    \]
\end{lemma}

\begin{proof}
    Suppose 
    \[
    x \leq \min \left\{ \frac{-c}{2b}, \sqrt{\frac{-c}{4a}}  \right\}\,.
    \]
    Then $bx \leq -\frac{c}{2}$ and $ax^2 \leq -\frac{c}{4}$. Hence $f(x) = ax^2 + bx + c \leq \frac{c}{4}$ which is $\leq 0$ as $c < 0$. 
\end{proof}

\addcontentsline{toc}{section}{References}
\bibliographystyle{alpha}
\bibliography{refs}

\end{document}